\documentclass{article} 
\usepackage{iclr2027_conference_preprint,times}

\usepackage[utf8]{inputenc}
\usepackage[T1]{fontenc}
\usepackage{booktabs}
\usepackage{nicefrac}
\usepackage{microtype}
\usepackage[dvipsnames]{xcolor}
\usepackage{graphicx}
\usepackage{newfloat}
\usepackage{listings}
\usepackage{algorithm}
\usepackage{algorithmicx}
\usepackage{wrapfig}
\usepackage{mathtools}
\usepackage{multirow}
\usepackage{bbm}
\usepackage{amssymb}
\usepackage{amsthm}
\usepackage{algpseudocode}
\usepackage{subcaption}
\usepackage{longtable}
\usepackage{array}
\usepackage[normalem]{ulem}

\usepackage{amsmath,amsfonts,bm}

\def\eqref#1{equation~\ref{#1}}

\def\1{\bm{1}}

\DeclareMathAlphabet{\mathsfit}{\encodingdefault}{\sfdefault}{m}{sl}
\SetMathAlphabet{\mathsfit}{bold}{\encodingdefault}{\sfdefault}{bx}{n}

\usepackage{comment}
\usepackage{hyperref}
\usepackage{url}
\usepackage{enumitem}
\usepackage{placeins}
\usepackage{tcolorbox}
\tcbuselibrary{breakable,skins}
\usepackage{xcolor}
\usepackage{cleveref}

\newtheorem{definition}{Definition}
\newtheorem{lemma}{Lemma}
\newtheorem{theorem}{Theorem}

\newtheorem{corollary}{Corollary}

\title{%
  Data-Topology Compatibility Enables \\
  Personalization in Decentralized Learning}

\iclrfinalcopy

\author{%
  Veronica Xu\\
  University of Edinburgh\\
  \And
  Lingfei Wang\\
  KTH\\
  \And
  Chenlei Leng\\
  PolyU\\
  \And
  Fengxiang He\\
  University of Edinburgh\\
}

\begin{document}

\maketitle

\begin{abstract}
Decentralized learning must balance collaboration with personalization across heterogeneous agents. This paper studies when cooperation can support statistically coherent groups of personalized models. We show that data-topology compatibility enables robust orbit-wise personalization: when local objectives and communication topology have (approximately) compatible automorphisms, agents within an orbit share (approximately) equal limiting distributions, while different orbits remain distinct without explicit personalization. This creates a new intermediate regime between global consensus and independent local learning. Under exact data-topology compatibility, as a reference system, we decompose the stochastic learning dynamics into orbit-level and transverse components, derive a necessary and sufficient stability criterion for distributional synchronization, and characterize the orbit-structured stationary law. Modelling an approximate data-topology compatible system as an \(\varepsilon\)-perturbed reference system, we prove the introduced errors have uniform-in-time and stationary \(O(\varepsilon)\) bounds in 2-Wasserstein distance. We also prove sufficient conditions for lower stationary average excess risk than both independent local training and the best single shared parameter. Experiments on synthetic data, MNIST, and CIFAR-10 are in full agreement with the theory. The code is available at \url{https://github.com/Veronicaaa1/orbit-personalization}.
\end{abstract}

\section{Introduction}

Decentralized learning allows agents to train through local computation and neighbor communication without sharing raw data \citep{nedic2009distributed}. When local objectives differ, however, cooperation must accommodate competing learning needs \citep{nassif2020multitaskII, zantedeschi2020fully, le_bars2023refined}. Independent training preserves each agent’s 
local objective but forgoes the benefits of collaboration \citep{lian2017can, koloskova2020unified,even2022sample}, whereas a single shared model can compromise performance across heterogeneous tasks \citep{hsieh2020non, le_bars2023refined}. Communication can reduce stochastic fluctuations \citep{xiao2007distributed, kar2008distributed,bamieh2012coherence}, yet coupling agents with different local optima can also introduce bias \citep{chen2015diffusion, nassif2020multitaskII, koloskova2020unified}. This raises a fundamental question: 
\begin{quote}
\it\centering
Under what conditions can local updates and neighbor communication produce statistically coherent groups of personalized models that benefit from cooperation?
\end{quote}
Answering this question requires understanding how local learning tasks and communication structure jointly determine which agents develop similar behavior, which remain distinct, and when these differences improve learning. This paper proves:
\begin{quote}
\it
\centering
(Approximately) compatible {data-topology} automorphism enables {robust orbit-wise personalization} in heterogeneous decentralized learning. 
\end{quote}

In an automorphic communication topology, a permutation of agents preserves the weighted communication graph, while associated orbits collect agents that can be mapped to one another. Exact data-topology compatibility additionally requires local objective characteristics, such as curvatures, noise intensities, and minimizers, to agree within the topology's algebraic structure; approximate compatibility allows small deviations from such a reference system. 
Under suitable stability conditions, agents within an orbit develop common stationary parameter marginals, while different orbits retain distinct distributions. This agreement is distributional: stochastic noise may sustain differences between individual parameter trajectories. 
The resulting organization provides a new intermediate regime between global consensus and independent local learning, without an explicit personalization step (see relevant works in Appendix~\ref{app:additional_literature}). Moreover, agents can exhibit this statistical agreement even when they have no direct communication edges \citep{nicosia2013remote}, reflecting the role of their positions in the wider network.

We develop the theory in three steps. 
For an exactly compatible system, as reference, we decompose the dynamics into orbit-level motion and transverse disagreement within orbits. This decomposition yields a necessary and sufficient transverse-stability criterion for distributional orbit synchronization, together with convergence rates for the transverse mean and covariance. 
We then model an approximate compatible system, the general case, by applying an \(\varepsilon\) perturbation to the local curvatures, minimizers, noise intensities, and communication weights of the corresponding reference system. 
Uniform-in-time and stationary \(O(\varepsilon)\) bounds in 2-Wasserstein distance \citep{villani2009optimal} are proved, which imply approximate agreement within reference orbits and explicit thresholds for preserving separation across them. Finally, by separating coupling bias from stationary fluctuation cost, we derive sufficient conditions under which decentralized learning achieves lower stationary excess risk than both independent local training and the best single shared parameter. Both advantages persist under sufficiently small perturbations, with guaranteed gains reduced by at most \(O(\varepsilon)\).

We examine our theory through synthetic systems and neural-network experiments on a communication graph with fourteen agents and four orbits. Quadratic experiments verify the predicted stationary distributions and show that orbit separation and both risk advantages persist across all 120 tested perturbation configurations. A smooth nonquadratic example exhibits similar qualitative grouping beyond the affine setting. Experiments with multi-layer perceptron (MLP) and ResNet \citep{he2016deep} models on MNIST \citep{lecun1998gradient} and CIFAR-10 \citep{krizhevsky2009learning} reveal corresponding prediction patterns: 
agents within reference orbits make relatively similar predictions, while differences remain across orbits, despite the absence of direct communication within each orbit. Comparisons with independent training and misaligned data allocations further show that both communication and data-topology alignment contribute to this organization. 
\subsection{Related works}
\label{sec:related_work}

\paragraph{Algebraic structures of communication topology.}
The literature of decentralised learning has seen studies on graph automorphisms~\citep{MacArthur2008}, symmetry-based reductions~\citep{SanchezGarcia2020}, and graph partitions for reduced network dynamics~\citep{schaub2016graph}, showing how communication topology induces invariant subspaces and reduced collective dynamics. Related group-consensus and multiconsensus results identify synchronized groups under suitable dynamics or control protocols~\citep{Klickstein2019,Tomaselli2023}, while cluster-synchronization theory studies quotient dynamics and transverse stability of symmetry- or partition-induced clusters~\citep{pecora2014cluster,schaub2016graph}. These works provide the algebraic foundation for our orbit-level decomposition, but mainly concern structural network analysis, deterministic consensus, oscillator synchronization, or controlled multiconsensus. Our work builds on these tools to characterize how local learning
tasks and communication roles jointly organize stationary parameter
distributions, connecting orbit-based coordination to statistical
behavior in decentralized learning.

\paragraph{Stochastic and non-consensus multi-agent dynamics.}
Stochastic consensus analyses establish agreement under noisy
measurements \citep{huang2009coordination}, time-varying topologies
with communication noise \citep{li2010consensus}, and random
communication networks \citep{tahbaz2010consensus}.
Non-consensus also arises in bounded-confidence opinion models,
which can produce opinion clusters \citep{Deffuant2000} and
polarization or fragmentation \citep{Hegselmann2002}.
Other stochastic opinion models show that random interactions with
agents holding fixed, different opinions can sustain persistent
fluctuations and disagreement while the joint state converges
in distribution \citep{acemoglu2013opinion}.
Network symmetries further organize time-averaged observables
in stochastic network models, yielding common long-run statistics
within automorphism orbits \citep{Sorrentino2019}.
Our analysis identifies when within-orbit statistical agreement
and cross-orbit separation persist under approximate data-topology
compatibility, and when this orbit-wise personalization benefits
decentralized learning.

More related works are given in Appendix~\ref{app:additional_literature}.
\section{Preliminaries}
\label{preliminaries}

\paragraph{Decentralized learning.}
Let \(\mathcal G=(\mathcal V,\mathcal E)\), with
\(\mathcal V=\{1,\dots,N\}\), be an undirected communication graph
over which \(N\) agents train without a central server. Each agent \(i\) updates its local parameter \(w_i(t)\in\mathbb R^m\)
by combining local stochastic-gradient steps based on samples from its
own data distribution \(\mathcal D_i\) , and exchanges information only with adjacent agents. For any local parameter value \(u\in\mathbb R^m\), agent \(i\)'s local population risk is \(L_i(u)=\mathbb E_{\xi\sim\mathcal D_i}[\ell(u;\xi)]\), where \(\xi\) is a local sample and \(\ell(u;\xi)\) is the sample loss. 
Around a reference local minimizer \(w_i^\star\in\mathbb R^m\) of \(L_i\), we use the quadratic approximation
\(\widetilde L_i(u)=L_i(w_i^\star)+\frac12(u-w_i^\star)^\top
Q_i(u-w_i^\star)\), where \(Q_i=Q_i^\top\succeq0\) is the local curvature matrix, and hence the local 
gradient model
\(\nabla\widetilde L_i(u)=Q_i(u-w_i^\star)\). Local quadratic and stochastic-process
approximations are standard tools for studying SGD near minima
\citep{mandt2017sgd, he2019control,liu2021noise}.
A conventional equal-node-weighted centralized reference is the single-model objective \(L_{\rm glob}(u)=N^{-1}\sum_{i=1}^N L_i(u)\), which optimizes one shared parameter \(u\) against the average of all agents' local population risks. In contrast, our personalized analysis keeps node-specific parameters \(w=(w_1^\top,\ldots,w_N^\top)^\top\in\mathbb R^{Nm}\) and studies their stochastic dynamics under local objectives and graph-mediated communication. The node-wise population risk is \(R(w)=N^{-1}\sum_{i=1}^N L_i(w_i)\). For a fixed training sample \(\mathcal S\), we define the sample-dependent generalization gap \(G_{\mathcal S}(w)=R(w)-\widehat R_{\mathcal S}(w)\), where \(\widehat R_{\mathcal S}\) is the empirical counterpart of \(R\).

\paragraph{Optimization dynamics.}
Let
\(\eta>0\) be the learning-rate scale and
\(H\in\mathbb R^{m\times m}\) the parameter-space coupling operator.
For each node \(i\), let \(\sigma_i\in\mathbb R^{m\times m}\) be the
noise-intensity matrix and \(B_i(t)\) an \(m\)-dimensional standard
Brownian motion, independent across nodes. The communication matrix
\(C=[C_{ij}]\in\mathbb R^{N\times N}\) satisfies
\(C_{ij}=C_{ji}\ge0\), \(C_{ii}=0\), and \(C_{ij}=0\) whenever
\(\{i,j\}\notin\mathcal E\). The node-wise affine SDE is
\[
dw_i(t)=
\left[
-\eta Q_i(w_i(t)-w_i^\star)
+\sum_{j=1}^N C_{ij}H(w_j(t)-w_i(t))
\right]dt+\sigma_i\,dB_i(t).
\]
Define \(z_i(t):=w_i(t)-w_i^\star\) and
\(b_i^\star:=\sum_{j=1}^N C_{ij}H(w_j^\star-w_i^\star)\).
Stack \(w_i^\star\), \(B_i(t)\), and \(b_i^\star\) in node order to
form \(w^\star\), \(B(t)\), and \(b^\star\), respectively, and write
\(z(t)=w(t)-w^\star\). Define the \(m\times m\) blocks of
\(K,\Sigma\in\mathbb R^{Nm\times Nm}\) by
\(K_{ii}:=-\eta Q_i-(\sum_{j=1}^N C_{ij})H\),
\(K_{ij}:=C_{ij}H\) for \(i\ne j\),
\(\Sigma_{ii}:=\sigma_i\), and \(\Sigma_{ij}:=0\) for \(i\ne j\).
With calculation detailed in
Appendix~\ref{app:continuous_time_dsgd}, the centered dynamics is
\begin{equation}
dz(t)=\bigl(Kz(t)+b^\star\bigr)dt+\Sigma\,dB(t).
\label{eq:stacked_centered_sde}
\end{equation}

\paragraph{Graph automorphisms and \(\Gamma\)-orbits.}
Fix a subgroup \(\Gamma\le\mathrm{Aut}(\mathcal G)\), where
\(\mathrm{Aut}(\mathcal G)\) is the group of adjacency-preserving
node permutations. For \(\pi\in\Gamma\), let
\(P_\pi\in\mathbb R^{N\times N}\) be defined by
\(P_\pi e_j=e_{\pi(j)}\), where \(e_j\) is the \(j\)-th standard
basis vector. The \(\Gamma\)-orbit of node \(i\) is
\(\mathcal O(i):=\{\pi(i):\pi\in\Gamma\}\).
These orbits partition \(\mathcal V\) as
\(\mathcal V=\bigsqcup_{k=1}^q\mathcal O_k\), with
\(n_k:=|\mathcal O_k|\); we assume \(2\le q<N\), so that there
are multiple orbits and at least one orbit contains more than one agent. See Appendix~\ref{app:def:graph-automorphisms} for further details
on graph automorphisms.

\section{Main results}

We study decentralized stochastic learning under approximate
data-topology compatibility relative to a fixed exactly compatible
reference.
Our results quantify distributional stability and approximate
agreement of stationary marginals within reference orbits, and
give conditions for preserving cross-orbit separation.
We also establish sufficient conditions for lower stationary
average excess risk than both independent local training and
the best single shared parameter.

\subsection{Data-topology compatibility}
\label{subsec:orbit_decomposition}

We define approximately compatible systems as perturbations of a fixed
exactly compatible reference, using the dynamics and orbit partition of
Section~\ref{preliminaries}.
Both processes are expressed in the original \(w\)-coordinates.

\begin{definition}[Exact data-topology compatibility]
\label{ass:orbit_identical_main}
A reference collection has the form
\[
\theta_0
=
\left(
\{(\bar Q^{(k)},\bar\sigma^{(k)},\bar w^{\star(k)})\}_{k=1}^q,
C_0
\right),
\]
where \(\bar Q^{(k)}=(\bar Q^{(k)})^\top\succeq0\) and
\(C_0\) is an admissible communication matrix on \(\mathcal G\),
as specified in Section~\ref{preliminaries}. It assigns
\((\bar Q^{(k)},\bar\sigma^{(k)},\bar w^{\star(k)})\)
to every node in \(\mathcal O_k\).
The collection is \emph{exactly data-topology compatible} with
respect to \(\Gamma\) if its communication weights are invariant
under every \(\pi\in\Gamma\):
\(
P_\pi C_0=C_0P_\pi,
\)
equivalently, \((C_0)_{\pi(i),\pi(j)}=(C_0)_{ij}\) for all
\(i,j\).
\end{definition}

\begin{definition}[Exactly compatible reference dynamics]
\label{def:reference_dynamics}
Fix an exactly data-topology compatible collection \(\theta_0\)
as in Definition~\ref{ass:orbit_identical_main}. Its
\emph{reference dynamics} \(w_0(t)\) satisfies
\[
dw_0(t)=\bigl(K_0w_0(t)+a_0\bigr)dt+\Sigma_0\,dB(t).
\]
Here \(K_0\) and \(\Sigma_0\) are obtained from the block definitions
preceding \eqref{eq:stacked_centered_sde}, using \(C_0\) and the
reference local quantities. The node blocks of \(a_0\) are
\((a_0)_i=\eta\bar Q^{(k)}\bar w^{\star(k)}\) for
\(i\in\mathcal O_k\). We keep this reference fixed throughout
the perturbation analysis.
\end{definition}

The index \(0\) denotes the fixed exactly compatible reference system,
corresponding to \(\varepsilon=0\) in
Definition~\ref{ass:approx_orbit_compatibility}.
Its orbit decomposition and limiting distributions provide the analytical basis
for quantifying within-orbit agreement and cross-orbit separation
in approximately compatible systems.

\begin{definition}[Approximate data-topology compatibility]
\label{ass:approx_orbit_compatibility}
Relative to the fixed exactly compatible reference
\(\theta_0\) in Definition~\ref{def:reference_dynamics}, for
\(\varepsilon\ge0\), let \(
\theta_\varepsilon
=
\left(
\{(Q_i,\sigma_i,w_i^\star)\}_{i=1}^N,
C_\varepsilon
\right)
\)
be a perturbed collection. It is \emph{approximately
data-topology compatible} with \(\theta_0\) at level
\(\varepsilon\) if
\begin{equation}
\|C_\varepsilon-C_0\|_{\mathrm F}
+
\max_{1\le k\le q}\max_{i\in\mathcal O_k}
\left(
\|Q_i-\bar Q^{(k)}\|
+
\|\sigma_i-\bar\sigma^{(k)}\|_{\mathrm F}
+
\|w_i^\star-\bar w^{\star(k)}\|
\right)
\le\varepsilon.
\label{eq:approx_data_topology_defect}
\end{equation}
Uniformly over the perturbation class, all reference and perturbed
curvature matrices, noise-intensity matrices, and minimizers have
operator, Frobenius, and Euclidean norms, respectively, bounded by
\(M_{\rm par}<\infty\).
Both systems satisfy the structural and communication conditions of
Section~\ref{preliminaries}, on the same graph and with the same fixed
\(\eta\) and \(H\). No additional invariance under \(\Gamma\) is
imposed on \(C_\varepsilon\).
\end{definition}

\begin{definition}[Perturbed dynamics]
\label{def:perturbed_dynamics}
For a perturbed collection \(\theta_\varepsilon\) as in
Definition~\ref{ass:approx_orbit_compatibility}, take
\(K_\varepsilon,\Sigma_\varepsilon\) from the block formulas in
Section~\ref{preliminaries} and set
\((a_\varepsilon)_i=\eta Q_iw_i^\star\). Its
\emph{perturbed dynamics} \(w_\varepsilon(t)\) satisfies
\[
dw_\varepsilon(t)
=
\bigl(K_\varepsilon w_\varepsilon(t)+a_\varepsilon\bigr)dt
+\Sigma_\varepsilon\,dB(t).
\]
\end{definition}

\subsection{Orbit synchronization and separation}
\label{subsec:approx_orbit_compatibility}

The main theorem controls departures from the exactly compatible
reference and gives conditions for approximate within-orbit
distributional synchronization and persistent cross-orbit differences at stationarity.
The corollary strengthens the pairwise conclusions to separation
of the full reference orbit partition under an additional margin
condition.

\begin{theorem}[Orbit-wise personalization under perturbations]
\label{thm:robust_orbit_auto-personalization}
Under the conditions of Definition~\ref{ass:approx_orbit_compatibility},
suppose \(K_0\) is Hurwitz, i.e., its eigenvalues have
strictly negative real parts. Then there exists
\(C_{\rm coef}<\infty\) such that
\begin{equation}
\|K_\varepsilon-K_0\|
+
\|a_\varepsilon-a_0\|
+
\|\Sigma_\varepsilon-\Sigma_0\|_{\mathrm F}
\le C_{\rm coef}\varepsilon.
\label{eq:coefficient_perturbation_bound}
\end{equation}
Consequently, there exists \(\varepsilon_0>0\) such that
\(K_\varepsilon\) is Hurwitz for every
\(0\le\varepsilon\le\varepsilon_0\). For these values of
\(\varepsilon\), the following conclusions hold under any
additional conditions stated in the respective parts:
\begin{enumerate}[label=(\roman*), labelindent=0pt, leftmargin=*, align=left,topsep=0pt,itemsep=-1ex]
\item \emph{Uniform-in-time distributional stability.}
Suppose \(w_\varepsilon(0)\) and \(w_0(0)\) have a common law \(\nu\)
on \(\mathbb R^{Nm}\) with finite second moment and are independent
of the driving Brownian motion.
Then there exists \(C_{\rm dyn}(\nu)<\infty\) such that
\[
\sup_{t\ge0}
W_2\!\left(
\mathcal L(w_\varepsilon(t)),
\mathcal L(w_0(t))
\right)
\le C_{\rm dyn}(\nu)\varepsilon.
\]
\item \emph{Stationary-distribution stability.}
For \(\delta\in\{0,\varepsilon\}\), let
\(\varrho_\delta=\mathcal N(m_\delta,\Omega_\delta)\) denote the
unique stationary distribution of the corresponding system, where
\[
m_\delta=-K_\delta^{-1}a_\delta,
\qquad
K_\delta\Omega_\delta+\Omega_\delta K_\delta^\top
+\Sigma_\delta\Sigma_\delta^\top=0.
\]
The covariance \(\Omega_\delta\succeq0\) may be singular.
There exists \(C_{\rm stat}<\infty\) such that
\[
W_2(\varrho_\varepsilon,\varrho_0)
+
\|m_\varepsilon-m_0\|
+
\|\Omega_\varepsilon-\Omega_0\|
\le C_{\rm stat}\varepsilon.
\]

\item \emph{Approximate agreement within reference orbits.}
Let \(W^\delta\sim\varrho_\delta\) have node blocks \(W_i^\delta\),
and define \(\mu_k^0:=(m_0)_i\) and
\(\Xi_k^0:=\Omega_0^{(i,i)}\) for \(i\in\mathcal O_k\).
These definitions are independent of \(i\) by
Theorem~\ref{thm:symmetry_breaking_main}.
After enlarging \(C_{\rm stat}\) if necessary, for every
\(i\in\mathcal O_k\),
\begin{equation}
\|(m_\varepsilon)_i-\mu_k^0\|
+\|\Omega_\varepsilon^{(i,i)}-\Xi_k^0\|
\le C_{\rm stat}\varepsilon.
\label{eq:block_parameter_perturbation_bound}
\end{equation}
Moreover, for every distinct \(i,j\in\mathcal O_k\),
\[
W_2(W_i^\varepsilon,W_j^\varepsilon)\le C_{\rm stat}\varepsilon,
\qquad\text{and}\qquad
W_2\!\left(W_i^\varepsilon-W_j^\varepsilon,W_i^0-W_j^0\right)
\le C_{\rm stat}\varepsilon.
\]
\item \emph{Persistence of cross-orbit separation at stationarity.}
For every pair \(1\le k<\ell\le q\), define
\[
\Delta_{k\ell}^0
:=
\|\mu_k^0-\mu_\ell^0\|
+
\|\Xi_k^0-\Xi_\ell^0\|.
\]
If \(2C_{\rm stat}\varepsilon<\Delta_{k\ell}^0\), the stationary
marginal distributions of \(W_i^\varepsilon\) and
\(W_j^\varepsilon\) are distinct for every
\(i\in\mathcal O_k\) and \(j\in\mathcal O_\ell\).
\end{enumerate}
All constants are uniform over the stated perturbation class and may
depend on \(\theta_0,N,m,\eta,H,M_{\rm par}\) and the reference stability
bounds. Only \(C_{\rm dyn}(\nu)\) also depends on the initial second moment;
one common \(C_{\rm stat}\) is used in (ii)--(iv).
\end{theorem}

At stationarity, approximate agreement of marginal distributions
within reference orbits can coexist with persistent cross-orbit
differences when the reference gaps are sufficiently large relative
to the perturbation level.
This does not require individual parameter trajectories to coincide.
Section~\ref{sec:reference_theory} develops the exactly compatible
reference structure underlying these conclusions.
It separates orbit-level motion from within-orbit disagreement,
characterizes distributional synchronization under transverse
stability and stationary orbit marginals under full stability,
and explains how this reference structure combines with
perturbation bounds to yield the approximate agreement and
separation guarantees above.
The full proof of Theorem~\ref{thm:robust_orbit_auto-personalization}
is given in Appendix~\ref{app:proof_approximation_error}.

To extend the pairwise separation guarantees to the full reference
orbit partition, we compare the smallest cross-orbit distance
with the largest within-orbit distance between stationary marginals.
The following corollary gives a sufficient perturbation condition
under which a single distance threshold separates all within-orbit
pairs from all cross-orbit pairs.

For Gaussian distributions, define
\(d_{\rm par}(\mathcal N(\mu,\Xi),\mathcal N(\mu',\Xi'))
:=\|\mu-\mu'\|+\|\Xi-\Xi'\|\).
Write \(\nu_i^\varepsilon:=\mathcal L(W_i^\varepsilon)\), and set
\[
D_{\rm in}(\varepsilon)
:=\max_{1\le k\le q}\max_{i,j\in\mathcal O_k}
d_{\rm par}(\nu_i^\varepsilon,\nu_j^\varepsilon),
\qquad
D_{\rm out}(\varepsilon)
:=\min_{1\le k<\ell\le q}
\min_{i\in\mathcal O_k,\,j\in\mathcal O_\ell}
d_{\rm par}(\nu_i^\varepsilon,\nu_j^\varepsilon).
\]

\begin{corollary}[Orbit separation margin]
\label{cor:orbit_separation_margin}
Suppose the conditions of
Definition~\ref{ass:approx_orbit_compatibility} hold and
\(K_0\) is Hurwitz. Assume further that
\(\underline\Delta_0:=\min_{1\le k<\ell\le q}\Delta_{k\ell}^0>0\).
Let \(\varepsilon_0\) and \(C_{\rm stat}\) be as in
Theorem~\ref{thm:robust_orbit_auto-personalization}.
For every \(0\le\varepsilon\le\varepsilon_0\),
\[
D_{\rm in}(\varepsilon)\le2C_{\rm stat}\varepsilon,
\qquad
D_{\rm out}(\varepsilon)
\ge\underline\Delta_0-2C_{\rm stat}\varepsilon.
\]
Consequently, \(\mathcal M_{\rm orb}(\varepsilon):=D_{\rm out}(\varepsilon)-D_{\rm in}(\varepsilon)\ge\underline\Delta_0-4C_{\rm stat}\varepsilon\).
In particular, if
\(4C_{\rm stat}\varepsilon<\underline\Delta_0\), the threshold
\(\underline\Delta_0/2\) separates all within-orbit pairs from all
cross-orbit pairs in \(d_{\rm par}\).
\end{corollary}

\subsection{Expected excess risk}
\label{subsec:expected-excess-risk}

We compare stationary average excess risk with two benchmarks:
independent local training and the best single shared parameter.
The following result gives sufficient conditions for advantages over
both benchmarks to persist under small compatibility perturbations.

For \(W^\varepsilon\sim\varrho_\varepsilon\), define the average
stationary excess risk by
\(\mathcal E_\varepsilon
:=N^{-1}\sum_{i=1}^N
\mathbb E[\widetilde L_i(W_i^\varepsilon)
-\widetilde L_i(w_i^\star)]\).
For quadratic losses with Hessians \(Q_i\), this risk satisfies
\begin{equation}
\mathcal E_\varepsilon
=
\frac{1}{2N}\sum_{i=1}^N
\left[
\bigl((m_\varepsilon)_i-w_i^\star\bigr)^\top
Q_i\bigl((m_\varepsilon)_i-w_i^\star\bigr)
+\operatorname{tr}\!\left(Q_i\Omega_\varepsilon^{(i,i)}\right)
\right].
\label{eq:dt-risk-definition}
\end{equation}
The terms measure the cost of displacement of the stationary
mean from the local optimum and the cost of stationary fluctuations,
respectively.
We suppress \(\varepsilon\) in the notation for the losses,
\(Q_i\), and \(w_i^\star\).

Let \(\mathcal E_{\rm ind}^\varepsilon\) denote the stationary
average excess risk of independent local training with the same
losses, learning rate, and diffusion.
Let \(\mathcal H_{\rm one}^\varepsilon\) denote the minimum average
excess risk attained by a single shared parameter for these losses.

\begin{theorem}[Learning advantages under approximate compatibility]
\label{thm:dt-learning-robustness}
Under the conditions of
Definition~\ref{ass:approx_orbit_compatibility}, suppose that
\(\bar Q^{(k)}\succ0\) for every reference orbit and
\(H+H^\top\succeq0\).
Let \(\mathsf Q_0=\operatorname{blkdiag}(Q_i^0)\), where
\(Q_i^0=\bar Q^{(k)}\) for \(i\in\mathcal O_k\).
Define \(d_0=\lambda_{\min}(\Sigma_0\Sigma_0^\top)\) and
\(\beta_0=\lambda_{\max}(-(K_0+K_0^\top)/2)\), and let
\(b_0^\star\) be the centered forcing in
\eqref{eq:stacked_centered_sde} evaluated at the reference system.
Set
\begin{equation}
g_0 := \frac{d_0\operatorname{tr}(H)\sum_{i,j=1}^N(C_0)_{ij}}
{4N\eta\beta_0}
-\frac{(b_0^\star)^\top\mathsf Q_0^{-1}b_0^\star}{2N\eta^2}.
\label{eq:dt-coefficient-gain}
\end{equation}

If \(g_0>\max\{0,\mathcal E_{\rm ind}^0-\mathcal H_{\rm one}^0\}\),
then there exist \(C_{\rm risk}>0\) and \(\varepsilon_0>0\),
with \(\varepsilon_0\) no larger than the radius in
Theorem~\ref{thm:robust_orbit_auto-personalization},
such that, for every \(0\le\varepsilon\le\varepsilon_0\),
the following bounds hold:
\begin{enumerate}[label=(\roman*), labelindent=0pt, leftmargin=*, align=left, topsep=0pt, itemsep=0pt]
\item \emph{Communication gain.}\quad $\mathcal E_{\rm ind}^\varepsilon-\mathcal E_\varepsilon \ge g_0-C_{\rm risk}\varepsilon>0.$
\item \emph{Personalization gain.}\quad $\mathcal H_{\rm one}^\varepsilon-\mathcal E_\varepsilon \ge g_0+\mathcal H_{\rm one}^0-\mathcal E_{\rm ind}^0-C_{\rm risk}\varepsilon>0.$
\end{enumerate}
\end{theorem}

The theorem gives sufficient conditions under which reduced
stationary fluctuations outweigh coupling bias and the stationary
average excess risk is lower than both benchmarks.
The \(O(\varepsilon)\) terms quantify the possible loss in the
guaranteed reference gains under compatibility perturbations.
Appendix~\ref{app:dt-learning-proofs} gives the benchmark formulas,
risk identities, and proof; Appendix~\ref{app:four_orbit_risk}
verifies both advantages in the four-orbit configuration.

\section{Development of theory}
\label{sec:reference_theory}

This section analyzes the exactly compatible reference system
underlying our main results, corresponding to the \(\varepsilon=0\)
case of approximate data-topology compatibility.
We develop its orbit decomposition, characterize distributional
orbit synchronization, and describe stationary orbit marginals
and conditions for cross-orbit separation. The final subsection connects these results to approximate compatibility. Throughout this section, the reference subscript \(0\) is suppressed.

\subsection{Orbit decomposition of reference dynamics}
\label{subsec:exact_reference_decomposition}

Under Definition~\ref{ass:orbit_identical_main},
\eqref{eq:stacked_centered_sde} is a \(\Gamma\)-equivariant affine SDE.
To separate orbit-level motion from transverse within-orbit disagreement,
define \(T_\parallel\in\mathbb R^{q\times N}\) by
\((T_\parallel)_{ki}:=n_k^{-1/2}\mathbf 1\{i\in\mathcal O_k\}\), and
choose \(T_\perp\in\mathbb R^{(N-q)\times N}\) whose rows form an
orthonormal basis of the orthogonal complement of the row space of
\(T_\parallel\). Set \(T:=[T_\parallel^\top,T_\perp^\top]^\top\), and let
\(\mathcal M_\Gamma:=\operatorname{range}(T_\parallel^\top\otimes I_m)\)
be the orbit-consensus subspace, whose elements are constant within each orbit.

Define the drift blocks
\(K_\parallel:=(T_\parallel\otimes I_m)K
(T_\parallel^\top\otimes I_m)\) and
\(K_\perp:=(T_\perp\otimes I_m)K
(T_\perp^\top\otimes I_m)\), and the diffusion blocks
\(\Sigma_\parallel:=(T_\parallel\otimes I_m)\Sigma
(T_\parallel^\top\otimes I_m)\) and
\(\Sigma_\perp:=(T_\perp\otimes I_m)\Sigma
(T_\perp^\top\otimes I_m)\).

\begin{theorem}[Decomposition of reference dynamics]
\label{thm:orbit_decomposition_main}
Under Definition~\ref{ass:orbit_identical_main}, define
\(\tilde z(t):=(T\otimes I_m)z(t)
=(\tilde z_{\parallel}(t)^\top,\tilde z_{\perp}(t)^\top)^\top\)
and
\(\tilde B(t):=(T\otimes I_m)B(t)
=(\tilde B_{\parallel}(t)^\top,\tilde B_{\perp}(t)^\top)^\top\). The transformed process \(\tilde B\) is a standard Brownian motion
with independent coordinate blocks \(\tilde B_\parallel\) and
\(\tilde B_\perp\).
The centered affine dynamics decomposes as
\begin{equation}
\begin{aligned}
d\tilde z_{\parallel}(t)
&=
\bigl[
K_{\parallel}\tilde z_{\parallel}(t)
+(T_\parallel\otimes I_m)b^\star
\bigr]dt
+\Sigma_{\parallel}\,d\tilde B_{\parallel}(t),\\
d\tilde z_{\perp}(t)
&=
K_{\perp}\tilde z_{\perp}(t)dt
+\Sigma_{\perp}\,d\tilde B_{\perp}(t).
\end{aligned}
\label{eq:para_orth_systems}
\end{equation}

Since the reference minimizers are identical within each orbit,
centering the parameters as \(z(t)=w(t)-w^\star\) gives
\(w_i(t)-w_j(t)=z_i(t)-z_j(t)\) for any \(i,j\) in the same orbit.
\end{theorem}

The quotient subsystem governs orbit-average motion, while the
transverse subsystem governs within-orbit disagreement. This separates
orbit-level personalization from within-orbit model disagreement and allows
the two types of dynamics to be analyzed independently. Full details are provided in
Appendix~\ref{app:proof_decomposition}.

\subsection{Convergence criterion and rates for the reference system}
\label{subsec:orbit_synchronization}

The decomposition in Theorem~\ref{thm:orbit_decomposition_main} shows
that parameter differences between agents in the same orbit depend
only on the transverse coordinate \(\tilde z_\perp(t)\).
We therefore study within-orbit synchronization through the transverse
subsystem in \eqref{eq:para_orth_systems}.
We formulate this synchronization in distribution and characterize
both when it occurs and how rapidly the transverse mean and covariance
approach stationarity.
The resulting limiting difference laws provide the reference for
the comparison of within-orbit parameter differences in part~(iii)
of Theorem~\ref{thm:robust_orbit_auto-personalization}.

\begin{definition}[Distributional orbit synchronization]
\label{def:orbit_synchronization}
The dynamics achieves distributional orbit synchronization if, for every
deterministic initial state \(z(0)=z_0\in\mathbb R^{Nm}\), every
\(\Gamma\)-orbit \(\mathcal O_k\), and every distinct
\(i,j\in\mathcal O_k\), the difference \(w_i(t)-w_j(t)\) converges in
distribution to a possibly degenerate centered Gaussian law as
\(t\to\infty\).
\end{definition}

\begin{theorem}[Criterion for distributional orbit synchronization]
\label{thm:orbit_synchronization_main}
Under the conditions of Theorem~\ref{thm:orbit_decomposition_main}, the
dynamics achieves distributional orbit synchronization if and only if
\(K_\perp\) is Hurwitz. When these equivalent conditions hold, the convergence
extends to every square-integrable initial condition \(z(0)\) independent of
the driving Brownian motion, and
\(\tilde z_\perp(t)\Rightarrow\mathcal N(0,\Omega_\perp)\) as
\(t\to\infty\), where \(\Omega_\perp\succeq0\), possibly singular, is the
unique solution of
\(
K_\perp\Omega_\perp+\Omega_\perp K_\perp^\top
+\Sigma_\perp\Sigma_\perp^\top=0.
\label{eq:transverse_stationary_covariance}
\)

For \(i,j\in\mathcal O_k\), define
\(R_{ij}:=((e_i-e_j)^\top T_\perp^\top)\otimes I_m\) and
\(\Omega_{ij}:=R_{ij}\Omega_\perp R_{ij}^\top\). Then
\(w_i(t)-w_j(t)=R_{ij}\tilde z_\perp(t)
\Rightarrow\mathcal N(0,\Omega_{ij})\) as \(t\to\infty\), where
\(\Omega_{ij}\) is independent of the initial condition.
\end{theorem}

\begin{theorem}[Convergence rates of transverse mean and covariance]
\label{thm:transverse_moment_rates_main}
Suppose the equivalent conditions of
Theorem~\ref{thm:orbit_synchronization_main} hold. For any
square-integrable initial condition independent of the driving Brownian
motion, let
\(
\alpha_{\perp}
:=
-\max\{\operatorname{Re}\lambda:
\lambda\in\operatorname{spec}(K_{\perp})\}>0
\),
define
\(
m_{\perp}(t):=\mathbb E[\tilde z_{\perp}(t)]
\),
and let
\(
\operatorname{Cov}_{\perp}(t)
:=
\mathbb E\!\left[
(\tilde z_{\perp}(t)-m_{\perp}(t))
(\tilde z_{\perp}(t)-m_{\perp}(t))^\top
\right].
\)
For every \(\beta\in(0,\alpha_\perp)\), there exists \(C_\beta\ge1\)
such that, for all \(t\ge0\),
\[
\|m_\perp(t)\|
\le
C_\beta e^{-\beta t}\|m_\perp(0)\|,
\qquad
\|\operatorname{Cov}_{\perp}(t)-\Omega_\perp\|
\le
C_\beta^2e^{-2\beta t}
\|\operatorname{Cov}_{\perp}(0)-\Omega_\perp\|.
\]
\end{theorem}

Under transverse stability, within-orbit differences have
initialization-independent limiting distributions, and their
means and covariances converge exponentially.
Noise may sustain residual fluctuations, so synchronization in
distribution does not require identical parameter trajectories.
Proofs of Theorems~\ref{thm:orbit_synchronization_main}
and~\ref{thm:transverse_moment_rates_main} are given in
Appendices~\ref{app:proof_orbit_synchronization}
and~\ref{app:proof_transverse_moment_rates}, respectively.

\subsection{Stationary distributions of the reference system}
\label{subsec:stationary_law}

To pass from within-orbit differences to stationary node marginals,
we impose stability of the full reference drift.
Under this stronger condition, each orbit has a common stationary
marginal, while differences in local objectives, stochastic-gradient
noise, and communication roles may lead to distinct marginals
across orbits.
These common marginals provide the reference for the approximate
within-orbit agreement in part~(iii) of
Theorem~\ref{thm:robust_orbit_auto-personalization},
while their pairwise gaps enter the separation condition in
part~(iv).

\begin{definition}[Distributional orbit separation]
\label{def:orbit_separation}
Suppose that every node parameter \(w_i(t)\) converges in distribution,
with common Gaussian limiting marginal \(\mathcal N(\mu_k,\Xi_k)\)
for all agents in each \(\Gamma\)-orbit \(\mathcal O_k\).
The limiting regime exhibits
distributional orbit separation if and only if there exist two distinct
orbits \(\mathcal O_k\) and \(\mathcal O_\ell\) whose Gaussian marginal
parameters differ, namely
\(
(\mu_k,\Xi_k)\neq(\mu_\ell,\Xi_\ell).
\)
\end{definition}

\begin{theorem}[Stationary distribution and orbit separation]
\label{thm:symmetry_breaking_main}
Under the conditions of
Theorem~\ref{thm:orbit_decomposition_main}, suppose that \(K\) is Hurwitz.
Define \(z_{\mathrm{eq}}:=-K^{-1}b^\star\in\mathcal M_\Gamma\), and let
\(\Omega\succeq0\), possibly singular, be the unique solution of the Lyapunov
equation \(K\Omega+\Omega K^\top+\Sigma\Sigma^\top=0\). For every
square-integrable initial condition \(z(0)\) independent of the driving
Brownian motion,
\[
w(t)\Rightarrow\mathcal N(w^\star+z_{\mathrm{eq}},\Omega).
\]
This limiting Gaussian distribution is the unique stationary distribution of
the \(w\)-process. For each orbit \(\mathcal O_k\), every
\(i\in\mathcal O_k\) has the same limiting marginal distributions,
\[
w_i(t)\Rightarrow\mathcal N(\mu_k,\Xi_k),
\qquad
\mu_k:=(w^\star+z_{\mathrm{eq}})_i,
\qquad
\Xi_k:=\Omega^{(i,i)}.
\]
Here \((w^\star+z_{\mathrm{eq}})_i\) denotes the \(i\)-th
\(m\)-dimensional vector block, while \(\Omega^{(i,j)}\) denotes the
\((i,j)\)-th \(m\times m\) block of \(\Omega\). These definitions are
independent of the choice of \(i\in\mathcal O_k\). The limiting regime
exhibits distributional orbit separation if and only if there exist
\(k\neq\ell\) such that
\(\mu_k\neq\mu_\ell\) or \(\Xi_k\neq\Xi_\ell\).
\end{theorem}

Distinct stationary orbit marginals represent differences in long-run
training behavior across groups, providing a distributional form of
personalization.
The proof is given in Appendix~\ref{app:proof_stationary_law}.

\subsection{From exact compatibility to approximate compatibility}
\label{subsec:reference_to_approximate}

The preceding analysis provides the reference distributions for
orbit-wise comparisons with approximately compatible systems.
Under full reference stability and sufficiently small compatibility
defects, Theorem~\ref{thm:robust_orbit_auto-personalization}
combines coefficient perturbation bounds with coupling and moment
estimates to control the stationary distribution in \(W_2\) and its
mean and covariance at order \(O(\varepsilon)\).

These full-system bounds control deviations of node marginals and within-orbit difference distributions from their reference counterparts. Equality of reference marginals within each orbit then yields approximate within-orbit marginal agreement. Pairwise cross-orbit separation persists when the corresponding reference gaps exceed the perturbation bounds. Corollary~\ref{cor:orbit_separation_margin} compares the minimum reference gap with the perturbation level to give a sufficient condition for a single threshold in \(d_{\rm par}\) to separate all within-orbit pairs from all cross-orbit pairs.  For quadratic losses, the moment bounds also enter the risk comparison. Together with perturbation estimates for local loss coefficients
and matched benchmarks, these moment bounds yield lower bounds
on both risk gains with an \(O(\varepsilon)\) reduction from
their reference guarantees, under the additional conditions of
Theorem~\ref{thm:dt-learning-robustness}.

\section{Experiments}
We conduct experiments to verify our theory.

Over the experiments, we use a four-orbit blown-up path graph with \(N=14\) agents,  as illustrated in Figure~\ref{fig:common_topology} in Appendix~\ref{app:four_orbit_quadratic}. The agents are partitioned into
\(\mathcal O_1=\{1,2\}\),
\(\mathcal O_2=\{3,4,5\}\),
\(\mathcal O_3=\{6,7,8,9\}\), and
\(\mathcal O_4=\{10,11,12,13,14\}\).
Adjacent orbits form complete bipartite blocks, with no within-orbit edges, and the quotient graph is
\(\mathcal O_1-\mathcal O_2-\mathcal O_3-\mathcal O_4\).
All perturbed systems are evaluated using this fixed reference partition.
Full details are given in
Appendices~\ref{app:four_orbit_quadratic},
\ref{app:synthetic_experiments}, and~\ref{app:real_world_experiments}.

\subsection{Experiments on synthetic data}

\paragraph{Quadratic case.}
We consider scalar quadratic losses
\(\mathcal L_i(w)=\tfrac12 Q_i(w-r_i)^2\), with
\(\eta=H=Q_i=\sigma_i=1\), \(C_0=0.1A\), and orbit-wise
minimizers \((1,0.4,-0.2,-1)\), where \(A\) is the reference
adjacency matrix.
Figure~\ref{fig:exact_stationary_marginals} illustrates the
reference orbit structure: stationary marginals coincide
within each orbit and have distinct means across orbits,
despite overlapping spreads.
Analytical verification and simulation details are provided
in Appendix~\ref{app:four_orbit_quadratic}.

\paragraph{Robustness under joint perturbations.}
We jointly perturb communication weights, local curvatures,
noise amplitudes, and minimizers along \(20\) fixed directions
at six defect levels between \(0.001\) and \(0.2\), always
evaluating the fixed reference orbit partition.
Figure~\ref{fig:orbit_separation_margin} reports the stationary
margin
\(\mathcal M_{\rm orb}=D_{\rm out}-D_{\rm in}\)
under the parameter metric of
Corollary~\ref{cor:orbit_separation_margin}.
The margin remains positive in all \(120\) evaluated
configurations: every cross-orbit distance exceeds every
within-orbit distance.
This illustrates persistence of orbit separation in the sampled
family, rather than a worst-case guarantee over all admissible
directions.
The perturbation construction, analytical bounds, and
\(O(\varepsilon)\) scaling diagnostics are provided in
Appendix~\ref{app:four_orbit_robustness}.

\begingroup
\setlength{\intextsep}{6pt}
\begin{figure}[b]
    \centering
    \captionsetup{skip=4pt}
    \captionsetup[subfigure]{
        font=small,
        labelfont=bf,
        justification=centering,
        singlelinecheck=false,
        skip=3pt
    }

    \begin{subfigure}[b]{0.34\textwidth}
        \centering
        \includegraphics[width=\linewidth]
            {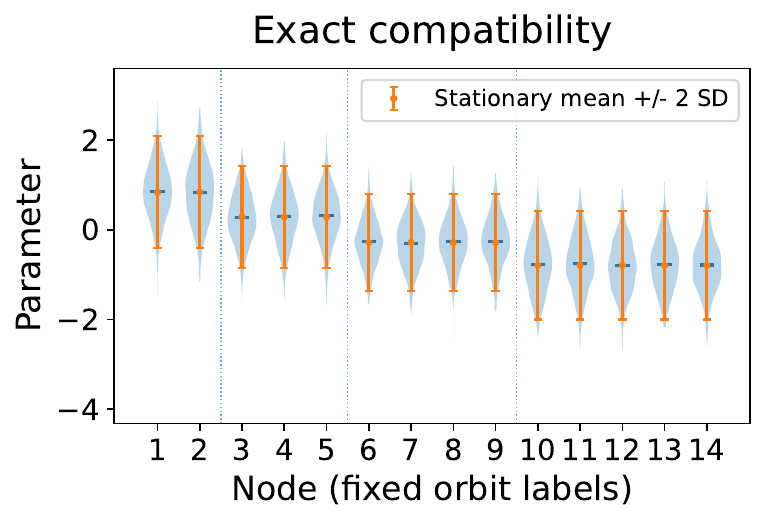}
        \caption{Stationary orbit marginals}
        \label{fig:exact_stationary_marginals}
    \end{subfigure}\hfill
    \begin{subfigure}[b]{0.3\textwidth}
        \centering
        \includegraphics[width=\linewidth]
            {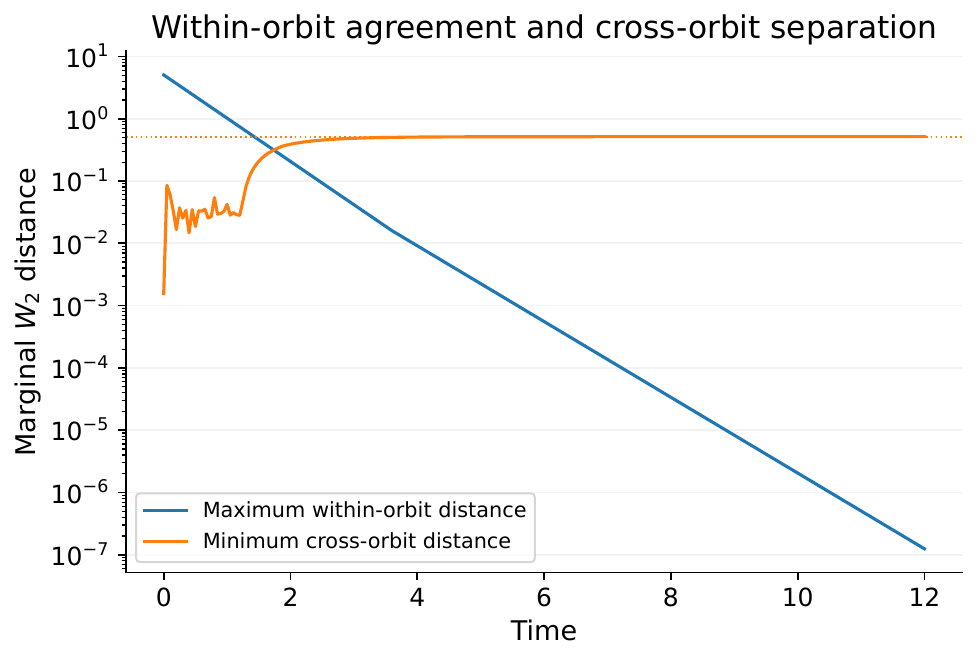}
        \caption{Reference marginal convergence}
        \label{fig:orbit-marginal-w2}
    \end{subfigure}\hfill
    \begin{subfigure}[b]{0.34\textwidth}
        \centering
        \includegraphics[width=\linewidth]
            {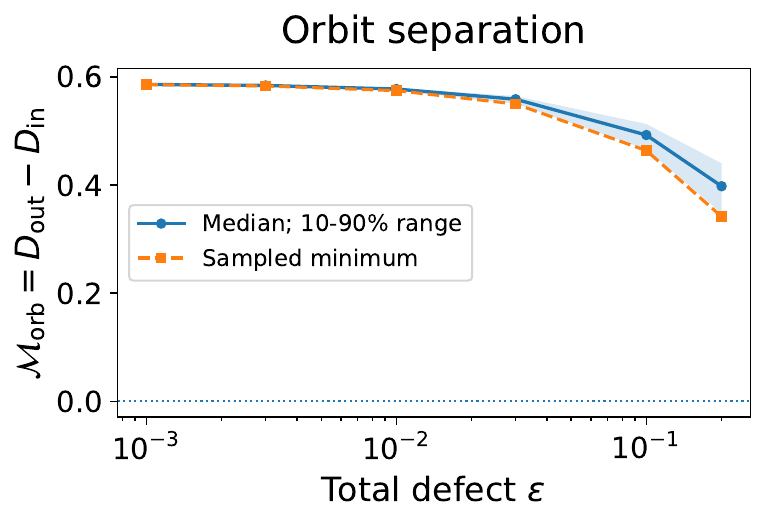}
        \caption{Robust orbit separation}
        \label{fig:orbit_separation_margin}
    \end{subfigure}

    \caption{
Orbit structure and robustness.
(a) Reference marginals with stationary means and
\(\pm2\) standard deviations.
(b) Marginal $W_2$ under exact compatibility:
within-orbit maximum and cross-orbit minimum.
(c) Separation margins under joint perturbations:
median (solid), sampled minimum (dashed), and
10th--90th percentiles (shaded).
}
    \label{fig:numerical_main}
    \vspace{-10pt}
\end{figure}
\endgroup

\paragraph{Learning advantages.}
The reference stationary excess risk is \(0.1802\), compared
with \(0.2500\) for matched independent training and \(0.2556\)
for the best single shared parameter.
Both advantages persist throughout the sampled perturbation
family.
Verification of the sufficient condition in
Theorem~\ref{thm:dt-learning-robustness} and the corresponding
risk comparisons are provided in
Appendices~\ref{app:four_orbit_risk}
and~\ref{app:four_orbit_risk_robustness}.

\paragraph{Nonquadratic case.}
A smooth, strongly convex nonquadratic example with the same
local minimizers and curvatures at those minimizers exhibits
similar finite-time marginal grouping.
This provides a qualitative illustration beyond the affine
setting; the construction, simulations, and scope of the
comparison are detailed in
Appendix~\ref{app:four_orbit_nonquadratic}.
\subsection{Experiments on real-world data}
\label{subsec:real_world_data}

\paragraph{Experimental design.}
We study orbit-aligned prediction patterns on MNIST \citep{lecun1998gradient} and CIFAR-10 \citep{krizhevsky2009learning} using MLP-LN \citep{ba2016layer} and ResNet20-GN \citep{he2016deep, wu2018group}.
The Main condition uses the graph in Figure~\ref{fig:common_topology}, with $2000$
distinct training images per agent and matching class-frequency
profiles within each reference orbit. Nodes in the same reference orbit share a class-count profile but
use distinct samples, while different orbits emphasize different classes. Data construction, training, validation-based tuning, and the
no-communication, misaligned-data, and fully symmetric control
conditions are detailed in Appendix~\ref{app:real_world_experiments}.

\paragraph{Prediction-space evaluation.}
Figure~\ref{fig:realworld_pairwise} presents final-checkpoint pairwise
squared distances between class-centered logits on a common test set.
The distance and associated dispersion measures are defined in
Appendix~\ref{app:real_world_experiments},
Eqs.~\eqref{eq:realworld_pairwise_distance}
and~\eqref{eq:realworld_dispersion}.
Agents are ordered by the reference orbits, so darker diagonal blocks
indicate relatively similar predictions within those orbits, while
lighter off-diagonal regions indicate differences across orbits.

\begin{figure}[t]
  \vspace{-6pt}
  \centering
  \captionsetup{font=small,skip=3pt}
  \captionsetup[subfigure]{font=small,labelfont=bf,skip=2pt}

  \begin{subfigure}[b]{0.24\linewidth}
    \centering
    \includegraphics[width=\linewidth]{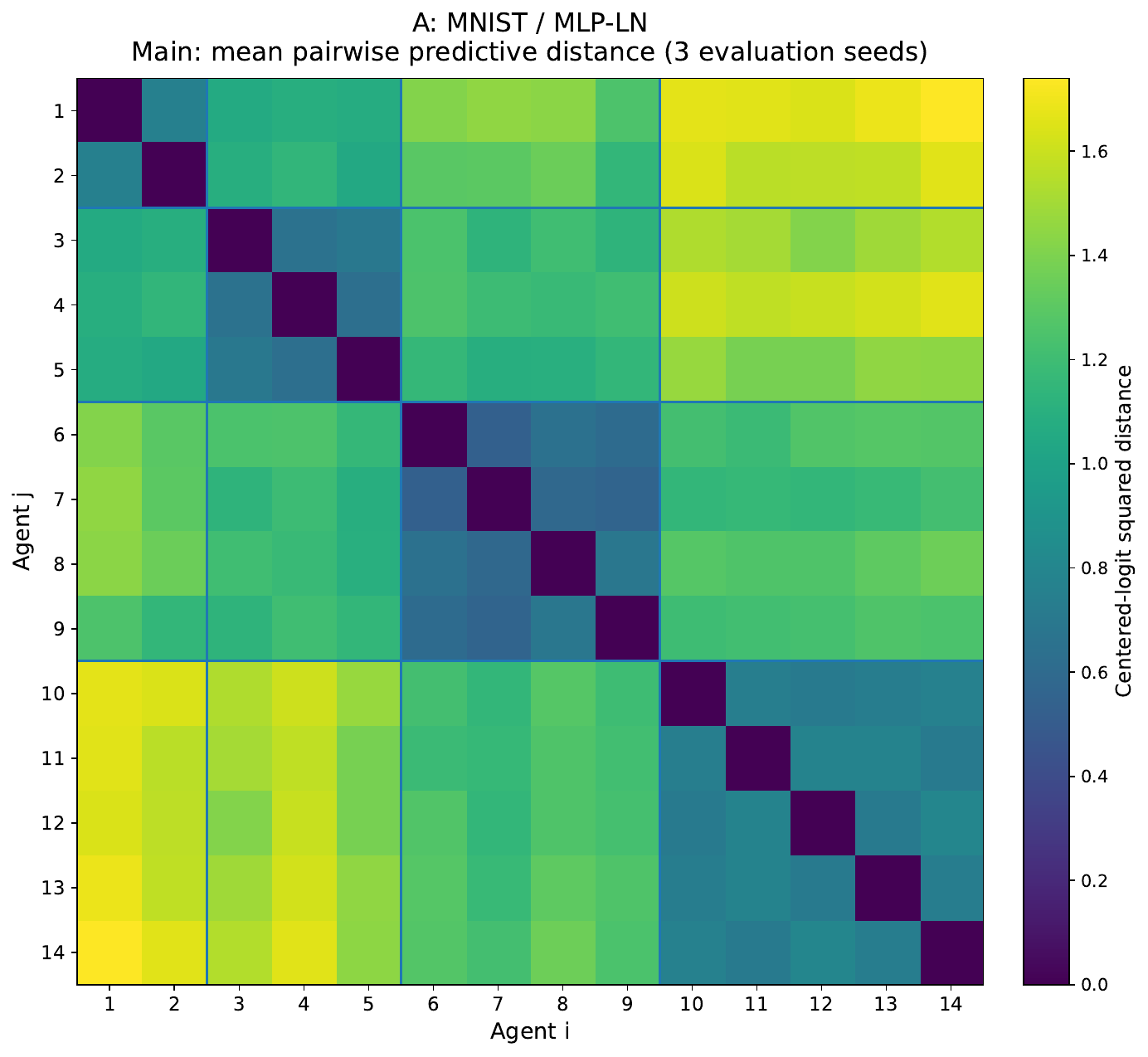}
    \caption{MNIST / MLP-LN}
    \label{fig:realworld_mnist_mlp}
  \end{subfigure}\hfill
  \begin{subfigure}[b]{0.24\linewidth}
    \centering
    \includegraphics[width=\linewidth]{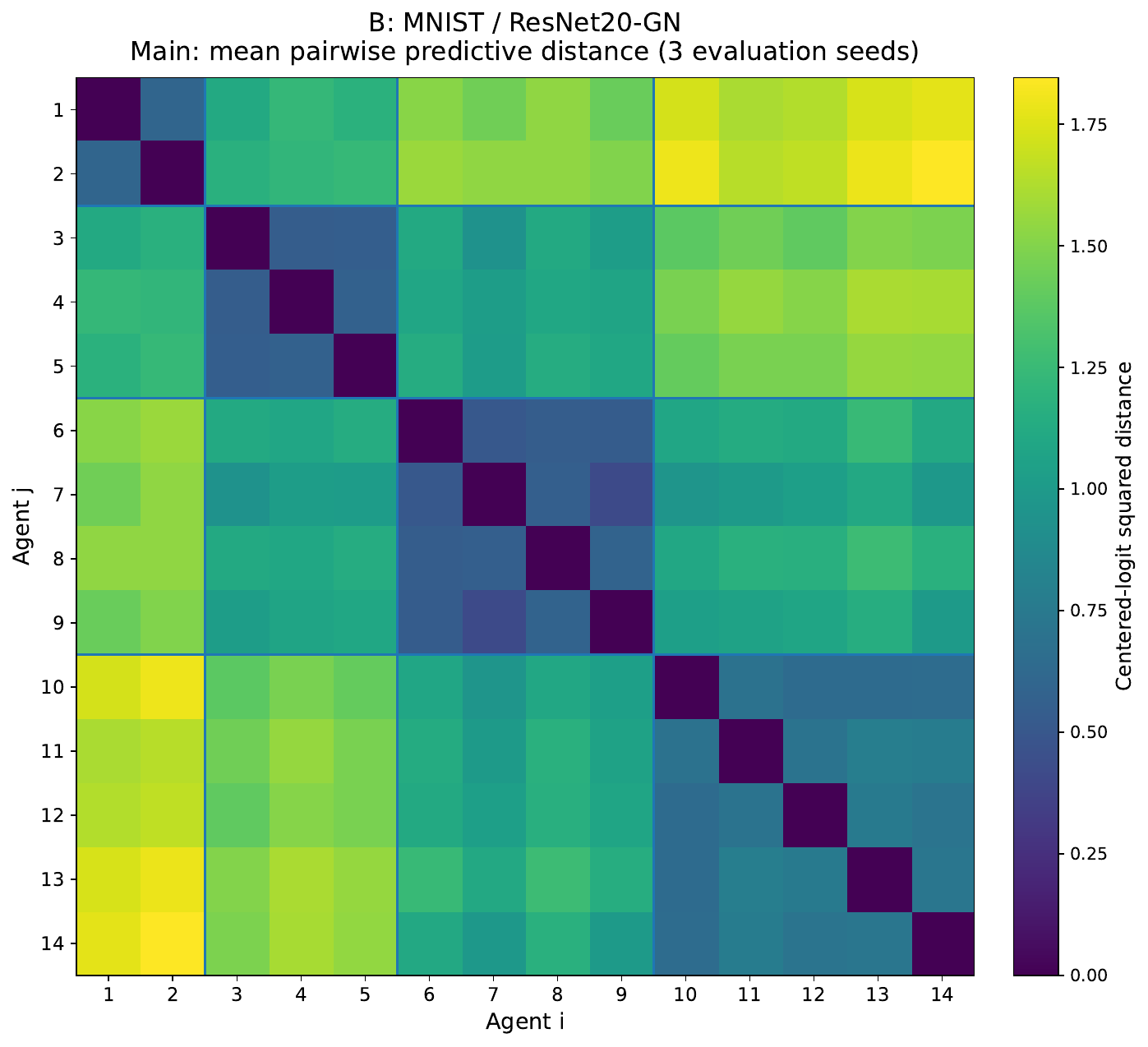}
    \caption{MNIST / ResNet}
    \label{fig:realworld_mnist_resnet}
  \end{subfigure}\hfill
  \begin{subfigure}[b]{0.24\linewidth}
    \centering
    \includegraphics[width=\linewidth]{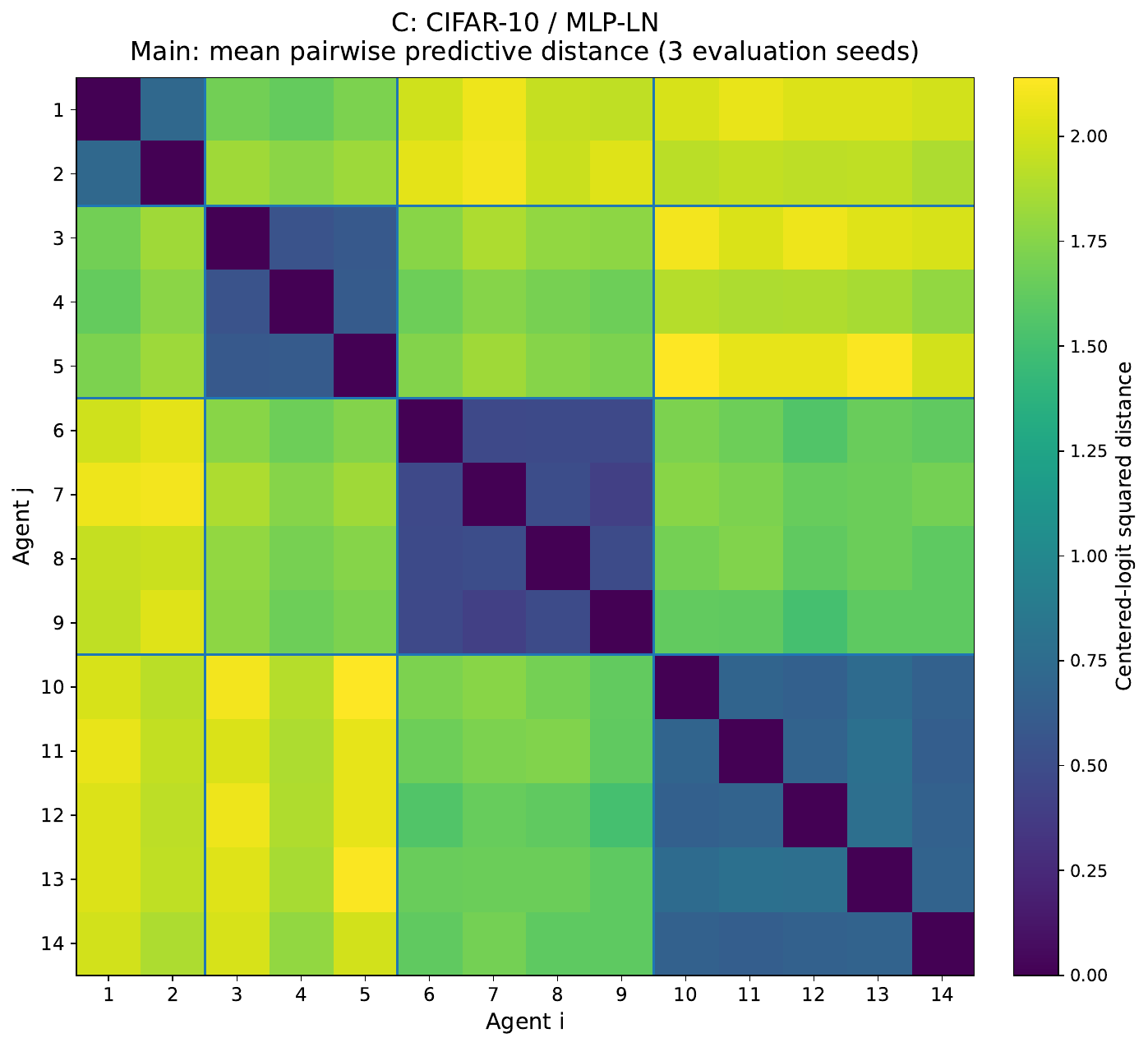}
    \caption{CIFAR-10 / MLP}
    \label{fig:realworld_cifar_mlp}
  \end{subfigure}\hfill
  \begin{subfigure}[b]{0.24\linewidth}
    \centering
    \includegraphics[width=\linewidth]{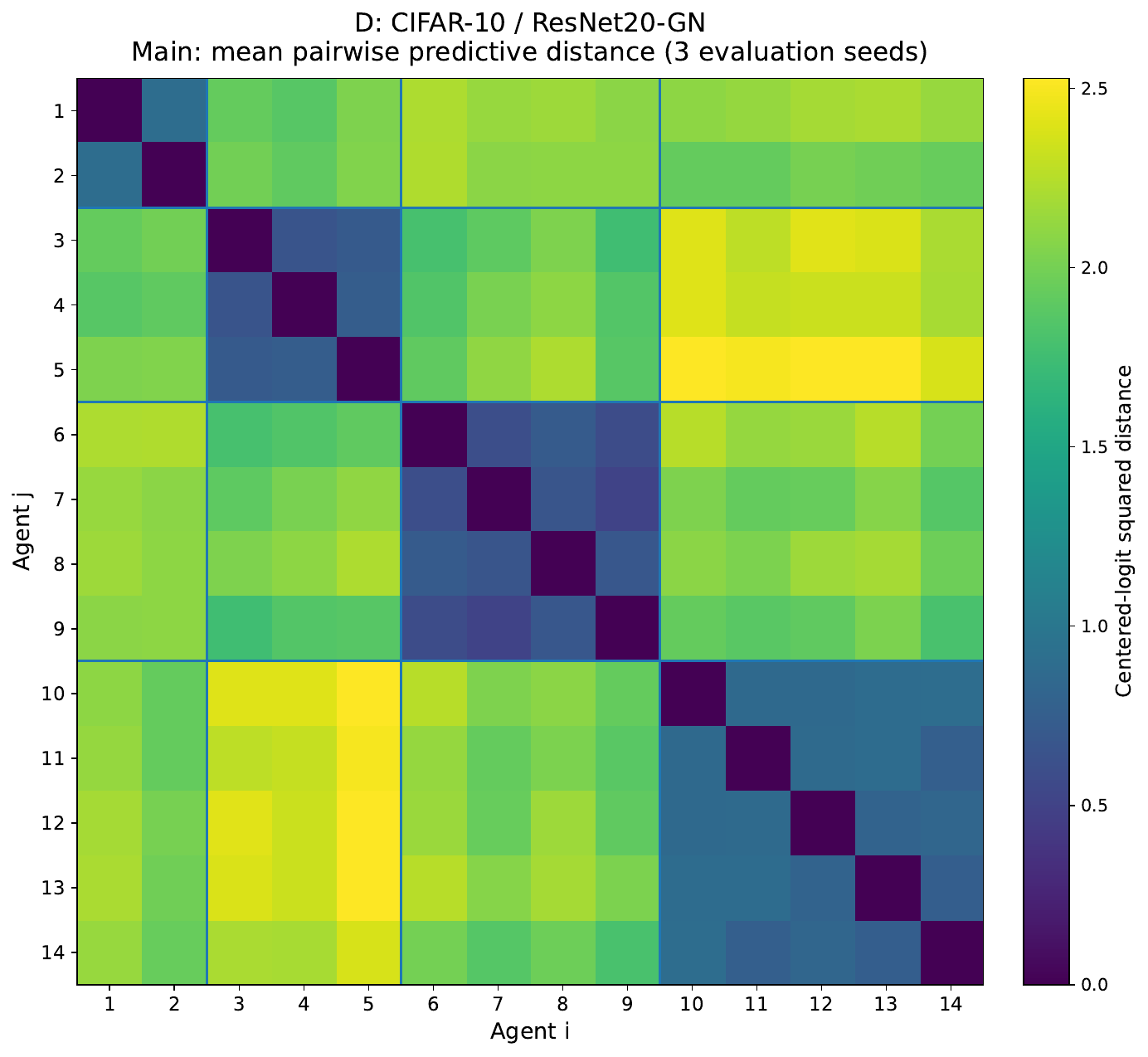}
    \caption{CIFAR-10 / ResNet}
    \label{fig:realworld_cifar_resnet}
  \end{subfigure}

  \caption{Final-checkpoint pairwise centered-logit squared distances
  in the Main condition, averaged over seeds.
  Lines mark reference-orbit boundaries; darker colors indicate
  closer predictions. Color scales vary across panels.}
    \label{fig:realworld_pairwise}
  \vspace{-8pt}
\end{figure}

\paragraph{Results.}
Figure~\ref{fig:realworld_pairwise} shows orbit-aligned regions
of relatively low predictive distance alongside visible
cross-orbit differences, despite the absence of within-orbit
communication edges.
Over the terminal evaluation window, between-orbit dispersion
exceeds within-orbit dispersion on average for all three
training seeds in each completed Main configuration.
These finite-time prediction-space observations are qualitatively
consistent with the within-orbit agreement and cross-orbit
separation characterized in
Theorem~\ref{thm:robust_orbit_auto-personalization}, parts~(iii)--(iv),
which concern stationary parameter laws.
Numerical summaries and control results are reported
under \emph{Additional results} in
Appendix~\ref{app:real_world_experiments}.

\section{Conclusion}

This work shows how alignment between local learning tasks and symmetries of the communication graph can organize decentralized learners into statistically coherent groups corresponding to reference orbits, without an explicit clustering step. Using an exactly compatible stochastic system as an analytical
reference, we characterize distributional orbit synchronization and
stationary marginals. For approximately compatible systems, we establish uniform-in-time and
stationary \(O(\varepsilon)\) perturbation bounds in 2-Wasserstein distance.
The stationary bounds yield approximate agreement within reference orbits
and explicit thresholds for preserving cross-orbit separation. By separating coupling bias from stationary fluctuation cost, we derive
sufficient conditions for lower stationary average excess risk than both
independent local training and the best single shared parameter.
Both advantages persist under sufficiently small perturbations, with
guaranteed gains reduced by at most \(O(\varepsilon)\). Synthetic experiments support these distributional and risk predictions, while experiments on MNIST and CIFAR-10 show qualitatively consistent orbit-aligned prediction patterns. 


\subsection*{AI use statement}


We have not used generative AI tools for any tasks requiring disclosure.
Additionally, we used these tools to polish the language of the manuscript and improve its grammar, clarity, and readability. We have reviewed all AI-assisted work. All AI-assisted translations and language edits were reviewed by the authors for accuracy and consistency with the intended technical meaning. We take responsibility for the final content of this work,
including text, claims or artifacts produced with the aid of generative AI.




\bibliography{iclr2027_conference}
\bibliographystyle{iclr2027_conference}

\appendix
\newpage
\section{Notation}
\label{app:notation}

\begingroup
\small
\setlength{\tabcolsep}{4pt}
\renewcommand{\arraystretch}{1.05}

\begin{longtable}{
  >{\raggedright\arraybackslash}p{0.25\linewidth}
  p{0.70\linewidth}
}
\caption{Main notation used throughout the paper.}
\label{tab:notation}\\
\toprule
\textbf{Symbol} & \textbf{Definition} \\
\midrule
\endfirsthead

\toprule
\textbf{Symbol} & \textbf{Graph automorphisms.}{Definition} \\
\midrule
\endhead

\bottomrule
\endlastfoot

\(\mathcal G=(\mathcal V,\mathcal E)\)
& Communication graph with node set \(\mathcal V\) and edge set \(\mathcal E\). \\

\(N,m\)
& Number of agents and local parameter dimension. \\

\(\Gamma\le\mathrm{Aut}(\mathcal G)\)
& Automorphism subgroup defining the graph symmetry. \\

\(\pi,P_\pi\)
& Graph automorphism and permutation matrix, with \(P_\pi e_j=e_{\pi(j)}\). \\

\(\mathcal O(i),\mathcal O_k,q,n_k\)
& Orbit of node \(i\), \(k\)-th orbit, number of orbits, and size \(n_k=|\mathcal O_k|\). \\

\(\mathcal D_i,\ell,L_i\)
& Local data distribution, sample loss, and local population risk. \\

\(w_i(t),w(t)\)
& Local model parameter and stacked parameter vector. \\

\(w_i^\star,w^\star\)
& Reference local minimizer and stacked reference minimizer. \\

\(\widetilde L_i,Q_i\)
& Local quadratic surrogate and curvature matrix of agent \(i\). \\

\(L_{\rm glob}(u)\)
& Equal-node-weighted single-model objective \(N^{-1}\sum_iL_i(u)\). \\

\(R(w),\widehat R_{\mathcal S}(w)\)
& Node-wise population and empirical risks. \\

\(\mathcal S,G_{\mathcal S}(w)\)
& Fixed training sample and generalization gap \(G_{\mathcal S}(w)=R(w)-\widehat R_{\mathcal S}(w)\). \\

\(\eta,H\)
& Learning-rate scale and parameter-space communication coupling operator. \\

\(C=[C_{ij}]\)
& Symmetric, nonnegative, graph-supported communication matrix; \(C_{ii}=0\). \\

\(\sigma_i,\Sigma\)
& Node noise intensity and stacked diffusion matrix \(\Sigma=\operatorname{blkdiag}(\sigma_1,\ldots,\sigma_N)\). \\

\(B_i(t),B(t)\)
& Independent node Brownian motions and stacked Brownian motion. \\

\(z_i(t),z(t)\)
& Centered parameter \(z_i(t)=w_i(t)-w_i^\star\) and stacked centered state. \\

\(b_i^\star,b^\star\)
& Communication bias \(b_i^\star=\sum_jC_{ij}H(w_j^\star-w_i^\star)\) and its stacked vector. \\

\(K\)
& Drift matrix of the centered SDE \(dz=(Kz+b^\star)dt+\Sigma\,dB\). \\

\(Q^{(k)},\sigma^{(k)},w^{\star(k)}\)
& Common curvature, noise intensity, and reference minimizer on orbit \(\mathcal O_k\). \\

\(T_\parallel,T_\perp,T\)
& Normalized orbit-average and transverse bases; full orthogonal matrix. \\

\(\mathcal M_\Gamma,\mathcal M_\Gamma^\perp\)
& Orbit-consensus subspace \(\operatorname{range}(T_\parallel^\top\otimes I_m)\) and its orthogonal complement. \\

\(\tilde z_\parallel(t),\tilde z_\perp(t)\)
& Parallel and transverse coordinates of \((T\otimes I_m)z(t)\). \\

\(K_\parallel,K_\perp\)
& Parallel and transverse drift matrices. \\

\(\Sigma_\parallel,\Sigma_\perp\)
& Parallel and transverse diffusion matrices. \\

\(\tilde B_\parallel(t),\tilde B_\perp(t)\)
& Independent parallel and transverse Brownian components. \\

\(\|\cdot\|,\|\cdot\|_{\mathrm F}\)
& Euclidean vector or operator matrix norm, and Frobenius matrix norm. \\

\(W_2,\Rightarrow\)
& \(2\)-Wasserstein distance between distributions, and convergence in distribution. \\

\(\theta_0,\theta_\varepsilon\)
& Exactly compatible reference and perturbed collections, including \(C_0,C_\varepsilon\). \\

\(\bar Q^{(k)},\bar\sigma^{(k)},\bar w^{\star(k)}\)
& Orbit-level quantities defining the exactly compatible reference collection. \\

\(\varepsilon,\varepsilon_0,M_{\rm par}\)
& Compatibility defect, perturbation radius, and uniform parameter bound. \\

\(K_\delta,a_\delta,\Sigma_\delta\)
& Drift, affine term, and diffusion in original coordinates, \(\delta\in\{0,\varepsilon\}\). \\

\(w_0(t),w_\varepsilon(t),\nu\)
& Reference and perturbed processes, and their common initial distribution. \\

\(\operatorname{spec}(A)\), \(A\) Hurwitz
& Spectrum of \(A\); \(A\) is Hurwitz when all eigenvalues have negative real parts. \\

\(\varrho_\delta,m_\delta,\Omega_\delta\)
& Stationary Gaussian distribution, mean, and covariance for \(\delta\in\{0,\varepsilon\}\). \\

\(W^\delta,W_i^\delta\)
& Random vector with distribution \(\varrho_\delta\) and its \(i\)-th node block. \\

\(\mu_k^0,\Xi_k^0\)
& Reference stationary marginal mean and covariance of orbit \(\mathcal O_k\). \\

\(\Delta_{k\ell}^0\)
& Reference gap \(\|\mu_k^0-\mu_\ell^0\|+\|\Xi_k^0-\Xi_\ell^0\|\) for a specified pair of distinct orbits. \\

\(C_{\rm coef}\), \(C_{\rm dyn}(\nu)\), \(C_{\rm stat}\)
& Constants in the coefficient, transient, and stationary perturbation bounds. \\

\(\Omega_\perp\)
& Transverse stationary covariance matrix. \\

\(R_{ij},\Omega_{ij}\)
& Same-orbit contrast map and limiting contrast covariance. \\

\(m_\perp(t)\), \(\operatorname{Cov}_\perp(t)\)
& Transverse mean and centered covariance matrix. \\

\(\alpha_\perp,\beta,C_\beta\)
& Transverse spectral margin, rate \(0<\beta<\alpha_\perp\), and semigroup constant. \\

\(z_{\mathrm{eq}},\Omega\)
& Exact-reference equilibrium \(-K^{-1}b^\star\) and full stationary covariance. \\

\(\mu_k,\Xi_k\)
& Exact-reference stationary marginal mean and covariance on \(\mathcal O_k\). \\

\(d_{\rm par},\nu_i^\varepsilon\)
& Gaussian parameter distance \(\|\mu-\mu^\prime\|+\|\Xi-\Xi^\prime\|\); stationary law \(\mathcal L(W_i^\varepsilon)\). \\

\(D_{\rm in}(\varepsilon),D_{\rm out}(\varepsilon)\)
& Largest within-orbit and smallest cross-orbit marginal parameter distances. \\

\(\underline\Delta_0,\mathcal M_{\rm orb}(\varepsilon)\)
& Minimal reference gap \(\min_{k<\ell}\Delta_{k\ell}^0\) and separation margin \(D_{\rm out}(\varepsilon)-D_{\rm in}(\varepsilon)\). \\

\(c_i,F_i,D_C,L_C\)
& Weighted degree, local drift, degree matrix, and weighted Laplacian. \\

\(U_\Gamma,U_\Gamma^\perp\)
& Orbit-constant node subspace and its orthogonal complement. \\

\(\widetilde P_\pi,\widetilde\Pi_\Gamma\)
& Lifted permutation and group-averaged orthogonal projector onto \(\mathcal M_\Gamma\). \\

\(T_{\mathrm{para}},T_{\mathrm{orth}}\)
& Case-study notation for \(T_\parallel,T_\perp\) and corresponding matrix blocks. \\

\(\mathrm{Gen}(t)\), \(\mathrm{Gen}_{\mathrm{para}}(t)\), \(L_{\mathrm{gen,orth}}\)
& Expected full and orbit-level generalization gaps; transverse Hessian bound. \\

\end{longtable}
\endgroup

\section{Additional literature review}
\label{app:additional_literature}

Decentralized stochastic optimization analyses study how
communication topology and data heterogeneity affect
convergence~\citep{yuan2016convergence,lian2017can,le_bars2023refined}. \citet{takezawa2026improved} shows that the full spectrum of the mixing matrix, rather than the spectral gap alone, governs the noise-driven part of the rate once the independence of stochastic gradient noise across nodes is exploited, which explains why sparse topologies perform well under near-homogeneous data. Topology-aware analyses further relate spectral properties of
the communication graph to algorithmic stability and
generalization in decentralized SGD~\citep{zhu2022topology}, while recent stability bounds quantifying how additional gossip steps shrink the generalization gap to centralized training \citep{li2026unveiling}.  Our analysis examines how compatibility between local tasks
and communication structure organizes stationary parameter
distributions across agents.

For agents with different local optima, diffusion LMS analyses
quantify the bias caused by task mismatch and identify
conditions under which cooperation improves on independent
adaptation~\citep{chen2015diffusion}.
Graph-regularized multitask learning further characterizes
how network topology and regularization strength affect
steady-state mean-square
performance~\citep{nassif2020multitaskII}.
We connect these performance questions to orbit-wise
parameter distributions and give sufficient conditions for
stationary excess-risk advantages over both independent
local training and the best single shared parameter,
with an \(O(\varepsilon)\) reduction in the guaranteed gains
under small joint perturbations of local learning quantities
and communication weights.
Complementary approaches accommodate heterogeneous tasks
by modeling task relationships in federated multi-task
learning~\citep{smith2017federated}, or by estimating client
groups and learning group-specific
models~\citep{ghosh2020efficient}. 

In server-based federated learning, local models can be coupled to a shared model \citep{dinh2020personalized,li2021ditto}, or clients can be assigned to groups that learn separate models \citep{ghosh2020efficient}. 
In decentralized settings, personalized models can be learned
over a given collaboration
graph~\citep{vanhaesebrouck2017decentralized}, jointly with
the collaboration graph~\citep{zantedeschi2020fully},
through soft clustering that accounts for different mixtures
of data clusters across clients~\citep{lin2025fedspd}, or through dynamically formed coalitions in which clients cooperate only with peers whose gradients and models are sufficiently coherent \citep{yang2026decentralized}. 
A related line shows that, under strong heterogeneity, sparse communication keeps diverging local models mergeable, so that inter-node parameter variance can act as a constructive rather than harmful component of decentralized training \citep{zhu2026surprising}.
In our analysis, the reference groups are orbits of a fixed
graph-automorphism subgroup, and the learning dynamics uses
local updates and neighbor coupling without a separate
clustering or graph-learning step.
Under approximate data-topology compatibility and reference
stability, we establish approximate agreement of stationary
marginals within reference orbits and persistence of
cross-orbit separation whenever the reference gaps exceed
the perturbation bounds.


\section{Proofs}

\subsection{Graph automorphisms}

\label{app:def:graph-automorphisms}
\paragraph{Graph automorphisms.}
Let $\mathcal G=(\mathcal V,\mathcal E)$ be a simple undirected graph with
binary adjacency matrix $A_G\in\{0,1\}^{N\times N}$. A graph automorphism is a
bijection $\pi:\mathcal V\to\mathcal V$ such that
\[
\{i,j\}\in\mathcal E
\quad\Longleftrightarrow\quad
\{\pi(i),\pi(j)\}\in\mathcal E,
\qquad \forall i,j\in\mathcal V,\ i\neq j .
\]
Equivalently,
\(
\forall i,j\in\mathcal V, (A_G)_{ij}=(A_G)_{\pi(i),\pi(j)}.
\)
The set of all graph automorphisms forms a group under composition, denoted
$\mathrm{Aut}(\mathcal G)$.

\begin{lemma}[Automorphism matrices commute with the adjacency matrix]
\label{lem:automorphism-adjacency-commutation}
Let $\mathcal G=(\mathcal V,\mathcal E)$ be a simple undirected graph with
binary adjacency matrix $A_G\in\{0,1\}^{N\times N}$. Let
$\pi\in\mathrm{Aut}(\mathcal G)$, and let $P_\pi$ be the permutation matrix
associated with $\pi$, using the convention
\[
(P_\pi)_{ij}=1
\quad\Longleftrightarrow\quad
i=\pi(j).
\]
Then
\(
P_\pi A_G=A_G P_\pi .
\)
\end{lemma}

\begin{proof}
By the convention for $P_\pi$, we have $P_\pi e_j=e_{\pi(j)}$. Hence, for any
$i,j\in\mathcal V$,
\[
(P_\pi A_G)_{ij}
=
\sum_{\ell=1}^N (P_\pi)_{i\ell}(A_G)_{\ell j}
=
(A_G)_{\pi^{-1}(i),j},
\]
because the only nonzero term in the sum occurs when $\ell=\pi^{-1}(i)$.
Similarly,
\[
(A_G P_\pi)_{ij}
=
\sum_{\ell=1}^N (A_G)_{i\ell}(P_\pi)_{\ell j}
=
(A_G)_{i,\pi(j)}.
\]
Since $\pi$ is a graph automorphism, it preserves adjacency and non-adjacency,
so
\[
(A_G)_{ab}=(A_G)_{\pi(a),\pi(b)}
\qquad
\forall a,b\in\mathcal V .
\]
Taking $a=\pi^{-1}(i)$ and $b=j$ gives
\[
(A_G)_{\pi^{-1}(i),j}
=
(A_G)_{i,\pi(j)}.
\]
Therefore
\[
(P_\pi A_G)_{ij}=(A_G P_\pi)_{ij}
\qquad
\forall i,j\in\mathcal V .
\]
Thus $P_\pi A_G=A_G P_\pi$.
\end{proof}

\subsection{Expanded form of the node-wise dynamics}
\label{app:continuous_time_dsgd}

This appendix gives the centering and stacking calculation for the dynamics in Section~\ref{preliminaries}. In original parameter coordinates, we write
\begin{equation}
dw_i(t)
=
\left[
-\eta Q_i\bigl(w_i(t)-w_i^\star\bigr)
+
\sum_{j=1}^N C_{ij}H\bigl(w_j(t)-w_i(t)\bigr)
\right]dt
+
\sigma_i\,dB_i(t).
\label{eq:app_original_coordinate_dynamics}
\end{equation}
Here the first term is induced by the local quadratic gradient model \(\nabla\widetilde L_i(w)=Q_i(w-w_i^\star)\), and the second term is the neighbor-communication term.

Let \(z_i(t):=w_i(t)-w_i^\star\). Since \(w_i^\star\) is fixed, \(dz_i(t)=dw_i(t)\), and \(w_j(t)-w_i(t)=z_j(t)-z_i(t)+w_j^\star-w_i^\star\). Substituting this identity into \eqref{eq:app_original_coordinate_dynamics} gives
\begin{equation}
dz_i(t)
=
\left[
-\eta Q_i z_i(t)
+
\sum_{j=1}^N C_{ij}H\bigl(z_j(t)-z_i(t)\bigr)
+
\sum_{j=1}^N C_{ij}H\bigl(w_j^\star-w_i^\star\bigr)
\right]dt
+
\sigma_i\,dB_i(t).
\label{eq:app_expanded_centered_dynamics}
\end{equation}

Accordingly, let \(c_i:=\sum_{j=1}^N C_{ij}\),
\(F_i:=-\eta Q_i-c_iH\), and
\(b_i^\star:=\sum_{j=1}^N C_{ij}H(w_j^\star-w_i^\star)\).
The compact centered node-wise equation is
\begin{equation}
dz_i(t)=
\left(
F_i z_i(t)+\sum_{j=1}^N C_{ij}H z_j(t)+b_i^\star
\right)dt+\sigma_i\,dB_i(t).
\label{eq:node_affine_network}
\end{equation}
Here \(F_i z_i(t)\) contains the local quadratic drift and the
self-contribution from communication,
\(\sum_j C_{ij}H z_j(t)\) collects incoming neighbor states, while
\(b_i^\star\) represents the communication bias induced by differences
among reference local minimizers. The scalar \(C_{ij}\) specifies the
graph-level communication strength, while \(H\) maps parameter differences
into update directions.

Stacking the node equations gives
\eqref{eq:stacked_centered_sde}, with \(K_{ii}=F_i\),
\(K_{ij}=C_{ij}H\) for \(i\ne j\),
\(\Sigma_{ii}=\sigma_i\), and \(\Sigma_{ij}=0\) for \(i\ne j\).
The identity \(F_i=-\eta Q_i-(\sum_j C_{ij})H\) gives the diagonal
blocks used in the main text directly. These centering and stacking identities do not require the exact
compatibility specified in Definition~\ref{ass:orbit_identical_main}.

\subsection{Proof of Theorem~\ref{thm:orbit_decomposition_main}}
\label{app:proof_decomposition}

We first record the equivariance of the affine coefficients induced by
Definition~\ref{ass:orbit_identical_main}. This algebraic step is what makes
the orbit-aligned decomposition exact.

\begin{lemma}[Equivariance of the affine coefficients]
\label{lem:commutation}
Under Definition~\ref{ass:orbit_identical_main}, let
\(\widetilde P_\pi:=P_\pi\otimes I_m\). Then, for every \(\pi\in\Gamma\),
\[
\widetilde P_\pi K=K\widetilde P_\pi,\qquad
\widetilde P_\pi\Sigma=\Sigma\widetilde P_\pi,\qquad
\widetilde P_\pi b^\star=b^\star .
\]
\end{lemma}

\begin{proof}
Fix \(\pi\in\Gamma\). Because \(Q_i\), \(\sigma_i\), and \(w_i^\star\) are constant within
each orbit and \(P_\pi C=CP_\pi\), the block matrices defining \(K\) and
\(\Sigma\) are invariant under the lifted permutation
\(\widetilde P_\pi=P_\pi\otimes I_m\). 

Let \(c_i:=\sum_{j=1}^N C_{ij}\),
\(D_C:=\operatorname{diag}(c_1,\ldots,c_N)\), and \(L_C:=D_C-C\).
Because \(\pi\) is bijective and
\(C_{\pi(i),\pi(j)}=C_{ij}\), reindexing the weighted-degree sum gives
\[
c_{\pi(i)}
=\sum_{j=1}^N C_{\pi(i),j}
=\sum_{\ell=1}^N C_{\pi(i),\pi(\ell)}
=\sum_{\ell=1}^N C_{i\ell}
=c_i.
\]
Hence \(D_C\) commutes with \(P_\pi\): both \(P_\pi D_C\) and
\(D_CP_\pi\) map \(e_i\) to \(c_i e_{\pi(i)}\). Together with
\(P_\pi C=CP_\pi\), this implies \(P_\pi L_C=L_CP_\pi\).

Since \(Q_i\) and \(c_i\) are constant within each orbit, so is
\(F_i=-\eta Q_i-c_iH\); denote its common value on \(\mathcal O_k\) by
\(F^{(k)}\). For each \(k=1,\ldots,q\), define the diagonal orbit-selector
matrix
\[
J^{(k)}
:=
\operatorname{diag}\!\left(
\mathbf 1\{1\in\mathcal O_k\},
\ldots,
\mathbf 1\{N\in\mathcal O_k\}
\right)
\in\mathbb R^{N\times N}.
\]
Since every \(\pi\in\Gamma\) maps each orbit \(\mathcal O_k\) onto itself,
\(P_\pi J^{(k)}=J^{(k)}P_\pi\). Using \(C_{ii}=0\), the block definitions
of \(K\) and \(\Sigma\) can therefore be written as
\[
K
=
\sum_{k=1}^q J^{(k)}\otimes F^{(k)}
+
C\otimes H,
\qquad
\Sigma
=
\sum_{k=1}^q J^{(k)}\otimes\sigma^{(k)}.
\]
The mixed-product rule then gives
\[
\begin{aligned}
\widetilde P_\pi K
&=\sum_{k=1}^q(P_\pi J^{(k)})\otimes F^{(k)}
  +(P_\pi C)\otimes H\\
&=\sum_{k=1}^q(J^{(k)}P_\pi)\otimes F^{(k)}
  +(CP_\pi)\otimes H
=K\widetilde P_\pi .
\end{aligned}
\]
Applying the same argument to the orbit-wise constant diffusion blocks
\(\sigma^{(k)}\) gives
\(\widetilde P_\pi\Sigma=\Sigma\widetilde P_\pi\).

The \(i\)-th block of \(-(L_C\otimes H)w^\star\) is
\(-c_iHw_i^\star+\sum_jC_{ij}Hw_j^\star
=\sum_jC_{ij}H(w_j^\star-w_i^\star)=b_i^\star\).
Thus, \(b^\star=-(L_C\otimes H)w^\star\). Moreover, the orbit-wise
equality of the reference minimizers implies
\(\widetilde P_\pi w^\star=w^\star\). Consequently,
\[
\widetilde P_\pi b^\star
=-\widetilde P_\pi(L_C\otimes H)w^\star
=-(L_C\otimes H)\widetilde P_\pi w^\star
=b^\star.
\]
Therefore,
\(\widetilde P_\pi K=K\widetilde P_\pi\),
\(\widetilde P_\pi\Sigma=\Sigma\widetilde P_\pi\), and
\(\widetilde P_\pi b^\star=b^\star\) for every \(\pi\in\Gamma\).
\end{proof}

\begin{proof}[Proof of Theorem~\ref{thm:orbit_decomposition_main}]
\phantomsection
\label{proof:orbit_decomposition_main}

By Lemma~\ref{lem:commutation}, for every \(\pi\in\Gamma\),
\(\widetilde P_\pi K=K\widetilde P_\pi\),
\(\widetilde P_\pi\Sigma=\Sigma\widetilde P_\pi\), and
\(\widetilde P_\pi b^\star=b^\star\), where
\(\widetilde P_\pi=P_\pi\otimes I_m\).

The rows of \(T_\parallel\) are orthonormal because the orbits are disjoint;
the chosen rows of \(T_\perp\) complete them to an orthonormal basis, so
\(T\) is orthogonal. Its parallel rows compute normalized orbit averages,
while its transverse rows capture deviations from those averages.
A vector is fixed by all node relabelings in \(\Gamma\) exactly when its
node blocks are constant within each orbit. Hence
\[
\mathcal M_\Gamma
=\{x\in\mathbb R^{Nm}:\widetilde P_\pi x=x
\text{ for every }\pi\in\Gamma\}
=\operatorname{range}(T_\parallel^\top\otimes I_m).
\]

Define the group-averaged operator
\[
\widetilde\Pi_\Gamma
:=
\frac{1}{|\Gamma|}
\sum_{\pi\in\Gamma}\widetilde P_\pi .
\]
We first verify that \(\widetilde\Pi_\Gamma\) is the orthogonal projector
onto \(\mathcal M_\Gamma\). Since
\(\widetilde P_\pi^\top=\widetilde P_{\pi^{-1}}\) and inversion permutes
the elements of \(\Gamma\),
\[
\widetilde\Pi_\Gamma^\top
=
\frac{1}{|\Gamma|}
\sum_{\pi\in\Gamma}\widetilde P_{\pi^{-1}}
=
\widetilde\Pi_\Gamma.
\]
Moreover, every element of \(\Gamma\) occurs exactly \(|\Gamma|\) times
among the products \(\rho\pi\), so
\[
\widetilde\Pi_\Gamma^2
=
\frac{1}{|\Gamma|^2}
\sum_{\rho,\pi\in\Gamma}\widetilde P_{\rho\pi}
=
\frac{1}{|\Gamma|}
\sum_{\tau\in\Gamma}\widetilde P_\tau
=
\widetilde\Pi_\Gamma.
\]
For any \(\rho\in\Gamma\), group closure gives
\[
\widetilde P_\rho\widetilde\Pi_\Gamma
=
\frac{1}{|\Gamma|}
\sum_{\pi\in\Gamma}\widetilde P_{\rho\pi}
=
\widetilde\Pi_\Gamma.
\]
Thus every vector in
\(\operatorname{range}(\widetilde\Pi_\Gamma)\) is fixed by all
\(\widetilde P_\rho\). Conversely, if
\(x\in\mathcal M_\Gamma\), then
\(\widetilde P_\pi x=x\) for every \(\pi\), and hence
\(\widetilde\Pi_\Gamma x=x\). Therefore,
\[
\operatorname{range}(\widetilde\Pi_\Gamma)=\mathcal M_\Gamma,
\qquad
\ker(\widetilde\Pi_\Gamma)=\mathcal M_\Gamma^\perp.
\]

Because the rows of \(T_\parallel\) form an orthonormal basis of the
orbit-constant node subspace and the rows of \(T_\perp\) span its
orthogonal complement,
\[
\widetilde\Pi_\Gamma
=
(T_\parallel^\top T_\parallel)\otimes I_m,
\qquad
I_{Nm}-\widetilde\Pi_\Gamma
=
(T_\perp^\top T_\perp)\otimes I_m.
\]

Averaging the coefficient identities in
Lemma~\ref{lem:commutation} over \(\Gamma\) gives
\[
K\widetilde\Pi_\Gamma
=
\widetilde\Pi_\Gamma K,
\qquad
\Sigma\widetilde\Pi_\Gamma
=
\widetilde\Pi_\Gamma\Sigma,
\qquad
\widetilde\Pi_\Gamma b^\star=b^\star.
\]
The first identity shows that both
\(\mathcal M_\Gamma\) and \(\mathcal M_\Gamma^\perp\) are invariant under
\(K\). Indeed, if
\(\widetilde\Pi_\Gamma x=x\), then
\(\widetilde\Pi_\Gamma Kx=Kx\); if
\(\widetilde\Pi_\Gamma x=0\), then
\(\widetilde\Pi_\Gamma Kx=0\). The same argument applies to \(\Sigma\).
Consequently, all parallel-transverse cross-blocks vanish:
\[
\begin{aligned}
(T_\parallel\otimes I_m)K(T_\perp^\top\otimes I_m)&=0,
&
(T_\perp\otimes I_m)K(T_\parallel^\top\otimes I_m)&=0,\\
(T_\parallel\otimes I_m)\Sigma(T_\perp^\top\otimes I_m)&=0,
&
(T_\perp\otimes I_m)\Sigma(T_\parallel^\top\otimes I_m)&=0.
\end{aligned}
\]
Moreover, \(\widetilde\Pi_\Gamma b^\star=b^\star\) implies
\(b^\star\in\mathcal M_\Gamma\), and therefore
\[
(T_\perp\otimes I_m)b^\star=0.
\]

Now define
\(\tilde z(t):=(T\otimes I_m)z(t)\) and
\(\tilde B(t):=(T\otimes I_m)B(t)\). Since \(T\otimes I_m\) is
orthogonal, it preserves the Gaussian independent increments and covariance
of \(B\). In particular, for \(0\le s<t\),
\[
\operatorname{Cov}\!\left(
\tilde B(t)-\tilde B(s)
\right)
=
(t-s)(T\otimes I_m)(T^\top\otimes I_m)
=
(t-s)I_{Nm}.
\]
Hence \(\tilde B\) is a standard \(Nm\)-dimensional Brownian motion. Its
parallel and transverse coordinate blocks have zero cross-covariance; since
they are jointly Gaussian, \(\tilde B_\parallel\) and
\(\tilde B_\perp\) are independent standard Brownian motions.

Multiplying \eqref{eq:stacked_centered_sde} by \(T\otimes I_m\), and using
\(z=(T^\top\otimes I_m)\tilde z\) and
\(dB=(T^\top\otimes I_m)d\tilde B\), gives
\[
\begin{aligned}
d\tilde z(t)
={}&
\Bigl[
(T\otimes I_m)K(T^\top\otimes I_m)\tilde z(t)
+(T\otimes I_m)b^\star
\Bigr]dt\\
&+
(T\otimes I_m)\Sigma(T^\top\otimes I_m)
\,d\tilde B(t).
\end{aligned}
\]
Using the vanishing cross-blocks, the definitions of
\(K_\parallel,K_\perp,\Sigma_\parallel,\Sigma_\perp\), and the identity
\((T_\perp\otimes I_m)b^\star=0\), this equation separates into
\[
\begin{aligned}
d\tilde z_\parallel(t)
&=
\bigl[
K_\parallel\tilde z_\parallel(t)
+(T_\parallel\otimes I_m)b^\star
\bigr]dt
+\Sigma_\parallel\,d\tilde B_\parallel(t),\\
d\tilde z_\perp(t)
&=
K_\perp\tilde z_\perp(t)dt
+\Sigma_\perp\,d\tilde B_\perp(t).
\end{aligned}
\]
This proves the exact dynamical decomposition.

It remains to justify its orbit-level interpretation. From the definition of
\(T_\parallel\), the \(k\)-th block of the parallel coordinate is
\[
\bigl(\tilde z_\parallel(t)\bigr)_k
=
\frac{1}{\sqrt{n_k}}
\sum_{i\in\mathcal O_k}z_i(t),
\]
which is the normalized orbit aggregate, equivalently
\(\sqrt{n_k}\) times the arithmetic orbit average.

The communication weights also induce a well-defined quotient coupling.
Indeed, let \(i,i'\in\mathcal O_k\). Then \(i'=\pi(i)\) for some
\(\pi\in\Gamma\). Because every \(\Gamma\)-orbit is invariant under
\(\pi\), we have \(\pi(\mathcal O_\ell)=\mathcal O_\ell\), and therefore
\[
\sum_{j\in\mathcal O_\ell}C_{i'j}
=
\sum_{r\in\mathcal O_\ell}C_{\pi(i),\pi(r)}
=
\sum_{r\in\mathcal O_\ell}C_{ir}.
\]
Thus the total communication weight from a node in \(\mathcal O_k\) to
\(\mathcal O_\ell\) does not depend on the selected node in
\(\mathcal O_k\). Consequently, \(C\) maps orbit-constant vectors to
orbit-constant vectors, and
\(T_\parallel C T_\parallel^\top\) is its quotient representation in the
normalized orbit basis. Together with the orbit-wise constancy of
\(Q_i\), \(\sigma_i\), and \(w_i^\star\), this shows that the parallel
equation is precisely the affine orbit-level quotient dynamics.

Finally, for \(i,j\) in the same orbit,
\((e_i-e_j)^\top T_\parallel^\top=0\). Since \(T\) is orthogonal,
\[
z_i(t)-z_j(t)
=
\left[
(e_i-e_j)^\top T_\perp^\top\otimes I_m
\right]\tilde z_\perp(t),
\]
so the transverse coordinate contains all within-orbit disagreement modes.
Definition~\ref{ass:orbit_identical_main} also gives
\(w_i^\star=w_j^\star\). Hence
\[
w_i(t)-w_j(t)
=
\bigl(z_i(t)+w_i^\star\bigr)
-
\bigl(z_j(t)+w_j^\star\bigr)
=
z_i(t)-z_j(t).
\]
This proves all claims.
\end{proof}

\subsection{Proof of Theorem~\ref{thm:robust_orbit_auto-personalization}}
\label{app:proof_approximation_error}

The proof first bounds the coefficient perturbation and establishes a common
stability margin. We then prove the four conclusions of
Theorem~\ref{thm:robust_orbit_auto-personalization} in their stated order.
The reference orbit structure is derived directly from equivariance and
uniqueness of the stationary law, without using the later reference proofs.

\begin{proof}
We first establish the coefficient and stability bounds used throughout the
proof. Let \(M_{\rm par}<\infty\) be a common uniform bound on the reference
and perturbed local quantities, as guaranteed by
Definition~\ref{ass:approx_orbit_compatibility}. For
\(i\in\mathcal O_k\), the drift perturbation has blocks
\[
\begin{aligned}
(K_\varepsilon-K_0)_{ii}
&=-\eta(Q_i-\bar Q^{(k)})
-\left\{\sum_{j=1}^N(C_\varepsilon-C_0)_{ij}\right\}H,\\
(K_\varepsilon-K_0)_{ij}
&=(C_\varepsilon-C_0)_{ij}H\qquad(i\ne j),
\end{aligned}
\]
while
\[
(a_\varepsilon-a_0)_i
=\eta\!\left[
(Q_i-\bar Q^{(k)})w_i^\star
+\bar Q^{(k)}(w_i^\star-\bar w^{\star(k)})
\right].
\]
The off-diagonal communication perturbation is
\((C_\varepsilon-C_0)\otimes H\), and the weighted-degree correction
satisfies
\[
\max_i\left|\sum_{j=1}^N(C_\varepsilon-C_0)_{ij}\right|
\le\sqrt N\,\|C_\varepsilon-C_0\|_{\mathrm F}.
\]
Consequently,
\[
\begin{aligned}
\|K_\varepsilon-K_0\|
&\le\{\eta+(1+\sqrt N)\|H\|\}\varepsilon,\\
\|a_\varepsilon-a_0\|
&\le2\eta M_{\rm par}\sqrt N\,\varepsilon,\qquad
\|\Sigma_\varepsilon-\Sigma_0\|_{\mathrm F}
\le\sqrt N\,\varepsilon.
\end{aligned}
\]
Thus, for some \(C_{\rm coef}<\infty\) independent of \(\varepsilon\),
\[
\|K_\varepsilon-K_0\|
+\|a_\varepsilon-a_0\|
+\|\Sigma_\varepsilon-\Sigma_0\|_{\mathrm F}
\le C_{\rm coef}\varepsilon.
\]

This coefficient bound does not require stability. For the remaining claims,
assume that the reference drift matrix
\(K_0\) is Hurwitz. Then there exist \(M_K\ge1\) and \(\alpha_K>0\) such
that \(\|e^{K_0t}\|\le M_Ke^{-\alpha_Kt}\) for all \(t\ge0\).
Duhamel's formula followed by Gr\"onwall's inequality gives
\[
\|e^{K_\varepsilon t}\|
\le
M_K
\exp\!\left\{
-\bigl(\alpha_K-M_K\|K_\varepsilon-K_0\|\bigr)t
\right\}.
\]
Choose \(\varepsilon_0>0\) such that
\(M_KC_{\rm coef}\varepsilon_0\le\alpha_K/2\), and set
\(\gamma:=\alpha_K/2\). Then, for every
\(0\le\varepsilon\le\varepsilon_0\),
\[
\|e^{K_\delta t}\|
\le M_Ke^{-\gamma t},
\qquad
t\ge0,\quad \delta\in\{0,\varepsilon\}.
\]
In particular, \(K_\varepsilon\) is Hurwitz.

\textbf{Part (i): Uniform-in-time distributional stability.}
Let \(X\sim\nu\) have finite second moment and be independent of the driving
Brownian motion. Couple the two systems synchronously by setting
\(w_\varepsilon(0)=w_0(0)=X\) and using the same Brownian motion. The mild
representation of the reference process and the uniform semigroup bound give
\[
\sup_{t\ge0}
\bigl(\mathbb E\|w_0(t)\|^2\bigr)^{1/2}
\le
M_K\bigl(\mathbb E\|X\|^2\bigr)^{1/2}
+\frac{M_K}{\gamma}\|a_0\|
+\frac{M_K}{\sqrt{2\gamma}}\|\Sigma_0\|_{\mathrm F}
=:M_\nu<\infty.
\]

Define \(\zeta(t):=w_\varepsilon(t)-w_0(t)\). Since \(\zeta(0)=0\),
subtracting the two affine SDEs and applying variation of constants yields
\[
\begin{aligned}
\zeta(t)
={}&
\int_0^t e^{K_\varepsilon(t-s)}
\left[
(K_\varepsilon-K_0)w_0(s)+(a_\varepsilon-a_0)
\right]ds\\
&+
\int_0^t e^{K_\varepsilon(t-s)}
(\Sigma_\varepsilon-\Sigma_0)\,dB(s).
\end{aligned}
\]
Minkowski's inequality, It\^o's isometry, and the preceding coefficient and
semigroup bounds imply
\[
\sup_{t\ge0}
\bigl(\mathbb E\|\zeta(t)\|^2\bigr)^{1/2}
\le
\left[
\frac{M_KC_{\rm coef}}{\gamma}(M_\nu+1)
+\frac{M_KC_{\rm coef}}{\sqrt{2\gamma}}
\right]\varepsilon
=:C_{\rm dyn}(\nu)\varepsilon.
\]
Because the synchronous construction is an admissible coupling,
\[
\sup_{t\ge0}
W_2\!\left(w_\varepsilon(t),w_0(t)\right)
\le C_{\rm dyn}(\nu)\varepsilon.
\]
This proves part~(i).

\textbf{Part (ii): Stationary-distribution stability.}
For \(\delta\in\{0,\varepsilon\}\), define
\[
m_\delta:=-K_\delta^{-1}a_\delta,\qquad
\Omega_\delta:=
\int_0^\infty
e^{K_\delta s}\Sigma_\delta\Sigma_\delta^\top
e^{K_\delta^\top s}\,ds.
\]
The uniform exponential bound makes the integral finite. Moreover,
\(\Omega_\delta\succeq0\), possibly singular, and it is the unique solution
of
\[
K_\delta\Omega_\delta+\Omega_\delta K_\delta^\top
+\Sigma_\delta\Sigma_\delta^\top=0.
\]
Indeed, differentiation under the integral verifies the equation, while
Hurwitz stability gives uniqueness.

Splitting the integral defining \(\Omega_\delta\) at any \(t\ge0\) and using
the variation-of-constants representation show that
\(\varrho_\delta=\mathcal N(m_\delta,\Omega_\delta)\) is stationary. For any
initial random vector \(Y\), independent of the future Brownian increments,
\[
w_\delta(t)
=
m_\delta
+
e^{K_\delta t}(Y-m_\delta)
+
\int_0^t e^{K_\delta(t-s)}\Sigma_\delta\,dB(s).
\]
The initial-condition term converges to zero almost surely, and the
stochastic convolution converges in distribution to
\(\mathcal N(0,\Omega_\delta)\). Hence every stationary distribution must
equal \(\varrho_\delta\), proving uniqueness.

Subtracting the two stationary mean equations gives
\[
m_\varepsilon-m_0
=
-K_\varepsilon^{-1}
\left[
(K_\varepsilon-K_0)m_0+(a_\varepsilon-a_0)
\right],
\qquad
\|K_\varepsilon^{-1}\|
\le\int_0^\infty M_Ke^{-\gamma t}\,dt
=\frac{M_K}{\gamma}.
\]
It follows that
\(\|m_\varepsilon-m_0\|\le C_{\rm mean}\varepsilon\) for some
\(C_{\rm mean}<\infty\) independent of \(\varepsilon\).

Next, let \(D_\delta:=\Sigma_\delta\Sigma_\delta^\top\). Uniform boundedness
and the coefficient estimate give
\[
\|D_\varepsilon-D_0\|
\le
\bigl(\|\Sigma_\varepsilon\|+\|\Sigma_0\|\bigr)
\|\Sigma_\varepsilon-\Sigma_0\|
\le C_D\varepsilon.
\]
Define
\[
\mathfrak E_\varepsilon
:=
(K_\varepsilon-K_0)\Omega_0
+\Omega_0(K_\varepsilon-K_0)^\top
+(D_\varepsilon-D_0).
\]
Then \(\|\mathfrak E_\varepsilon\|\le C_E\varepsilon\), and subtracting the
two Lyapunov equations yields
\[
\Omega_\varepsilon-\Omega_0
=
\int_0^\infty
e^{K_\varepsilon t}
\mathfrak E_\varepsilon
e^{K_\varepsilon^\top t}\,dt.
\]
Consequently,
\[
\|\Omega_\varepsilon-\Omega_0\|
\le
\frac{M_K^2}{2\gamma}\|\mathfrak E_\varepsilon\|
\le C_{\rm cov}\varepsilon
\]
for some \(C_{\rm cov}<\infty\).

It remains to control the stationary Wasserstein distance. Extend \(B\) to a
two-sided Brownian motion and construct
\[
\widehat W^\delta
:=
m_\delta+
\int_{-\infty}^0 e^{-K_\delta s}\Sigma_\delta\,dB(s),
\qquad
\widehat W^\delta\sim\varrho_\delta.
\]
For \(u\ge0\), Duhamel's formula and the uniform semigroup estimate give
\[
\|e^{K_\varepsilon u}-e^{K_0u}\|
\le
M_K^2\|K_\varepsilon-K_0\|u e^{-\gamma u},
\qquad
\|e^{K_\varepsilon u}\Sigma_\varepsilon
-e^{K_0u}\Sigma_0\|_{\mathrm F}
\le
C_G\varepsilon(1+u)e^{-\gamma u}
\]
for some \(C_G<\infty\). Using the same two-sided Brownian motion in both
stationary representations and applying It\^o's isometry therefore gives
\[
\begin{aligned}
W_2(\varrho_\varepsilon,\varrho_0)
&\le
\bigl(
\mathbb E\|\widehat W^\varepsilon-\widehat W^0\|^2
\bigr)^{1/2}\\
&\le
\|m_\varepsilon-m_0\|
+
C_G\varepsilon
\left[
\int_0^\infty(1+u)^2e^{-2\gamma u}\,du
\right]^{1/2}
\le C_W\varepsilon.
\end{aligned}
\]
Increasing \(C_{\rm stat}\) if necessary gives
\[
W_2(\varrho_\varepsilon,\varrho_0)
+\|m_\varepsilon-m_0\|
+\|\Omega_\varepsilon-\Omega_0\|
\le C_{\rm stat}\varepsilon.
\]
This proves part~(ii).

\textbf{Part (iii): Approximate agreement within reference orbits.}
Let \(W^\delta\sim\varrho_\delta\), with node blocks \(W_i^\delta\).
Lemma~\ref{lem:commutation} applied to the reference, together with the
orbit-wise constancy of the blocks \(\eta\bar Q^{(k)}\bar w^{\star(k)}\)
of \(a_0\), implies, for every \(\pi\in\Gamma\),
\[
\widetilde P_\pi K_0=K_0\widetilde P_\pi,\qquad
\widetilde P_\pi a_0=a_0,\qquad
\widetilde P_\pi\Sigma_0=\Sigma_0\widetilde P_\pi.
\]
The transformed mean and covariance therefore satisfy the same stationary
equations as \(m_0\) and \(\Omega_0\). By uniqueness,
\[
\widetilde P_\pi m_0=m_0,\qquad
\widetilde P_\pi\Omega_0\widetilde P_\pi^\top=\Omega_0.
\]
Thus, for every \(i\in\mathcal O_k\), the quantities
\(\mu_k^0:=(m_0)_i\) and \(\Xi_k^0:=\Omega_0^{(i,i)}\) are independent of
the selected representative, and the reference marginal distributions
coincide within each orbit.

Since coordinate projection is \(1\)-Lipschitz, for distinct
\(i,j\in\mathcal O_k\),
\[
W_2(W_i^\varepsilon,W_j^\varepsilon)
\le
W_2(W_i^\varepsilon,W_i^0)
+
W_2(W_j^0,W_j^\varepsilon)
\le
2W_2(\varrho_\varepsilon,\varrho_0)
\le
2C_W\varepsilon.
\]
Moreover, the map
\(x\mapsto((e_i-e_j)^\top\otimes I_m)x\) is
\(\sqrt2\)-Lipschitz, so
\[
W_2\!\left(
W_i^\varepsilon-W_j^\varepsilon,
W_i^0-W_j^0
\right)
\le
\sqrt2\,W_2(\varrho_\varepsilon,\varrho_0)
\le
\sqrt2\,C_W\varepsilon.
\]
After increasing \(C_{\rm stat}\) if necessary, these are the two
same-orbit bounds in part~(iii).

Block projection of the stationary mean and covariance bounds further gives,
for every \(i\in\mathcal O_k\),
\[
\|(m_\varepsilon)_i-\mu_k^0\|
+
\|\Omega_\varepsilon^{(i,i)}-\Xi_k^0\|
\le
C_{\rm stat}\varepsilon.
\]
This proves part~(iii).

\textbf{Part (iv): Persistence of cross-orbit separation.}
Fix any pair \(1\le k<\ell\le q\), and use the reference gap
\(\Delta_{k\ell}^0\) defined in the theorem. For arbitrary
\(i\in\mathcal O_k\) and \(j\in\mathcal O_\ell\), the triangle inequality
gives
\[
\|(m_\varepsilon)_i-(m_\varepsilon)_j\|
+
\|\Omega_\varepsilon^{(i,i)}-\Omega_\varepsilon^{(j,j)}\|
\ge
\Delta_{k\ell}^0-2C_{\rm stat}\varepsilon.
\]
If \(2C_{\rm stat}\varepsilon<\Delta_{k\ell}^0\), the right-hand side is
positive. Hence the two stationary Gaussian marginals differ in their means
or covariances. This proves part~(iv) for every reference orbit pair.

All constants are uniform over the admissible perturbed collections for
\(0\le\varepsilon\le\varepsilon_0\). The stationary constants depend only
on the fixed reference, \(N,m,\eta,H\), and \(M_{\rm par}\), whereas
\(C_{\rm dyn}(\nu)\) may additionally depend on the second moment of
\(\nu\). Taking a single enlarged \(C_{\rm stat}\) to cover the bounds in
parts~(ii)--(iv) completes the proof.
\end{proof}

\subsection{Proof of Theorem~\ref{thm:orbit_synchronization_main}}
\label{app:proof_orbit_synchronization}

We prove the synchronization criterion using the orbit-aligned decomposition
from Theorem~\ref{thm:orbit_decomposition_main}. These exact-reference
results require only transverse stability, unlike the full stability used in
Appendix~\ref{app:proof_approximation_error}. The key linear-algebraic fact
is that same-orbit contrasts span the transverse subspace.

\begin{lemma}[Same-orbit contrasts span the transverse space]
\label{lem:orbit_contrasts_span}
Let \(U_\Gamma\subseteq\mathbb R^N\) be the space of orbit-constant vectors.
Then
\[
U_\Gamma^\perp
=
\Bigl\{x\in\mathbb R^N:
\sum_{i\in\mathcal O_k}x_i=0,\ k=1,\ldots,q\Bigr\}.
\]
Moreover, \(U_\Gamma^\perp\) is spanned by the vectors \(e_i-e_j\) with
distinct \(i,j\) belonging to the same orbit.
\end{lemma}

\begin{proof}
Every \(u\in U_\Gamma\) can be written as
\(u=\sum_{k=1}^q a_k\mathbf 1_{\mathcal O_k}\), where
\(\mathbf 1_{\mathcal O_k}\) denotes the indicator vector of
\(\mathcal O_k\). Hence \(x\perp U_\Gamma\) if and only if
\[
0
=
\langle x,\mathbf 1_{\mathcal O_k}\rangle
=
\sum_{i\in\mathcal O_k}x_i,
\qquad k=1,\ldots,q.
\]
For each orbit \(\mathcal O_k\), the subspace
\[
\Bigl\{x\in\mathbb R^{\mathcal O_k}:
\sum_{i\in\mathcal O_k}x_i=0\Bigr\}
\]
is spanned by the vectors \(e_i-e_j\) with distinct
\(i,j\in\mathcal O_k\). Taking the direct sum over all orbits proves the
claim.
\end{proof}

\begin{proof}[Proof of Theorem~\ref{thm:orbit_synchronization_main}]
By Theorem~\ref{thm:orbit_decomposition_main}, the transverse coordinate
satisfies the autonomous Ornstein--Uhlenbeck equation
\[
d\tilde z_\perp(t)
=
K_\perp\tilde z_\perp(t)\,dt
+
\Sigma_\perp\,d\tilde B_\perp(t).
\]

For distinct \(i,j\in\mathcal O_k\), the conditions of
Theorem~\ref{thm:orbit_decomposition_main} give
\(w_i^\star=w_j^\star\), and hence
\(w_i(t)-w_j(t)=z_i(t)-z_j(t)\). Let
\[
d_{ij}:=e_i-e_j,
\qquad
r_{ij}:=T_\perp d_{ij},
\qquad
R_{ij}:=r_{ij}^\top\otimes I_m.
\]
Because \(d_{ij}\in U_\Gamma^\perp\), it is orthogonal to the parallel
subspace. Using
\(z(t)=(T^\top\otimes I_m)\tilde z(t)\), we obtain
\[
\begin{aligned}
w_i(t)-w_j(t)
&=
(d_{ij}^\top\otimes I_m)z(t)\\
&=
(r_{ij}^\top\otimes I_m)\tilde z_\perp(t)
=
R_{ij}\tilde z_\perp(t).
\end{aligned}
\]
Moreover,
\[
R_{ij}
=
\bigl((e_i-e_j)^\top T_\perp^\top\bigr)\otimes I_m,
\]
which agrees with the definition in
Theorem~\ref{thm:orbit_synchronization_main}.

We first prove sufficiency. Suppose that \(K_\perp\) is Hurwitz. The
variation-of-constants formula gives
\[
\tilde z_\perp(t)
=
e^{K_\perp t}\tilde z_\perp(0)
+
\int_0^t
e^{K_\perp(t-s)}\Sigma_\perp\,d\tilde B_\perp(s).
\]
For every square-integrable initial condition independent of the driving
Brownian motion,
\[
\mathbb E\!\left[
\left\|e^{K_\perp t}\tilde z_\perp(0)\right\|^2
\right]
\le
\left\|e^{K_\perp t}\right\|^2
\mathbb E\!\left[\left\|\tilde z_\perp(0)\right\|^2\right]
\longrightarrow0.
\]
The stochastic convolution is centered Gaussian with covariance
\[
\Omega_\perp(t)
=
\int_0^t
e^{K_\perp s}\Sigma_\perp\Sigma_\perp^\top
e^{K_\perp^\top s}\,ds.
\]
Hurwitz stability implies
\[
\Omega_\perp(t)\longrightarrow\Omega_\perp
:=
\int_0^\infty
e^{K_\perp s}\Sigma_\perp\Sigma_\perp^\top
e^{K_\perp^\top s}\,ds,
\]
where the integral is finite and positive semidefinite. Moreover,
\[
\begin{aligned}
K_\perp\Omega_\perp+\Omega_\perp K_\perp^\top
&=
\int_0^\infty
\frac{d}{ds}
\left(
e^{K_\perp s}\Sigma_\perp\Sigma_\perp^\top
e^{K_\perp^\top s}
\right)\,ds\\
&=
-\Sigma_\perp\Sigma_\perp^\top.
\end{aligned}
\]
Thus \(\Omega_\perp\) solves
\[
K_\perp\Omega_\perp+\Omega_\perp K_\perp^\top
+\Sigma_\perp\Sigma_\perp^\top=0.
\]

To prove uniqueness, suppose that \(\Omega_1\) and \(\Omega_2\) both solve
this equation. Then \(D:=\Omega_1-\Omega_2\) satisfies
\(K_\perp D+DK_\perp^\top=0\), and therefore
\[
e^{K_\perp t}D e^{K_\perp^\top t}=D,
\qquad t\ge0.
\]
The left-hand side converges to zero because \(K_\perp\) is Hurwitz, so
\(D=0\). Hence \(\Omega_\perp\) is the unique solution.

The variation-of-constants representation and Slutsky's theorem now give
\[
\tilde z_\perp(t)
\Rightarrow
\mathcal N(0,\Omega_\perp).
\]
Applying \(R_{ij}\) yields
\[
w_i(t)-w_j(t)
=
R_{ij}\tilde z_\perp(t)
\Rightarrow
\mathcal N(0,\Omega_{ij}),
\qquad
\Omega_{ij}:=R_{ij}\Omega_\perp R_{ij}^\top.
\]
The limiting covariance depends only on the system coefficients and is
therefore independent of the initial condition. This proves sufficiency and
the stated extension to square-integrable initial conditions.

We next prove necessity. Suppose that distributional orbit synchronization
holds for every deterministic initial state. By
Lemma~\ref{lem:orbit_contrasts_span}, the vectors \(d_{ij}=e_i-e_j\), with
distinct \(i,j\) in the same orbit, span \(U_\Gamma^\perp\). Hence the vectors
\(r_{ij}=T_\perp d_{ij}\) span \(\mathbb R^{N-q}\), and the maps
\(R_{ij}=r_{ij}^\top\otimes I_m\) separate points in
\(\mathbb R^{(N-q)m}\). Since there are finitely many same-orbit pairs, norm
equivalence gives a constant \(c_{\mathrm{obs}}<\infty\) such that
\[
\|v\|
\le
c_{\mathrm{obs}}
\max_{\substack{i\neq j\\
i,j\text{ in the same orbit}}}
\|R_{ij}v\|,
\qquad
v\in\mathbb R^{(N-q)m}.
\]

Suppose, toward a contradiction, that \(K_\perp\) is not Hurwitz. In finite
dimensions, \(K_\perp\) is Hurwitz if and only if
\(e^{K_\perp t}\to0\). Hence there exists
\(\xi_0\in\mathbb R^{(N-q)m}\) such that
\(e^{K_\perp t}\xi_0\not\to0\). Choose the deterministic initial state
\[
z(0)
=
(T^\top\otimes I_m)
\begin{bmatrix}
0\\
\xi_0
\end{bmatrix},
\qquad
\tilde z_\perp(0)=\xi_0.
\]
For this initial state,
\[
\mathbb E[\tilde z_\perp(t)]
=
e^{K_\perp t}\xi_0.
\]
The observability inequality implies that there exists a distinct same-orbit
pair \(i,j\) for which
\[
R_{ij}e^{K_\perp t}\xi_0\not\to0.
\]
Consequently,
\[
\mathbb E[w_i(t)-w_j(t)]
=
R_{ij}e^{K_\perp t}\xi_0
\not\to0.
\]

For a deterministic initial state, \(w_i(t)-w_j(t)\) is Gaussian at every
time because the dynamics are linear with additive Gaussian noise. Weak
convergence of finite-dimensional Gaussian  to a centered Gaussian law
implies convergence of their mean vectors to zero, even when the limiting
covariance is singular. This contradicts the preceding display. Therefore
\(K_\perp\) is Hurwitz, completing the proof.
\end{proof}

\subsection{Proof of Theorem~\ref{thm:transverse_moment_rates_main}}
\label{app:proof_transverse_moment_rates}

\begin{proof}
Under the standing nontrivial-orbit setting, \(K_\perp\) is nonempty. By the
conditions of Theorem~\ref{thm:orbit_synchronization_main}, it is Hurwitz,
and hence
\[
\alpha_\perp
=
-\max\{\operatorname{Re}\lambda:
\lambda\in\operatorname{spec}(K_\perp)\}
>0.
\]
For every \(\beta\in(0,\alpha_\perp)\), the standard semigroup estimates for
\(K_\perp\) and \(K_\perp^\top\) give a common constant \(C_\beta\ge1\)
such that
\[
\max\bigl\{
\|e^{K_\perp t}\|,
\|e^{K_\perp^\top t}\|
\bigr\}
\le
C_\beta e^{-\beta t},
\qquad t\ge0.
\]
Taking expectations in the transverse equation gives
\[
\|m_\perp(t)\|
=\|e^{K_\perp t}m_\perp(0)\|
\le C_\beta e^{-\beta t}\|m_\perp(0)\|.
\]
For the transverse covariance matrix
\[
\operatorname{Cov}_{\perp}(t)
=
\mathbb E\left[
(\tilde z_{\perp}(t)-m_{\perp}(t))
(\tilde z_{\perp}(t)-m_{\perp}(t))^\top
\right],
\]
the Ornstein--Uhlenbeck covariance evolution gives
\[
\frac{d}{dt}\operatorname{Cov}_{\perp}(t)
=
K_{\perp}\operatorname{Cov}_{\perp}(t)
+
\operatorname{Cov}_{\perp}(t)K_{\perp}^{\top}
+
\Sigma_{\perp}\Sigma_{\perp}^{\top}.
\]
Subtracting the Lyapunov equation for \(\Omega_\perp\) and solving the
resulting homogeneous matrix equation gives
\[
\operatorname{Cov}_{\perp}(t)-\Omega_\perp
=
e^{K_\perp t}
(\operatorname{Cov}_{\perp}(0)-\Omega_\perp)
e^{K_\perp^\top t}.
\]
Therefore,
\[
\begin{aligned}
\|\operatorname{Cov}_{\perp}(t)-\Omega_\perp\|
&\le
\|e^{K_\perp t}\|
\|\operatorname{Cov}_{\perp}(0)-\Omega_\perp\|
\|e^{K_\perp^\top t}\|\\
&\le
C_\beta^2 e^{-2\beta t}
\|\operatorname{Cov}_{\perp}(0)-\Omega_\perp\|.
\end{aligned}
\]
This proves both moment-rate bounds.
\end{proof}

\subsection{Proof of Theorem~\ref{thm:symmetry_breaking_main}}
\label{app:proof_stationary_law}

We first derive the affine equilibrium from Hurwitz stability and establish
convergence to the unique stationary Gaussian distribution. We then use the
\(\Gamma\)-equivariance of the coefficients to prove that the node marginals
are common within each orbit, after which the separation criterion follows
from the characterization of finite-dimensional Gaussian distributions.

\begin{proof}[Proof of Theorem~\ref{thm:symmetry_breaking_main}]
For \(\pi\in\Gamma\), write
\(\widetilde P_\pi:=P_\pi\otimes I_m\). Under the conditions of
Theorem~\ref{thm:orbit_decomposition_main},
Lemma~\ref{lem:commutation} gives
\[
\widetilde P_\pi K=K\widetilde P_\pi,
\qquad
\widetilde P_\pi\Sigma=\Sigma\widetilde P_\pi,
\qquad
\widetilde P_\pi b^\star=b^\star.
\]
The same conditions also imply
\(\widetilde P_\pi w^\star=w^\star\).

Since \(K\) is Hurwitz, zero is not an eigenvalue of \(K\), so \(K\) is
invertible. Consequently,
\[
z_{\mathrm{eq}}:=-K^{-1}b^\star
\]
is the unique solution of
\(Kz_{\mathrm{eq}}+b^\star=0\). Moreover, for every \(\pi\in\Gamma\),
\[
K\widetilde P_\pi z_{\mathrm{eq}}+b^\star
=
\widetilde P_\pi
\bigl(Kz_{\mathrm{eq}}+b^\star\bigr)
=
0.
\]
Uniqueness of the affine equilibrium therefore gives
\(\widetilde P_\pi z_{\mathrm{eq}}=z_{\mathrm{eq}}\) for every
\(\pi\in\Gamma\). Hence
\(z_{\mathrm{eq}}\in\mathcal M_\Gamma\).

Define the recentered process
\(y(t):=z(t)-z_{\mathrm{eq}}\). Since
\(Kz_{\mathrm{eq}}+b^\star=0\), the affine SDE for \(z(t)\) becomes
\[
dy(t)=Ky(t)\,dt+\Sigma\,dB(t),
\]
where \(B(t)\) is the stacked driving Brownian motion. Its
variation-of-constants representation is
\[
y(t)
=
e^{Kt}y(0)
+
G(t),
\qquad
G(t):=
\int_0^t e^{K(t-s)}\Sigma\,dB(s).
\]

Because \(K\) is Hurwitz, there exist constants \(M_K,c_K>0\) such that
\[
\|e^{Kt}\|
\le
M_Ke^{-c_Kt},
\qquad t\ge0.
\]
For every square-integrable initial condition independent of the driving
Brownian motion,
\[
\mathbb E\!\left[\|e^{Kt}y(0)\|^2\right]
\le
M_K^2e^{-2c_Kt}\mathbb E[\|y(0)\|^2]
\longrightarrow0.
\]
Thus \(e^{Kt}y(0)\to0\) in \(L^2\), and hence in probability. In
particular, the initial condition need not be Gaussian.

The stochastic convolution \(G(t)\) is a centered, possibly degenerate,
Gaussian vector. By It\^o's isometry, its covariance is
\[
\Omega(t)=\int_0^t e^{K(t-s)}\Sigma\Sigma^\top e^{K^\top(t-s)}\,ds
=\int_0^t e^{Ks}\Sigma\Sigma^\top e^{K^\top s}\,ds.
\]
The exponential semigroup estimate gives
\[
\bigl\|
e^{Ks}\Sigma\Sigma^\top e^{K^\top s}
\bigr\|
\le
M_K^2e^{-2c_Ks}\|\Sigma\Sigma^\top\|.
\]
Therefore
\[
\Omega
:=
\int_0^\infty
e^{Ks}\Sigma\Sigma^\top e^{K^\top s}\,ds
\]
is finite and \(\Omega(t)\to\Omega\). Since every integrand is positive
semidefinite, \(\Omega\succeq0\), although it may be singular.

We next verify the Lyapunov equation. Differentiating the integrand and using
the exponential decay at infinity gives
\[
K\Omega+\Omega K^\top
=\int_0^\infty \frac{d}{ds}
\left(e^{Ks}\Sigma\Sigma^\top e^{K^\top s}\right)\,ds
=-\Sigma\Sigma^\top.
\]
Hence
\[
K\Omega+\Omega K^\top+\Sigma\Sigma^\top=0.
\]

To prove uniqueness, suppose that \(\Omega_1\) and \(\Omega_2\) both solve
this Lyapunov equation, and let \(D:=\Omega_1-\Omega_2\). Then
\(KD+DK^\top=0\), and therefore
\[
\frac{d}{dt}
\left(e^{Kt}De^{K^\top t}\right)
=
e^{Kt}(KD+DK^\top)e^{K^\top t}
=
0.
\]
It follows that \(e^{Kt}De^{K^\top t}=D\) for every \(t\ge0\). Since \(K\)
is Hurwitz, the left-hand side converges to zero as \(t\to\infty\), and
hence \(D=0\). Thus \(\Omega\) is the unique solution of the Lyapunov
equation.

For every \(u\in\mathbb R^{Nm}\), the characteristic function of \(G(t)\)
satisfies
\[
\mathbb E\!\left[e^{\mathrm{i}u^\top G(t)}\right]
=
\exp\!\left(-\frac12u^\top\Omega(t)u\right)
\longrightarrow
\exp\!\left(-\frac12u^\top\Omega u\right).
\]
Consequently,
\(G(t)\Rightarrow\mathcal N(0,\Omega)\), including when \(\Omega\) is
singular. Combining this convergence with
\(e^{Kt}y(0)\to0\) in probability and applying Slutsky's theorem yields
\[
y(t)\Rightarrow\mathcal N(0,\Omega).
\]
Since \(w(t)=w^\star+z_{\mathrm{eq}}+y(t)\), the continuous mapping theorem
gives
\[
w(t)
\Rightarrow
\mathcal N(w^\star+z_{\mathrm{eq}},\Omega).
\]

We now show that this limiting Gaussian distribution is stationary. Splitting
the integral defining \(\Omega\) at time \(t\) gives
\[
\Omega
=
e^{Kt}\Omega e^{K^\top t}
+
\Omega(t).
\]
If \(y(0)\sim\mathcal N(0,\Omega)\) independently of the future Brownian
increments, then \(e^{Kt}y(0)\) and \(G(t)\) are independent centered
Gaussian vectors. Their covariance matrices sum to \(\Omega\), so
\(y(t)\sim\mathcal N(0,\Omega)\) for every \(t\ge0\). Hence
\(\mathcal N(0,\Omega)\) is stationary for the \(y\)-process, and its
translation
\[
\mathcal N(w^\star+z_{\mathrm{eq}},\Omega)
\]
is stationary for the \(w\)-process.

This stationary distribution is unique among all probability distributions,
not only among Gaussian distributions. Let \(\nu\) be any stationary
distribution of the \(y\)-process, and initialize
\(y_\nu(0)\sim\nu\) independently of the Brownian motion. Then
\[
y_\nu(t)
=
e^{Kt}y_\nu(0)+G(t)
\]
has distribution \(\nu\) for every \(t\ge0\). Moreover,
\[
\|e^{Kt}y_\nu(0)\|
\le
M_Ke^{-c_Kt}\|y_\nu(0)\|
\longrightarrow0
\qquad\text{almost surely}.
\]
Since \(G(t)\Rightarrow\mathcal N(0,\Omega)\), Slutsky's theorem gives
\(y_\nu(t)\Rightarrow\mathcal N(0,\Omega)\). But
\(y_\nu(t)\) has distribution \(\nu\) for every \(t\), so
\(\nu=\mathcal N(0,\Omega)\). Translation now shows that
\(\mathcal N(w^\star+z_{\mathrm{eq}},\Omega)\) is the unique stationary
distribution of the \(w\)-process.

It remains to establish the orbit structure of this distribution. Since
\(\widetilde P_\pi\) is orthogonal and commutes with \(\Sigma\),
\[
\widetilde P_\pi\Sigma\Sigma^\top\widetilde P_\pi^\top
=(\widetilde P_\pi\Sigma)(\widetilde P_\pi\Sigma)^\top
=(\Sigma\widetilde P_\pi)(\Sigma\widetilde P_\pi)^\top
=\Sigma\Sigma^\top.
\]
Consequently,
\(\widetilde P_\pi\Omega\widetilde P_\pi^\top\) satisfies
\[
K\bigl(\widetilde P_\pi\Omega\widetilde P_\pi^\top\bigr)+\bigl(\widetilde P_\pi\Omega\widetilde P_\pi^\top\bigr)K^\top+\Sigma\Sigma^\top
=\widetilde P_\pi\bigl(K\Omega+\Omega K^\top+\Sigma\Sigma^\top\bigr)\widetilde P_\pi^\top=0.
\]
By uniqueness of the Lyapunov solution,
\[
\widetilde P_\pi\Omega\widetilde P_\pi^\top=\Omega,
\qquad
\pi\in\Gamma.
\]
Together with
\(\widetilde P_\pi w^\star=w^\star\) and
\(\widetilde P_\pi z_{\mathrm{eq}}=z_{\mathrm{eq}}\), this shows that both
the limiting mean and covariance are \(\Gamma\)-invariant.

For a block vector \(x\in\mathbb R^{Nm}\), let \(x_i\) denote its \(i\)-th
\(m\)-dimensional block. For a conformably partitioned matrix
\(A\in\mathbb R^{Nm\times Nm}\), let \(A^{(i,j)}\) denote its
\((i,j)\)-th \(m\times m\) block. The convention
\(P_\pi e_i=e_{\pi(i)}\) implies
\[
(\widetilde P_\pi x)_{\pi(i)}=x_i,
\qquad
(\widetilde P_\pi A\widetilde P_\pi^\top)^{(\pi(i),\pi(j))}
=
A^{(i,j)}.
\]

If \(i,i'\in\mathcal O_k\), then there exists \(\pi\in\Gamma\) such that
\(i'=\pi(i)\). The preceding invariance identities therefore imply
\[
(w^\star+z_{\mathrm{eq}})_{i'}
=
(w^\star+z_{\mathrm{eq}})_i,
\qquad
\Omega^{(i',i')}
=
\Omega^{(i,i)}.
\]
Thus, for any \(i\in\mathcal O_k\), the quantities
\[
\mu_k:=(w^\star+z_{\mathrm{eq}})_i,
\qquad
\Xi_k:=\Omega^{(i,i)}
\]
are independent of the chosen representative \(i\).

Applying the continuous mapping theorem to the \(i\)-th node block of
\(w(t)\) gives
\[
w_i(t)
\Rightarrow
\mathcal N(\mu_k,\Xi_k),
\qquad
i\in\mathcal O_k.
\]
Therefore every orbit has a well-defined common limiting marginal
distribution.

Finally, since \(q\ge2\) by the standing assumption, the collection of
distinct orbit pairs is nonempty. Two finite-dimensional Gaussian
distributions, including singular ones, are identical if and only if their
mean vectors and covariance matrices are identical. Therefore, for any
distinct \(k,\ell\), the common limiting marginal distributions of
\(\mathcal O_k\) and \(\mathcal O_\ell\) are different if and only if
\[
\mu_k\neq\mu_\ell
\qquad\text{or}\qquad
\Xi_k\neq\Xi_\ell.
\]
Consequently, by Definition~\ref{def:orbit_separation}, the limiting regime
exhibits distributional orbit separation if and only if there exist
\(k\neq\ell\) such that
\[
\mu_k\neq\mu_\ell
\qquad\text{or}\qquad
\Xi_k\neq\Xi_\ell.
\]
This completes the proof.
\end{proof}

\subsection{Proof of Corollary~\ref{cor:orbit_separation_margin}}
\label{app:proof_auto-personalization_margin}

The block-parameter bound in
Theorem~\ref{thm:robust_orbit_auto-personalization} gives the margin
corollary by comparing every stationary node marginal with its reference
orbit marginal.

\begin{proof}[Proof of Corollary~\ref{cor:orbit_separation_margin}]
Under the stated assumptions, \eqref{eq:block_parameter_perturbation_bound}
gives, for every \(i\in\mathcal O_k\),
\[
d_{\rm par}\!\left(\nu_i^\varepsilon,\mathcal N(\mu_k^0,\Xi_k^0)\right)
\le C_{\rm stat}\varepsilon.
\]
For \(i,j\in\mathcal O_k\), the triangle inequality through their common
reference marginal yields
\[
d_{\rm par}(\nu_i^\varepsilon,\nu_j^\varepsilon)
\le 2C_{\rm stat}\varepsilon.
\]
Taking the maximum over same-orbit pairs gives
\(D_{\rm in}(\varepsilon)\le2C_{\rm stat}\varepsilon\).
For \(i\in\mathcal O_k\), \(j\in\mathcal O_\ell\), and \(k<\ell\), the
reverse triangle inequality gives
\[
d_{\rm par}(\nu_i^\varepsilon,\nu_j^\varepsilon)
\ge\Delta_{k\ell}^0-2C_{\rm stat}\varepsilon
\ge\underline\Delta_0-2C_{\rm stat}\varepsilon.
\]
Taking the minimum over cross-orbit pairs gives the stated lower bound for
\(D_{\rm out}(\varepsilon)\). Therefore
\[
\mathcal M_{\rm orb}(\varepsilon)
=D_{\rm out}(\varepsilon)-D_{\rm in}(\varepsilon)
\ge\underline\Delta_0-4C_{\rm stat}\varepsilon.
\]
By assumption, \(\underline\Delta_0>0\).
If \(0\le\varepsilon\le\varepsilon_0\) and
\(4C_{\rm stat}\varepsilon<\underline\Delta_0\), then
\[
D_{\rm in}(\varepsilon)
\le 2C_{\rm stat}\varepsilon
<\frac{\underline\Delta_0}{2}
<\underline\Delta_0-2C_{\rm stat}\varepsilon
\le D_{\rm out}(\varepsilon).
\]
Hence the separation margin is positive, and the threshold
\(\underline\Delta_0/2\) separates all within-orbit pairs from
all cross-orbit pairs in \(d_{\rm par}\).
\end{proof}

\subsection{Proof of Theorem~\ref{thm:dt-learning-robustness}}
\label{app:dt-learning-proofs}

We first derive the stationary risk decomposition, the two comparison
benchmarks, and the exact communication gain identity for a fixed
configuration with \(Q_i\succ0\), symmetric nonnegative communication
weights, and \(H+H^\top\succeq0\).
We then establish the reference gain bound and use
Theorem~\ref{thm:robust_orbit_auto-personalization} to transfer both
learning advantages to approximately compatible collections.

All risks use the quadratic evaluation losses from
Section~\ref{preliminaries}, with the local coefficients matched within
each comparison. The expectation is over the stationary parameter law;
for quadratic population losses this is their expected excess risk,
and otherwise it is the risk of the local quadratic surrogate.
The exact gain identity separates the reduction in stationary fluctuation
cost from the bias due to coupling different local optima. The coefficient
bound \(g_0\) is a sufficient certificate for a positive gain, while
\(C_{\rm risk}\varepsilon\) controls the loss of its margin under
perturbation. An explicit coefficient check is given in
Appendix~\ref{app:four_orbit_risk}.

\begin{proof}[Proof of Theorem~\ref{thm:dt-learning-robustness}]
For \(\delta\in\{0,\varepsilon\}\), write
\(Q_i^\delta,w_i^{\star,\delta}\) for the local curvature and minimizer,
and set \(\mathsf Q_\delta:=\operatorname{blkdiag}(Q_i^\delta)\) and
\(w_\delta^\star:=\operatorname{col}(w_i^{\star,\delta})\).
Thus \(Q_i^\varepsilon=Q_i\) and
\(w_i^{\star,\varepsilon}=w_i^\star\), while
\(Q_i^0=\bar Q^{(k)}\) and
\(w_i^{\star,0}=\bar w^{\star(k)}\) for \(i\in\mathcal O_k\).
Let \(\varepsilon_0\) initially be the radius supplied by
Theorem~\ref{thm:robust_orbit_auto-personalization}, and set
\(q_*:=\frac12\min_k\lambda_{\min}(\bar Q^{(k)})>0\).
After decreasing \(\varepsilon_0\) so that \(\varepsilon_0\le q_*\),
the curvature perturbation bound gives
\(Q_i^\delta\succeq q_*I_m\)
for every node and \(\delta\in\{0,\varepsilon\}\).

\paragraph{Risk decomposition and comparison benchmarks.}
Let \(e_\delta:=m_\delta-w_\delta^\star\). Expanding the quadratic
losses around \(m_\delta\), using
\(\mathbb E(W^\delta-m_\delta)=0\) and
\(\operatorname{Cov}(W^\delta)=\Omega_\delta\), gives
\begin{equation}
\mathcal E_\delta=\frac{1}{2N}\left[e_\delta^\top\mathsf Q_\delta e_\delta+\operatorname{tr}(\mathsf Q_\delta\Omega_\delta)\right].
\label{eq:dt-risk-moments}
\end{equation}
This identity applies to both the reference and perturbed coefficients.

The independent-training benchmark is obtained by setting
\(C_\delta=0\) while keeping the local objectives, \(\eta\), and
\(\Sigma_\delta\) fixed. The resulting process has drift
\(-\eta\mathsf Q_\delta(w-w_\delta^\star)\), stationary mean
\(w_\delta^\star\), and covariance \(\Omega_{\rm ind}^\delta\)
satisfying
\(\eta\mathsf Q_\delta\Omega_{\rm ind}^\delta+\eta\Omega_{\rm ind}^\delta\mathsf Q_\delta=\Sigma_\delta\Sigma_\delta^\top\).
Taking traces and applying \eqref{eq:dt-risk-moments} gives
\begin{equation}
\mathcal E_{\rm ind}^\delta=\frac{\operatorname{tr}(\mathsf Q_\delta\Omega_{\rm ind}^\delta)}{2N}=\frac{\|\Sigma_\delta\|_{\mathrm F}^2}{4N\eta}.
\label{eq:dt-independent-risk}
\end{equation}

For a common parameter \(v\in\mathbb R^m\), define the average excess
risk \(F_\delta(v):=(2N)^{-1}\sum_i(v-w_i^{\star,\delta})^\top Q_i^\delta(v-w_i^{\star,\delta})\).
The shared-parameter benchmark is
\[
\mathcal H_{\rm one}^\delta:=\min_{v\in\mathbb R^m}F_\delta(v).
\]
The unique minimizer is
\(v_\delta^\star=(\sum_iQ_i^\delta)^{-1}\sum_iQ_i^\delta w_i^{\star,\delta}\),
with \(\mathcal H_{\rm one}^\delta=F_\delta(v_\delta^\star)\).
Completing the square yields
\begin{equation}
F_\delta(v)=\mathcal H_{\rm one}^\delta+\frac{1}{2N}(v-v_\delta^\star)^\top\left(\sum_iQ_i^\delta\right)(v-v_\delta^\star).
\label{eq:dt-common-oracle}
\end{equation}
Consequently, \(\mathbb E F_\delta(Y)\ge\mathcal H_{\rm one}^\delta\)
for every random shared parameter \(Y\); positive definiteness makes
its risk infinite when \(\mathbb E\|Y\|^2=\infty\).

\paragraph{Exact communication gain.}
For either \(\delta\in\{0,\varepsilon\}\), set
\(L_{C_\delta}:=\operatorname{diag}(C_\delta\mathbf1_N)-C_\delta\)
and \(\mathsf S_\delta:=L_{C_\delta}\otimes(H+H^\top)/2\).
The graph assumptions and \(H+H^\top\succeq0\) imply
\(\mathsf S_\delta\succeq0\). Moreover,
\[
-K_\delta=\eta\mathsf Q_\delta+L_{C_\delta}\otimes H,\qquad -\frac{K_\delta+K_\delta^\top}{2}=\eta\mathsf Q_\delta+\mathsf S_\delta\succ0.
\]
Taking the trace of the stationary Lyapunov equation gives
\[
\|\Sigma_\delta\|_{\mathrm F}^2=2\eta\operatorname{tr}(\mathsf Q_\delta\Omega_\delta)+2\operatorname{tr}(\mathsf S_\delta\Omega_\delta).
\]
Combining this equality with \eqref{eq:dt-risk-moments} and
\eqref{eq:dt-independent-risk} yields the exact gain identity
\begin{equation}
\mathcal E_{\rm ind}^\delta-\mathcal E_\delta=\frac{\operatorname{tr}(\mathsf S_\delta\Omega_\delta)}{2N\eta}-\frac{e_\delta^\top\mathsf Q_\delta e_\delta}{2N}.
\label{eq:dt-communication-gain-identity}
\end{equation}
The first term is the reduction in stationary fluctuation cost, and the
second is the coupling bias. The derivation uses neither
orbit compatibility nor positive definiteness of the diffusion covariance.

\paragraph{Reference gain certificate.}
The reference centered forcing is
\(b_0^\star=-(L_{C_0}\otimes H)w_0^\star\), so the stationary
mean equation gives \(-K_0e_0=b_0^\star\).
Weighted Cauchy--Schwarz therefore gives
\[
\eta e_0^\top\mathsf Q_0e_0\le-e_0^\top K_0e_0=e_0^\top b_0^\star\le\sqrt{e_0^\top\mathsf Q_0e_0}\sqrt{(b_0^\star)^\top\mathsf Q_0^{-1}b_0^\star}.
\]
Thus
\begin{equation}
\frac{e_0^\top\mathsf Q_0e_0}{2N}\le\frac{(b_0^\star)^\top\mathsf Q_0^{-1}b_0^\star}{2N\eta^2},
\label{eq:dt-reference-bias-bound}
\end{equation}
including when \(e_0=0\).

For the covariance lower bound, set
\[
\mathsf R_0:=-\frac{K_0+K_0^\top}{2}=\eta\mathsf Q_0+\mathsf S_0\succ0.
\]
Thus \(\beta_0=\lambda_{\max}(\mathsf R_0)>0\), and
\(\Sigma_0\Sigma_0^\top\succeq d_0I\) with \(d_0>0\).
Since \(\mathsf R_0\preceq\beta_0I\),
\[
\Sigma_0\Sigma_0^\top-\frac{d_0}{\beta_0}\mathsf R_0\succeq d_0\left(I-\frac{\mathsf R_0}{\beta_0}\right)\succeq0.
\]
Since \(K_0\) is Hurwitz, integration of the derivative of
\(e^{K_0t}e^{K_0^\top t}\) gives
\(\int_0^\infty e^{K_0t}(K_0+K_0^\top)e^{K_0^\top t}\,dt=-I\).
Combining this identity with the stationary covariance
representation yields
\[
\Omega_0-\frac{d_0}{2\beta_0}I=\int_0^\infty e^{K_0t}\left[\Sigma_0\Sigma_0^\top-\frac{d_0}{\beta_0}\mathsf R_0\right]e^{K_0^\top t}\,dt\succeq0.
\]
Since \(\mathsf S_0\succeq0\), the covariance lower bound and
\eqref{eq:dt-reference-bias-bound}, applied to
\eqref{eq:dt-communication-gain-identity} at \(\varepsilon=0\), imply
\[
\mathcal E_{\rm ind}^0-\mathcal E_0\ge\frac{d_0\operatorname{tr}(\mathsf S_0)}{4N\eta\beta_0}-\frac{(b_0^\star)^\top\mathsf Q_0^{-1}b_0^\star}{2N\eta^2}=g_0.
\]
The last equality follows from \eqref{eq:dt-coefficient-gain}.
Indeed, the zero diagonal of \(C_0\) and the block-diagonal form of
\(\mathsf Q_0\) give
\[
\operatorname{tr}(\mathsf S_0)=\operatorname{tr}(H)\operatorname{tr}(L_{C_0})=\operatorname{tr}(H)\sum_{i,j=1}^N(C_0)_{ij},\qquad (b_0^\star)^\top\mathsf Q_0^{-1}b_0^\star=\sum_{k=1}^q\sum_{i\in\mathcal O_k}(b_0^\star)_i^\top(\bar Q^{(k)})^{-1}(b_0^\star)_i.
\]

\paragraph{Perturbation of the risks and benchmarks.}
In the remaining estimates, \(C<\infty\) may change from line to line
but depends only on the quantities permitted in
Theorem~\ref{thm:robust_orbit_auto-personalization}.
Definition~\ref{ass:approx_orbit_compatibility} gives
\[
\|\mathsf Q_\varepsilon-\mathsf Q_0\|\le\varepsilon,\qquad \|w_\varepsilon^\star-w_0^\star\|\le\sqrt N\,\varepsilon,\qquad \|\Sigma_\varepsilon-\Sigma_0\|_{\mathrm F}\le\sqrt N\,\varepsilon.
\]
Together with the stationary moment bounds in
Theorem~\ref{thm:robust_orbit_auto-personalization}, these imply
\(\|e_\varepsilon-e_0\|\le C\varepsilon\) and uniform bounds on
\(\|e_\delta\|\), \(\|\mathsf Q_\delta\|\), and
\(\|\Omega_\delta\|\). Expanding
\eqref{eq:dt-risk-moments}, using
\(|\operatorname{tr}(AB)|\le Nm\|A\|\|B\|\) for
\(Nm\times Nm\) matrices, and absorbing the uniform norm bounds
into \(C\), gives
\[
|\mathcal E_\varepsilon-\mathcal E_0|\le C\bigl(\|e_\varepsilon-e_0\|+\|\mathsf Q_\varepsilon-\mathsf Q_0\|+\|\Omega_\varepsilon-\Omega_0\|\bigr)\le C\varepsilon.
\]
Similarly, \eqref{eq:dt-independent-risk} gives
\[
|\mathcal E_{\rm ind}^\varepsilon-\mathcal E_{\rm ind}^0|\le\frac{\|\Sigma_\varepsilon\|_{\mathrm F}+\|\Sigma_0\|_{\mathrm F}}{4N\eta}\|\Sigma_\varepsilon-\Sigma_0\|_{\mathrm F}\le C\varepsilon.
\]

For the shared-parameter benchmark, write \(M:=M_{\rm par}\).
Since \(\sum_iQ_i^\delta\succeq Nq_*I_m\), the formula for
\(v_\delta^\star\) and the parameter bounds give
\[
\|v_\delta^\star\|\le\frac{1}{Nq_*}\sum_i\|Q_i^\delta\|\,\|w_i^{\star,\delta}\|\le\frac{M^2}{q_*}=:R.
\]
For any \(v\) with \(\|v\|\le R\), set
\(r_i^\delta:=v-w_i^{\star,\delta}\). Then
\(\|r_i^\delta\|\le R+M\) and
\(\|r_i^\varepsilon-r_i^0\|\le\varepsilon\).
Set \(C_{\rm one}:=\frac12(R+M)^2+M(R+M)\).
Expanding the quadratic forms and using these norm bounds gives
\[
|F_\varepsilon(v)-F_0(v)|\le\frac{(R+M)^2}{2N}\sum_i\|Q_i^\varepsilon-Q_i^0\|+\frac{M(R+M)}{N}\sum_i\|r_i^\varepsilon-r_i^0\|\le C_{\rm one}\varepsilon.
\]
The constant \(C_{\rm one}\) is independent of the perturbed collection.
Evaluating this bound at \(v_0^\star\) and
\(v_\varepsilon^\star\), respectively, yields
\[
\mathcal H_{\rm one}^\varepsilon\le F_\varepsilon(v_0^\star)\le\mathcal H_{\rm one}^0+C_{\rm one}\varepsilon,\qquad \mathcal H_{\rm one}^0\le F_0(v_\varepsilon^\star)\le\mathcal H_{\rm one}^\varepsilon+C_{\rm one}\varepsilon.
\]
Hence
\(|\mathcal H_{\rm one}^\varepsilon-\mathcal H_{\rm one}^0|\le C_{\rm one}\varepsilon\).

\textbf{Part (i): Benefit of communication.}
Combining the reference gain bound with the preceding perturbation
estimates gives a finite constant \(C_{\rm risk}\) such that
\[
\mathcal E_{\rm ind}^\varepsilon-\mathcal E_\varepsilon\ge\mathcal E_{\rm ind}^0-\mathcal E_0-C_{\rm risk}\varepsilon\ge g_0-C_{\rm risk}\varepsilon.
\]

\textbf{Part (ii): Benefit of task-specific parameters.}
The same estimates give, after a common enlargement of
\(C_{\rm risk}\),
\[
\mathcal H_{\rm one}^\varepsilon-\mathcal E_\varepsilon\ge\mathcal H_{\rm one}^0-\mathcal E_0-C_{\rm risk}\varepsilon\ge g_0+\mathcal H_{\rm one}^0-\mathcal E_{\rm ind}^0-C_{\rm risk}\varepsilon.
\]
The assumption on \(g_0\) makes both reference lower bounds positive.
Decrease \(\varepsilon_0\) further so that
\[
C_{\rm risk}\varepsilon_0\le\frac12\min\left\{g_0,\,g_0+\mathcal H_{\rm one}^0-\mathcal E_{\rm ind}^0\right\}.
\]
Both lower bounds are therefore strictly positive for every
\(0\le\varepsilon\le\varepsilon_0\), which establishes the two
strict risk comparisons.
Equation~\eqref{eq:dt-common-oracle} extends the second comparison to
every random shared parameter. All constants and the final
perturbation radius are uniform over the admissible collections.
\end{proof}

\section{Case study: Quadratic four-orbit system}
\label{app:four_orbit_quadratic}

\begin{figure}[htbp]
    \centering
    \includegraphics[width=\textwidth]{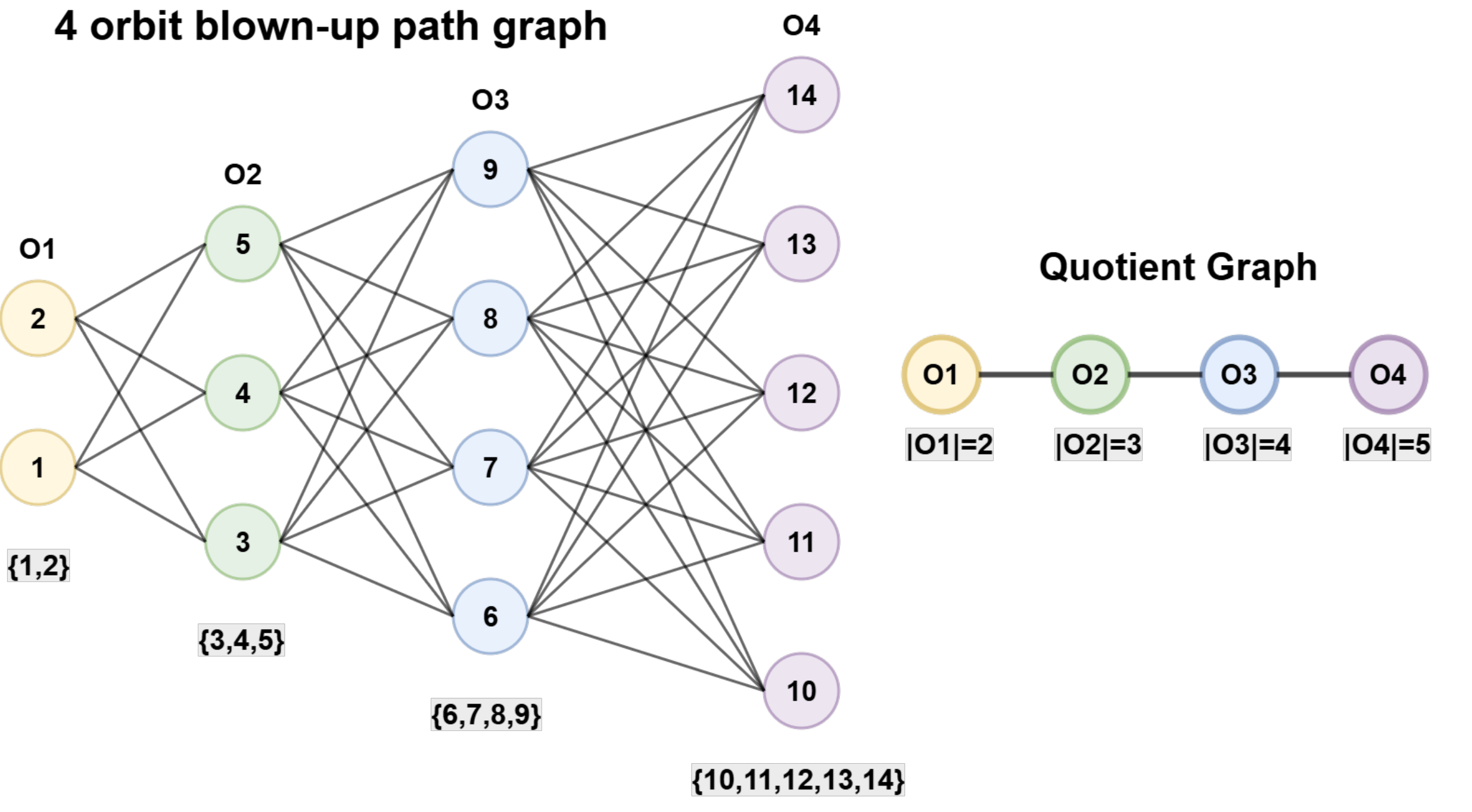}
    \caption{Four-orbit blown-up path graph and its quotient graph.
    Colors indicate automorphism orbits.}
    \label{fig:common_topology}
\end{figure}

This appendix analytically verifies the orbit structure and learning
advantages of the scalar quadratic four-orbit system.
The graph in Figure~\ref{fig:common_topology} is the common communication
topology used throughout the synthetic and real-data experiments
in the main text.
We first establish the reference decomposition and stability, then
derive the stationary orbit marginals and verify distributional
synchronization, cross-orbit separation, and transverse relaxation rates.
We next check the sufficient condition in
Theorem~\ref{thm:dt-learning-robustness}, quantify the gains over
matched independent training and the best single shared parameter,
and derive analytical bounds under joint perturbations.
Appendix~\ref{app:synthetic_experiments} presents the numerical
implementation and results, including a qualitative nonquadratic
extension, while Appendix~\ref{app:real_world_experiments}
is devoted to the real-data experiments.

\subsection{Reference orbit decomposition and stability}

The reference graph is the four-orbit blown-up path with
\[
\mathcal O_1=\{1,2\},\quad
\mathcal O_2=\{3,4,5\},\quad
\mathcal O_3=\{6,7,8,9\},\quad
\mathcal O_4=\{10,11,12,13,14\}.
\]
There are no within-orbit edges, and adjacent orbits form complete
bipartite blocks. Their sizes and unweighted degrees are
\[
(n_1,n_2,n_3,n_4)=(2,3,4,5),\qquad
(d_1,d_2,d_3,d_4)=(3,6,8,4).
\]
Every within-orbit permutation preserves adjacency, while the distinct
degrees prevent automorphisms from mapping nodes across these sets.
Thus these are exactly the automorphism orbits of the reference graph;
we take \(\Gamma=\mathrm{Aut}(\mathcal G)\).

Let \(R\in\mathbb R^{14\times4}\) be the orbit-indicator matrix,
\(D_o=R^\top R=\operatorname{diag}(2,3,4,5)\), and
\(\mathcal M_\Gamma=\operatorname{im}R\), the scalar orbit-consensus
subspace used in Theorem~\ref{thm:orbit_decomposition_main}.
Set \(T_{\parallel}=D_o^{-1/2}R^\top\), and let the rows of
\(T_{\perp}\) form an orthonormal basis of \(\mathcal M_\Gamma^\perp\).
The projections are
\(P_{\parallel}=T_{\parallel}^\top T_{\parallel}\) and
\(P_{\perp}=I-P_{\parallel}\).
Writing \(A\) for the reference adjacency matrix, we obtain
\[
\begin{gathered}
A_{\mathrm q}=D_o^{-1}R^\top A R
=
\begin{pmatrix}
0&3&0&0\\
2&0&4&0\\
0&3&0&5\\
0&0&4&0
\end{pmatrix},
\\
A_{\parallel}
=D_o^{1/2}A_{\mathrm q}D_o^{-1/2}
=
\begin{pmatrix}
0&\sqrt6&0&0\\
\sqrt6&0&2\sqrt3&0\\
0&2\sqrt3&0&2\sqrt5\\
0&0&2\sqrt5&0
\end{pmatrix}.
\end{gathered}
\]
For \(u\in \mathcal M_\Gamma^\perp\), coordinates sum to zero within each orbit,
so \(Au=0\). This eliminates the transverse adjacency contribution,
not the entire communication drift: the degree term in the Laplacian
still damps transverse modes.

The local losses and dynamics are
\begin{equation}
\mathcal L_i(w)=\tfrac12 Q_i(w-r_i)^2,\quad dw_i=\left[-Q_i(w_i-r_i) +\sum_j C_{ij}(w_j-w_i)\right]dt+\sigma_i\,dB_i,
\label{eq:app_affine_sde}
\end{equation}
where the Brownian coordinates are independent and \(\eta=H=1\).
Here \(r_i=w_i^\star\), and \(\mathcal L_i\) is the quadratic
evaluation loss used in the expected excess risk analysis.
The reference parameters are
\[
Q_{0,i}=1,\qquad \sigma_{0,i}=1,\qquad C_0=0.1A,\quad r_0=Rr_\star,\qquad r_\star=(1,0.4,-0.2,-1)^\top.
\]
With \(L_C=\operatorname{diag}(C\mathbf1)-C\), define
\[
K=-\operatorname{diag}(Q)-L_C,\qquad
a=\operatorname{diag}(Q)r,\qquad
\Sigma=\operatorname{diag}(\sigma).
\]
Then \(dw=(Kw+a)dt+\Sigma\,dB\). In particular,
\[
K_0=-I-0.1L_A,\qquad a_0=r_0,\qquad\Sigma_0=I,
\]
where \(L_A=\operatorname{diag}(A\mathbf1)-A\).
The orbit-value drift is
\[
K_{\mathrm q}=
\begin{pmatrix}
-1.3&0.3&0&0\\
0.2&-1.6&0.4&0\\
0&0.3&-1.8&0.5\\
0&0&0.4&-1.4
\end{pmatrix},
\qquad
K_{\parallel}=D_o^{1/2}K_{\mathrm q}D_o^{-1/2}.
\]
In an orbit-adapted transverse basis,
\[
K_{\perp} =\operatorname{diag}(-1.3,-1.6I_2,-1.8I_3,-1.4I_4).
\]
An arbitrary orthonormal transverse basis gives an orthogonally similar
matrix. Numerically,
\[
\operatorname{spec}(K_{\mathrm q}) \approx\{-2.243823,-1.647616,-1.208561,-1\}.
\]
These blocks verify Theorem~\ref{thm:orbit_decomposition_main} and
are Hurwitz. Independently, \(L_A\succeq0\) implies
\(K_0\preceq-I\), establishing full-system stability directly.
The transverse decay rate is
\(\alpha_{\perp}=-\lambda_{\max}(K_{\perp})=1.3\).

\subsection
{Stationary distributions and convergence rates}

The unique invariant law is \(\varrho_0=\mathcal N(m_0,\Omega_0)\), where
\[
K_0m_0+r_0=0,\qquad K_0\Omega_0+\Omega_0K_0^\top+I=0.
\]
Writing \(M_0=I+0.1L_A\), we have
\begin{equation}
m_0=M_0^{-1}r_0,\qquad
\Omega_0=\frac12M_0^{-1}.
\label{eq:app_reference_law}
\end{equation}
The mean is \(m_0=R\bar w\), with
\(\bar w=-K_{\mathrm q}^{-1}r_\star\).
In orbit-value coordinates, the orbit averages satisfy
\[
d\bar w(t)=(K_{\mathrm q}\bar w(t)+r_\star)dt +D_o^{-1/2}\,d\widetilde B(t).
\]
Thus orbit-value and orthonormal coordinates have different noise
normalizations. In shifted coordinates \(z=w-r_0\), the quotient
forcing instead becomes
\[
b_{\mathrm q}=-0.1L_{\mathrm q}r_\star =(-0.18,-0.12,-0.22,0.32)^\top,
\]
where \(L_{\mathrm q}=\operatorname{diag}(d_1,d_2,d_3,d_4)-A_{\mathrm q}\).
The shifted equilibrium is \(\bar w-r_\star\), not \(\bar w\).

The stationary node means and variances are
\[
\begin{array}{c|rr}
\text{Orbit}&\text{Mean}&\text{Variance}\\ \hline
\mathcal O_1& 0.834557&0.39060671\\
\mathcal O_2& 0.283080&0.32083706\\
\mathcal O_3&-0.284960&0.28745151\\
\mathcal O_4&-0.795703&0.36360148
\end{array}
\]
The invariant mean and covariance give identical marginals within each
orbit, as in Theorem~\ref{thm:symmetry_breaking_main}. The distinct
means establish separation of all four reference orbits; compatibility
alone does not require distinct orbit marginals.

To connect this stationary structure with transverse stability, set
\(z_\perp=T_{\perp}w=T_{\perp}(w-r_0)\), since
\(T_{\perp}r_0=0\). The affine forcing vanishes, giving
\[
dz_\perp=K_{\perp}z_\perp\,dt+dB_\perp.
\]
Its stationary covariance is
\[
\Omega_{\perp}
=\operatorname{diag}\left(
\frac{1}{2(1.3)},\frac{1}{2(1.6)}I_2,
\frac{1}{2(1.8)}I_3,\frac{1}{2(1.4)}I_4\right),
\]
with \(\operatorname{tr}\Omega_{\perp}\approx3.27152015\).
For deterministic initialization,
\[
\mathbb E z_\perp(t) =e^{K_{\perp}t}z_\perp(0),\quad \operatorname{Cov}(z_\perp(t))-\Omega_{\perp} =-e^{K_{\perp}t}\Omega_{\perp}e^{K_{\perp}^\top t}.
\]
The plotted errors are
\[
\frac{\|\mathbb E z_\perp(t)\|_2}{\|z_\perp(0)\|_2},
\qquad
\frac{\|\operatorname{Cov}(z_\perp(t))-\Omega_\perp\|_2}
{\|\Omega_\perp\|_2},
\]
where the matrix norm is the spectral norm and the chosen initialization
has nonzero transverse mean. These normalized errors are bounded
by $e^{-1.3t}$ and $e^{-2.6t}$, respectively, verifying the rates in
Theorem~\ref{thm:transverse_moment_rates_main} for this symmetric drift.
Figure~\ref{fig:app_temporal_validation}(a) displays the corresponding moment errors and envelopes.

More directly, for distinct $i,j\in\mathcal O_k$, let
$\delta_{ij}=w_i-w_j$ and $a_k=1+0.1d_k$. Their identical neighbors
and local coefficients give
\[
d\delta_{ij}=-a_k\delta_{ij}\,dt+\sqrt2\,dB_{ij},\quad \delta_{ij}(t)\sim\mathcal N\!\left( e^{-a_kt}\delta_{ij}(0),\frac{1-e^{-2a_kt}}{a_k}\right),
\]
where $B_{ij}=(B_i-B_j)/\sqrt2$ is standard Brownian motion.
The limiting difference law is centered with variance $1/a_k>0$,
verifying Theorem~\ref{thm:orbit_synchronization_main} without
pathwise agreement. In particular,
\[
\lim_{t\to\infty}\mathbb E\|P_{\perp}w(t)\|^2
=\operatorname{tr}\Omega_{\perp}\approx3.27152015>0.
\]

For scalar Gaussian marginals, the distance in
Figure~\ref{fig:orbit-marginal-w2} is
\[
W_2^2\!\left(\mathcal N(\mu_i,v_i),\mathcal N(\mu_j,v_j)\right)
=(\mu_i-\mu_j)^2+(\sqrt{v_i}-\sqrt{v_j})^2.
\]
From a deterministic initial state, the reference covariance is
invariant under within-orbit permutations at every time, so same-orbit
variances agree and
\[
W_2\!\left(\operatorname{Law}(w_i(t)),\operatorname{Law}(w_j(t))\right)
=e^{-a_kt}|w_i(0)-w_j(0)|,\qquad i,j\in\mathcal O_k.
\]
The maximum within-orbit distance therefore vanishes, whereas the minimum
cross-orbit distance tends to $0.515099$. The latter is a marginal
$W_2$ distance, distinct from the parameter gap $0.586893$ used below.
Overlapping marginal spreads in Figure~\ref{fig:exact_stationary_marginals}
are consistent with this positive separation.

\subsection{Reference risk advantages}
\label{app:four_orbit_risk}

For the reference, let $e_0=M_0^{-1}r_0-r_0$.
Equation~\eqref{eq:dt-risk-definition} gives the coupling-bias and
stationary-variance contributions
\begin{equation}
\mathcal E_0
=\frac{\|e_0\|^2}{2N}+\frac{\operatorname{tr}(M_0^{-1})}{4N}
\approx0.0119040482+0.1682692161=0.1801732643.
\label{eq:app_four_orbit_exact_risk}
\end{equation}
Matched independent training sets $C=0$ and keeps the same
$Q_i,r_i,\sigma_i,\eta$, giving $\mathcal E_{\rm ind}^0=1/4$.
The best single shared parameter is
$u_0^\star=N^{-1}\sum_i r_{0,i}\approx-0.185714$, with
\[
\mathcal H_{\rm one}^0
=\frac1{2N}\sum_i(r_{0,i}-u_0^\star)^2
\approx0.2556122449.
\]
This is the optimum over one common parameter, not a benchmark allowing
node-specific predictors.

To check Theorem~\ref{thm:dt-learning-robustness}, note that
$Q_{0,i}=1>0$, $H=1$, and $\Sigma_0\Sigma_0^\top=I$.
With $b_0^\star=-0.1L_A r_0$, $d_0=1$, and
$\beta_0=\lambda_{\max}(M_0)\approx2.243823$, its coefficient bound is
\[
g_0=\frac{\operatorname{tr}(0.1L_A)}{4N\beta_0}
-\frac{\|0.1L_A r_0\|^2}{2N}
\approx0.0604835026-0.0290571429=0.0314263597.
\]
Thus $g_0>\max\{0,\mathcal E_{\rm ind}^0-\mathcal H_{\rm one}^0\}=0$.
Both risk advantages are certified at the reference:
\[
\begin{array}{l|rr}
\text{Comparison}&\text{Exact gain}&\text{Coefficient lower bound}\\ \hline
\mathcal E_{\rm ind}^0-\mathcal E_0&0.0698267357&0.0314263597\\
\mathcal H_{\rm one}^0-\mathcal E_0&0.0754389806&0.0370386046
\end{array}
\]
The theorem therefore guarantees persistence in a sufficiently small
uniform perturbation neighborhood.

The noise level matters. If only the common noise amplitude is changed
to $s$, then
$\mathcal E_0(s)\approx0.0119040482+0.1682692161s^2$ and
$\mathcal E_{\rm ind}^0(s)=s^2/4$, while $\mathcal H_{\rm one}^0$
is unchanged. Writing $A_g=\operatorname{tr}(0.1L_A)/(4N\beta_0)$ and
$B_g=\|0.1L_A r_0\|^2/(2N)$, the sufficient condition holds on
\begin{equation}
\sqrt{B_g/A_g}<s<
\sqrt{(\mathcal H_{\rm one}^0-B_g)/(1/4-A_g)},
\qquad s\in(0.693119\ldots,1.093361\ldots).
\label{eq:app_four_orbit_noise_interval}
\end{equation}
These are sufficient-condition boundaries, not sharp risk thresholds.
For example, at $s=0.12$, communication has risk $0.014327>0.0036$
for independent training: its fixed coupling bias exceeds the variance
reduction. The reference system uses $s=1$.

\subsection{Analytical bounds under 
perturbations}
\label{app:four_orbit_perturbation_bounds}

\paragraph{Perturbation family and uniform stability.}
For $0\leq\varepsilon\leq0.2$, consider
\begin{equation}
(C_\varepsilon,Q_\varepsilon,\sigma_\varepsilon,r_\varepsilon)
=(C_0,Q_0,\sigma_0,r_0)
+\frac{\varepsilon}{2}(E,dQ,d\sigma,dr),
\label{eq:app_joint_perturbation_family}
\end{equation}
where
\[
\|E\|_{\mathrm F}=1,\qquad
\max_i(|dQ_i|+|d\sigma_i|+|dr_i|)=1.
\]
The matrix $E$ is symmetric, supported on reference edges, and has
zero sum on each adjacent-orbit block. Each local direction has zero
sum within every reference orbit. Consequently,
\[
\|C_\varepsilon-C_0\|_{\mathrm F}
+\max_i\bigl(|\Delta Q_i|+|\Delta\sigma_i|+|\Delta r_i|\bigr)
=\varepsilon.
\]
Centering preserves block totals, but need not preserve permutation
symmetry or equal weighted degrees within an orbit. All comparisons
retain the reference orbit partition.

Normalization gives $Q_{\varepsilon,i},\sigma_{\varepsilon,i}\in[0.9,1.1]$
and $|r_{\varepsilon,i}|\leq1.1$, so $M_{\rm par}=1.1$ suffices
in Definition~\ref{ass:approx_orbit_compatibility}.
Since \(E\) is symmetric with zero diagonal,
\(|E_{ij}|\leq1/\sqrt2\), so every reference edge satisfies
\[
C_{\varepsilon,ij}\geq0.1-\frac{0.2}{2\sqrt2}>0.
\]
The graph therefore remains connected, symmetric, and nonnegatively
weighted. For every vector \(v\),
\[
-v^\top K_\varepsilon v =\sum_iQ_{\varepsilon,i}v_i^2 +\frac12\sum_{i,j}C_{\varepsilon,ij}(v_i-v_j)^2 \geq0.9\|v\|^2.
\]
Thus \(K_\varepsilon\preceq-\alpha I\), with \(\alpha=0.9\),
uniformly over this normalized perturbation family. This establishes
stability without assuming exact compatibility after perturbation.

\paragraph{Distributional robustness and orbit separation.}
To control the response of the stationary distribution, write
\(L_E=\operatorname{diag}(E\mathbf1)-E\), and let \(\odot\)
denote componentwise multiplication. The coefficient changes are
\[
\Delta K=-\frac{\varepsilon}{2}\bigl(\operatorname{diag}(dQ)+L_E\bigr),\quad \Delta\Sigma=\frac{\varepsilon}{2}\operatorname{diag}(d\sigma),\quad \Delta a=\frac{\varepsilon}{2}(dr+dQ\odot r_0) +\frac{\varepsilon^2}{4}(dQ\odot dr).
\]
Hence
\(\|\Delta K\|_2+\|\Delta a\|_2+\|\Delta\Sigma\|_{\mathrm F}
=O(\varepsilon)\), uniformly over normalized directions.
The quadratic term in \(\Delta a\) arises because both curvature and
minimizer are perturbed.

Let \(\Delta m=m_\varepsilon-m_0\),
\(\Delta\Omega=\Omega_\varepsilon-\Omega_0\), and
\(S_\varepsilon=\Sigma_\varepsilon\Sigma_\varepsilon^\top\).
Subtracting the stationary equations gives
\[
\begin{gathered}
K_\varepsilon\Delta m=-(\Delta a+\Delta K m_0),
\\
K_\varepsilon\Delta\Omega+\Delta\Omega K_\varepsilon^\top
=-\bigl(\Delta K\Omega_0+\Omega_0\Delta K^\top
+S_\varepsilon-S_0\bigr).
\end{gathered}
\]
Since \(\|e^{K_\varepsilon t}\|_2\leq e^{-\alpha t}\),
\begin{equation}
\begin{aligned}
\|\Delta m\|_2&\leq
\frac{\|\Delta a\|_2+\|\Delta K\|_2\|m_0\|_2}{\alpha}=O(\varepsilon),\\
\|\Delta\Omega\|_2&\leq
\frac{2\|\Delta K\|_2\|\Omega_0\|_2+\|S_\varepsilon-S_0\|_2}{2\alpha}
=O(\varepsilon).
\end{aligned}
\label{eq:app_moment_robustness}
\end{equation}
These estimates verify the moment bounds in
Theorem~\ref{thm:robust_orbit_auto-personalization}(ii) for this
symmetric model, without a perturbed orbit decomposition.

For distributional robustness, couple the perturbed and reference
processes using the same Brownian motion and the same initial value
drawn from a common square-integrable law \(\nu\), independently of
that Brownian motion. Their difference \(Y=w_\varepsilon-w_0\) satisfies
\[
dY=\bigl(K_\varepsilon Y+\Delta K w_0+\Delta a\bigr)dt
+\Delta\Sigma\,dB.
\]
The stable reference has
\(\sup_{t\geq0}\mathbb E\|w_0(t)\|^2<\infty\).
It\^o's formula and Young's inequality yield
\[
\frac{d}{dt}\mathbb E\|Y(t)\|^2
\leq-\alpha\mathbb E\|Y(t)\|^2+C\varepsilon^2,
\]
where \(C\) may depend on the initial second moment but is independent
of \(t\), \(\varepsilon\), and the normalized direction. Since \(Y(0)=0\),
\[
W_2\bigl(\operatorname{Law}(w_\varepsilon(t)), \operatorname{Law}(w_0(t))\bigr) \leq\sqrt{\mathbb E\|Y(t)\|^2} \leq\sqrt{C/\alpha}\,\varepsilon.
\]
Passing to the invariant laws gives
\begin{equation}
W_2(\varrho_\varepsilon,\varrho_0)=O(\varepsilon).
\label{eq:app_stationary_w2}
\end{equation}
This proves the uniform-in-time and stationary $O(\varepsilon)$ bounds
in Theorem~\ref{thm:robust_orbit_auto-personalization}(i)--(ii)
for this family. For the stationary node-difference law
\(\zeta_{ij}^\varepsilon
=\operatorname{Law}_{\varrho_\varepsilon}(w_i-w_j)\),
the \(\sqrt2\)-Lipschitz difference map implies
\[
W_2(\zeta_{ij}^\varepsilon,\zeta_{ij}^0) \leq\sqrt2W_2(\varrho_\varepsilon,\varrho_0).
\]
Node projection and the triangle inequality also give, for
$i,j\in\mathcal O_k$,
\[
W_2(\nu_i^\varepsilon,\nu_j^\varepsilon)
\leq2W_2(\varrho_\varepsilon,\varrho_0)=O(\varepsilon),
\qquad \nu_i^\varepsilon=\operatorname{Law}_{\varrho_\varepsilon}(w_i).
\]
These verify both conclusions for within-orbit distributions in
Theorem~\ref{thm:robust_orbit_auto-personalization}(iii), without
requiring the difference law to concentrate at zero.
Unlike the exact reference, perturbed stationary marginals within a
reference orbit can remain distinct; their distance is controlled at
order $\varepsilon$, rather than required to vanish.

Finally, write
\(\nu_i^\varepsilon=\mathcal N(\mu_i^\varepsilon,\Xi_i^\varepsilon)\),
where \(\mu_i^\varepsilon=(m_\varepsilon)_i\) and
\(\Xi_i^\varepsilon=(\Omega_\varepsilon)_{ii}\) is the variance.
Using \(d_{\mathrm{par}}(\nu,\nu')=|\mu-\mu'|+|\Xi-\Xi'|\), define
\[
D_{\mathrm{in}}(\varepsilon) =\max_k\max_{i,j\in\mathcal O_k} d_{\mathrm{par}}(\nu_i^\varepsilon,\nu_j^\varepsilon),\quad D_{\mathrm{out}}(\varepsilon) =\min_{k<\ell}\min_{i\in\mathcal O_k,\,j\in\mathcal O_\ell} d_{\mathrm{par}}(\nu_i^\varepsilon,\nu_j^\varepsilon).
\]
The maximal displacement of a node marginal from its reference is
\[
\rho_\varepsilon
=\max_i d_{\mathrm{par}}(\nu_i^\varepsilon,\nu_i^0)
\leq\|\Delta m\|_2+\|\Delta\Omega\|_2=O(\varepsilon).
\]
With \(\underline\Delta_0=D_{\mathrm{out}}(0)\approx0.586893\), the triangle inequality
and exact within-orbit equality imply
\begin{equation}
D_{\mathrm{in}}(\varepsilon)\leq2\rho_\varepsilon,\qquad D_{\mathrm{out}}(\varepsilon)\geq\underline\Delta_0-2\rho_\varepsilon,\quad \mathcal M_{\mathrm{orb}}(\varepsilon) =D_{\mathrm{out}}(\varepsilon)-D_{\mathrm{in}}(\varepsilon) \geq\underline\Delta_0-4\rho_\varepsilon.
\label{eq:app_margin_certificate}
\end{equation}
This instance-specific margin bound follows the argument of
Corollary~\ref{cor:orbit_separation_margin}. Thus
$D_{\mathrm{in}}=O(\varepsilon)$, and separation is guaranteed
whenever \(4\rho_\varepsilon<\underline\Delta_0\). This is a sufficient condition,
not a sharp perturbation threshold. Positive margin orders parameter
distances; it does not imply disjoint supports of Gaussian marginals.

\paragraph{Perturbed risks and coefficient certificates.}
For the matched comparisons in Appendix~\ref{app:four_orbit_risk},
the perturbed stationary risks are
\[
\mathcal E_\varepsilon =\frac1{2N}\sum_iQ_{\varepsilon,i} \left[(m_{\varepsilon,i}-r_{\varepsilon,i})^2 +(\Omega_\varepsilon)_{ii}\right],\quad \mathcal E_{\rm ind}^\varepsilon =\frac1{4N}\sum_i\sigma_{\varepsilon,i}^2.
\]
The optimal shared parameter and its risk are
\[
u_\varepsilon^\star =\frac{\sum_iQ_{\varepsilon,i}r_{\varepsilon,i}} {\sum_iQ_{\varepsilon,i}},\quad \mathcal H_{\rm one}^\varepsilon =\frac1{2N}\sum_iQ_{\varepsilon,i} (r_{\varepsilon,i}-u_\varepsilon^\star)^2.
\]

To obtain a sufficient coefficient certificate, set
\(\mathsf Q_\varepsilon=\operatorname{diag}(Q_\varepsilon)\),
\(\mathsf S_\varepsilon=L_{C_\varepsilon}\), and
\(b_\varepsilon^\star=-L_{C_\varepsilon}r_\varepsilon\).
Taking the trace of the stationary covariance equation yields the
exact gain identity
\begin{equation}
\mathcal E_{\rm ind}^\varepsilon-\mathcal E_\varepsilon =\frac{\operatorname{tr}(\mathsf S_\varepsilon\Omega_\varepsilon)}{2N} \quad-\frac{(m_\varepsilon-r_\varepsilon)^\top \mathsf Q_\varepsilon(m_\varepsilon-r_\varepsilon)}{2N}.
\label{eq:app_exact_gain_identity}
\end{equation}
The two terms are the reduction in stationary fluctuation cost and
the coupling bias, respectively.
With \(d_\varepsilon=\min_i\sigma_{\varepsilon,i}^2\) and
\(\beta_\varepsilon=\lambda_{\max}(-K_\varepsilon)\), define
\[
g(\theta_\varepsilon)
=\frac{d_\varepsilon\operatorname{tr}(\mathsf S_\varepsilon)}
{4N\beta_\varepsilon}
-\frac{(b_\varepsilon^\star)^\top\mathsf Q_\varepsilon^{-1}
b_\varepsilon^\star}{2N}.
\]
The proof of the coefficient bound applies because
\(\mathsf Q_\varepsilon\succ0\), \(\mathsf S_\varepsilon\succeq0\),
and the dynamics are stable, without requiring exact compatibility.
Consequently,
\[
\mathcal E_{\rm ind}^\varepsilon-\mathcal E_\varepsilon \geq g(\theta_\varepsilon),\quad \mathcal H_{\rm one}^\varepsilon-\mathcal E_\varepsilon \geq g(\theta_\varepsilon) +\mathcal H_{\rm one}^\varepsilon-\mathcal E_{\rm ind}^\varepsilon.
\]
Define the joint certificate by
\begin{equation}
\gamma(\theta_\varepsilon)
=g(\theta_\varepsilon)
-\max\{0,\mathcal E_{\rm ind}^\varepsilon
-\mathcal H_{\rm one}^\varepsilon\}.
\label{eq:app_joint_risk_certificate}
\end{equation}
A positive value certifies both advantages for the specified
coefficients. This certificate is distinct from the perturbative
bound \(g_0-C_{\rm risk}\varepsilon\).

\section{Additional details of synthetic experiments}
\label{app:synthetic_experiments}

We evaluate the quadratic reference, its joint perturbations, and a
nonquadratic extension on the common four-orbit graph.
All settings and numerical values below refer to the default full run.
The quadratic calculations illustrate the analytical conclusions in
Appendix~\ref{app:four_orbit_quadratic}; the nonquadratic experiment
provides a qualitative comparison beyond the affine setting.

\subsection{Simulation and evaluation}
\label{app:four_orbit_simulation_settings}

The synthetic studies use prescribed population drifts and diffusion,
rather than training on a finite dataset. The graph, quadratic losses,
and reference coefficients are specified in
Appendix~\ref{app:four_orbit_quadratic}.
Only the reference and one representative perturbed affine system are
simulated with $600$ independent full-system realizations each;
the $120$-configuration perturbation sweep uses analytical Gaussian laws.

For the two simulated affine systems, the time horizon is $T=30$ and
the common deterministic initial vector is drawn once as
\[
\texttt{default\_rng(2027).normal(0,1.5,14)}.
\]
Simulation seeds are $30001$ for the reference and $30002$ for the
representative perturbation ($\varepsilon=0.1$, direction seed $101$;
see Appendix~\ref{app:four_orbit_robustness}).
At observation interval $h=0.02$, exact Ornstein--Uhlenbeck transitions are
\[
\begin{gathered}
w_{n+1}=m+F_h(w_n-m)+\xi_n,\qquad F_h=e^{Kh},\qquad
\xi_n\stackrel{\mathrm{iid}}{\sim}\mathcal N(0,V_h),
\\
V_h=\int_0^h e^{Ks}\Sigma\Sigma^\top e^{K^\top s}\,ds
=\Omega-F_h\Omega F_h^\top.
\end{gathered}
\]
Thus $h$ introduces no time-discretization bias in the affine simulations.
For each configuration, stationary moments solve
$Km+a=0$ and $K\Omega+\Omega K^\top+\Sigma\Sigma^\top=0$.
The exact finite-time parameters are
\[
m(t)=m+e^{Kt}(w(0)-m),\qquad
V(t)=\Omega-e^{Kt}\Omega e^{K^\top t}.
\]

Stationary orbit separation is evaluated using $d_{\rm par}$,
$D_{\rm in}$, $D_{\rm out}$, and $\mathcal M_{\rm orb}$ from
Appendix~\ref{app:four_orbit_perturbation_bounds}; scalar marginal
$W_2$ uses the standard-deviation difference as specified in
Appendix~\ref{app:four_orbit_quadratic}.
For the time-dependent marginals $\nu_i(t)=\operatorname{Law}(w_i(t))$,
define
\[
D_{\mathrm{in},W_2}(t) =\max_k\max_{i,j\in\mathcal O_k}W_2(\nu_i(t),\nu_j(t)),\quad D_{\mathrm{out},W_2}(t) =\min_{k<\ell}\min_{i\in\mathcal O_k,\,j\in\mathcal O_\ell} W_2(\nu_i(t),\nu_j(t)).
\]
These temporal marginal distances are distinct from the stationary
parameter distances $D_{\mathrm{in}}$ and $D_{\mathrm{out}}$.
For the representative perturbation, we evaluate all $14$ node means
and both temporal distances from the exact finite-time laws on
$t=0,0.05,\ldots,30$.
The joint stationary-law distance is
\[
W_2^2(\varrho_\varepsilon,\varrho_0)
=\|m_\varepsilon-m_0\|^2+
\operatorname{tr}\!\left(
\Omega_\varepsilon+\Omega_0
-2(\Omega_0^{1/2}\Omega_\varepsilon\Omega_0^{1/2})^{1/2}
\right).
\]
For numerical evaluation of Gaussian $W_2$, the code uses the equivalent
orthogonal-Procrustes formula, computing the covariance contribution
from an SVD-aligned difference of matrix square roots. This avoids
subtracting nearly equal covariance traces.
All stationary sweep distances, risks, and certificates are evaluated
from the analytical moments and coefficients, up to numerical
linear-algebra error.

For each simulated affine case, terminal Monte Carlo checks compare
$14$ means and $14$ variances with their exact finite-time values,
using Bonferroni-adjusted $95\%$ intervals over all $28$ quantities.
For risk, each independent sampling unit is an entire $14$-node
realization: losses are first averaged over nodes and then over the
$600$ realizations. Approximate $95\%$ Student-$t$ intervals are
computed across realizations, preserving dependence among nodes.

\subsection{Quadratic validation and robustness}
\label{app:four_orbit_quadratic_results}

\paragraph{Reference validation.}
\label{app:four_orbit_reference_validation}

Figure~\ref{fig:exact_stationary_marginals} uses the $600$ reference
endpoints, with analytical stationary means and $\pm2$ standard
deviations; these bars are not confidence intervals.
Figure~\ref{fig:orbit-marginal-w2} instead uses exact finite-time
Gaussian marginals on $t=0,0.05,\ldots,12$; the code additionally
evaluates the terminal law at $t=30$.
The within-orbit difference formulas in
Appendix~\ref{app:four_orbit_quadratic} agree with the full moment
calculation to within $8\times10^{-15}$ on this grid.
The largest within-orbit marginal $W_2$ distance decreases towards zero, while
the smallest cross-orbit distance approaches $0.515099$, consistently
with the reference stationary law.
Figure~\ref{fig:app_temporal_validation}(a) additionally shows the
normalized transverse mean norm and covariance error on
$t=0,0.05,\ldots,8$, together with the bounds $e^{-1.3t}$ and
$e^{-2.6t}$ derived in Appendix~\ref{app:four_orbit_quadratic}.

The estimated terminal reference risk is $0.182051$, with approximate
$95\%$ Monte Carlo interval $[0.176366,0.187735]$, which contains
the exact stationary risk $0.180173$.
The computed joint $W_2$ distance from the terminal law to stationarity
is below $2\times10^{-13}$, and the absolute finite-time risk bias is
below $10^{-15}$. These checks therefore primarily reflect sampling
variability. The analytical comparisons with independent training and
the best shared parameter are given in
Appendix~\ref{app:four_orbit_risk}.

\paragraph{Perturbation results.}
\label{app:four_orbit_robustness}

We generate directions for the family in
\eqref{eq:app_joint_perturbation_family} as follows.
For each direction, entries of each forward adjacent-orbit block are
drawn independently and uniformly from $[-1,1]$, centered by the block
mean, and reflected to obtain a symmetric matrix $E$.
We then divide $E$ by its Frobenius norm.
The three local direction columns $(dQ,d\sigma,dr)$ are drawn uniformly
from $[-1,1]$,
centered within each reference orbit, weighted by $(0.30,0.10,0.60)$,
and divided by the largest node-wise sum of absolute values.

The $20$ direction seeds are $101,\ldots,120$, each reused at
$\varepsilon\in\{0.001,0.003,0.01,0.03,0.1,0.2\}$, giving
$120$ evaluated configurations.

The default-run results are
\[
\begin{array}{l|rrr}
\text{Setting}&D_{\mathrm{in}}&D_{\mathrm{out}}&\mathcal M_{\mathrm{orb}}\\ \hline
\text{Exact reference}&0&0.586893&0.586893\\
\varepsilon=0.1,\ \text{seed }101&0.040085&0.551993&0.511908
\end{array}
\]
The reference within-orbit discrepancy is exactly zero; numerical
residuals are below $10^{-14}$.
All \(120\) sampled systems have positive margins, with minimum
\(0.341617\). The lower certificate in
\eqref{eq:app_margin_certificate} is also positive in all $120$
evaluated configurations,
with minimum
\(0.253167\). These certificates concern the evaluated coefficients,
not the worst case over all admissible perturbations.

\begin{figure}[t]
    \centering
    \captionsetup{font=small,skip=3pt}
    \captionsetup[subfigure]{font=small,labelfont=bf,skip=1pt}
    \begin{subfigure}[b]{0.44\linewidth}
        \centering
        \includegraphics[width=\linewidth]{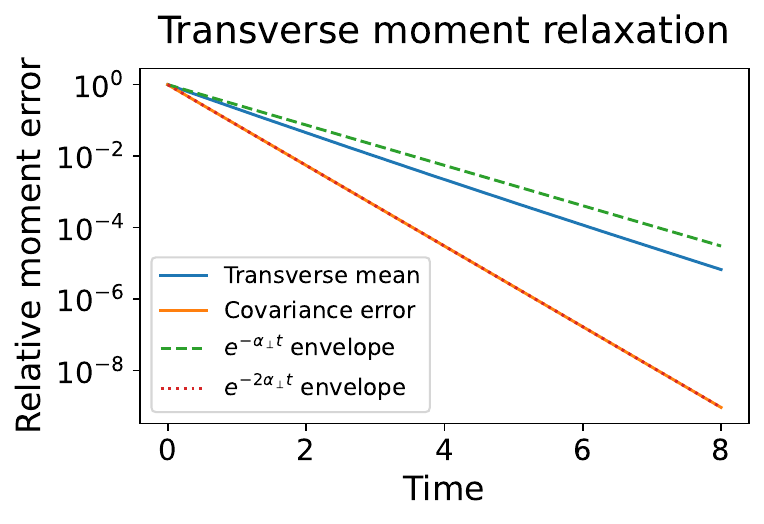}
        \caption{Reference transverse relaxation}
    \end{subfigure}\hfill
    \begin{subfigure}[b]{0.44\linewidth}
        \centering
        \includegraphics[width=\linewidth]{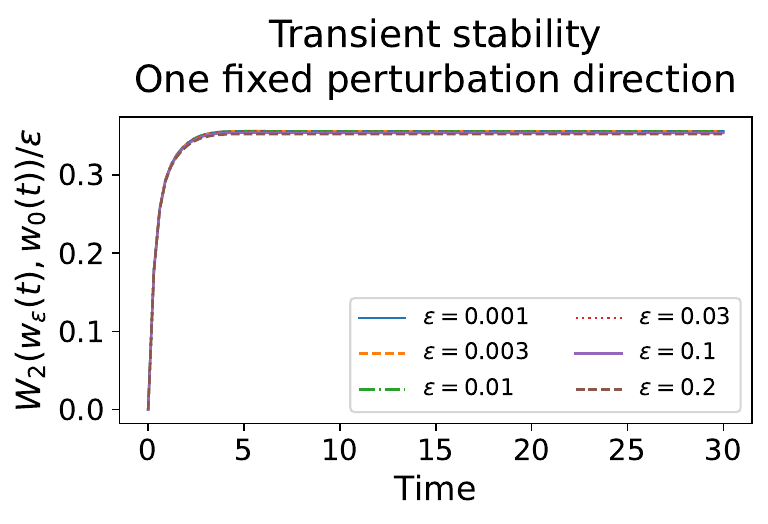}
        \caption{Perturbed joint-law stability}
    \end{subfigure}

    \smallskip
    \begin{subfigure}[b]{0.44\linewidth}
        \centering
        \includegraphics[width=\linewidth]{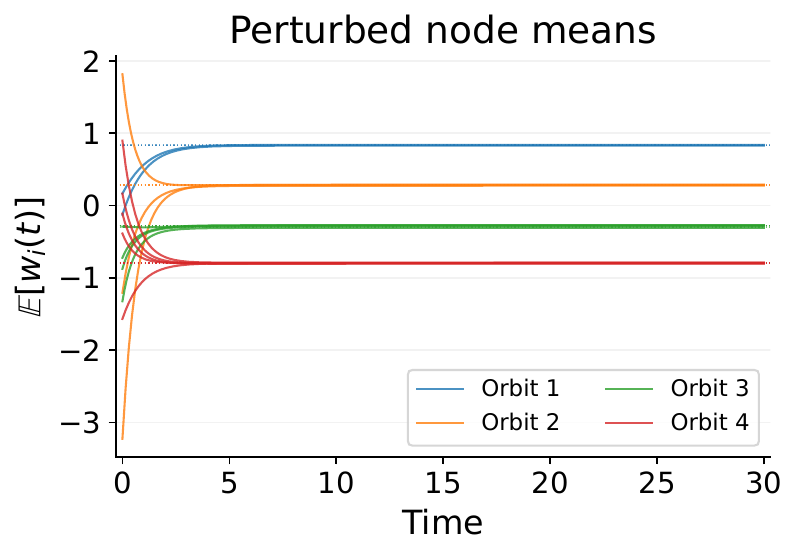}
        \caption{Perturbed node means}
    \end{subfigure}\hfill
    \begin{subfigure}[b]{0.44\linewidth}
        \centering
        \includegraphics[width=\linewidth]{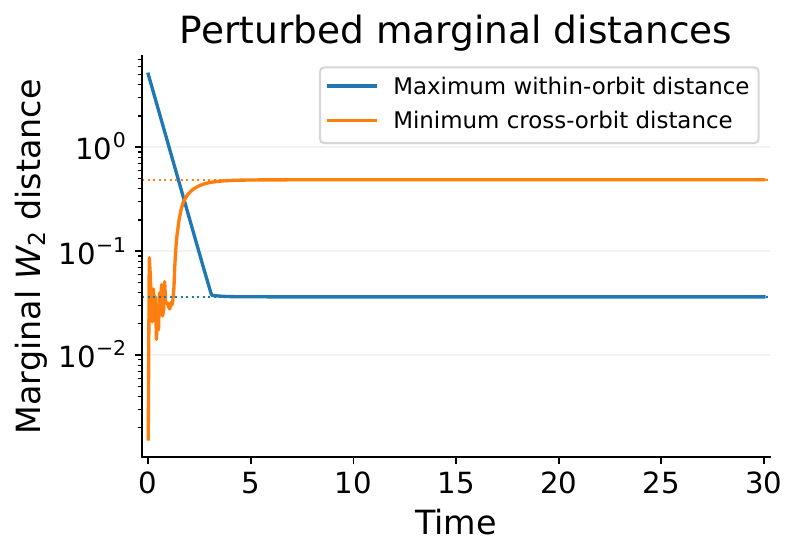}
        \caption{Perturbed marginal distances}
    \end{subfigure}
    \caption{Temporal validation from exact affine moments.
    (a) Reference transverse errors and exponential bounds.
    (b) Joint-law $W_2/\varepsilon$ for seed $101$ at six defect levels.
    (c)--(d) Perturbed means and marginal distances
    ($\varepsilon=0.1$, seed $101$); dotted lines are stationary limits.
    Same-orbit means share a color.}
    \label{fig:app_temporal_validation}
\end{figure}

\begin{figure}[t]
    \centering
    \captionsetup{font=small,skip=3pt}
    \captionsetup[subfigure]{font=small,labelfont=bf,skip=1pt}
    \begin{subfigure}[b]{0.44\linewidth}
        \centering
        \includegraphics[width=\linewidth]{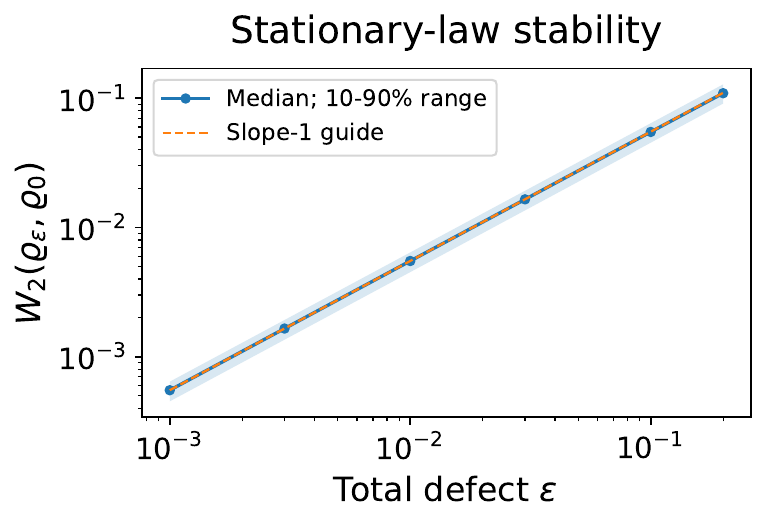}
        \caption{Joint stationary-law distance}
    \end{subfigure}\hfill
    \begin{subfigure}[b]{0.44\linewidth}
        \centering
        \includegraphics[width=\linewidth]{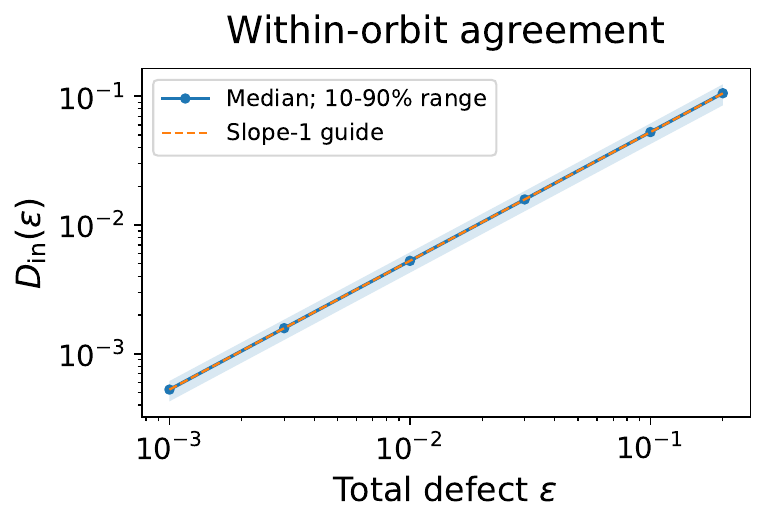}
        \caption{Within-orbit parameter discrepancy}
    \end{subfigure}
    \caption{Stationary perturbation responses.
    Solid curves and shading show medians and 10th--90th percentiles
    over $20$ fixed directions. Dashed lines are slope-one guides,
    not theoretical bounds.}
    \label{fig:app_scaling}
\end{figure}

Figure~\ref{fig:app_temporal_validation}(b) evaluates the exact
finite-time joint-law distance from each perturbed system to the
reference, divided by $\varepsilon$, for direction seed $101$ and
all six perturbation magnitudes. It uses $101$ equally spaced times in
$[0,30]$ and the common initial state. This finite-grid diagnostic
illustrates the uniform-in-time bound proved in
Appendix~\ref{app:four_orbit_perturbation_bounds}.
Panels (c)--(d) show the $14$ exact mean curves and the marginal
$W_2$ distances for the representative perturbation.
The within-orbit and cross-orbit marginal distances approach
$0.036367$ and $0.486425$, respectively. Thus the within-orbit limit
is small but nonzero, while the cross-orbit limit remains separated.

Figure~\ref{fig:app_scaling} compares the stationary joint $W_2$
displacement and the largest within-orbit parameter discrepancy with
slope-one guides. For each fixed direction, ordinary least squares of
$\log(\text{response})$ on $\log\varepsilon$ at
$\varepsilon\in\{0.001,0.003,0.01,0.03\}$ gives descriptive slopes;
their medians are $0.999867$ and $0.999877$, respectively.
These descriptive slopes are consistent with
\eqref{eq:app_stationary_w2} and \eqref{eq:app_moment_robustness};
they do not estimate uniform constants or a sharp threshold.
Figure~\ref{fig:orbit_separation_margin} reports the corresponding
positive separation margins. Shading in both figures describes variation
across the $20$ fixed directions, not confidence intervals.

\paragraph{Risk robustness.}
\label{app:four_orbit_risk_robustness}

For each of the $120$ configurations, we evaluate the risks and matched
benchmarks in Appendix~\ref{app:four_orbit_perturbation_bounds} and
the joint certificate \eqref{eq:app_joint_risk_certificate}.
Figure~\ref{fig:app_risk_robustness} summarizes the two gains and the
certificate across perturbation magnitudes.

\begin{figure}[t]
    \centering
    \captionsetup{font=small,skip=3pt}
    \includegraphics[width=\linewidth]{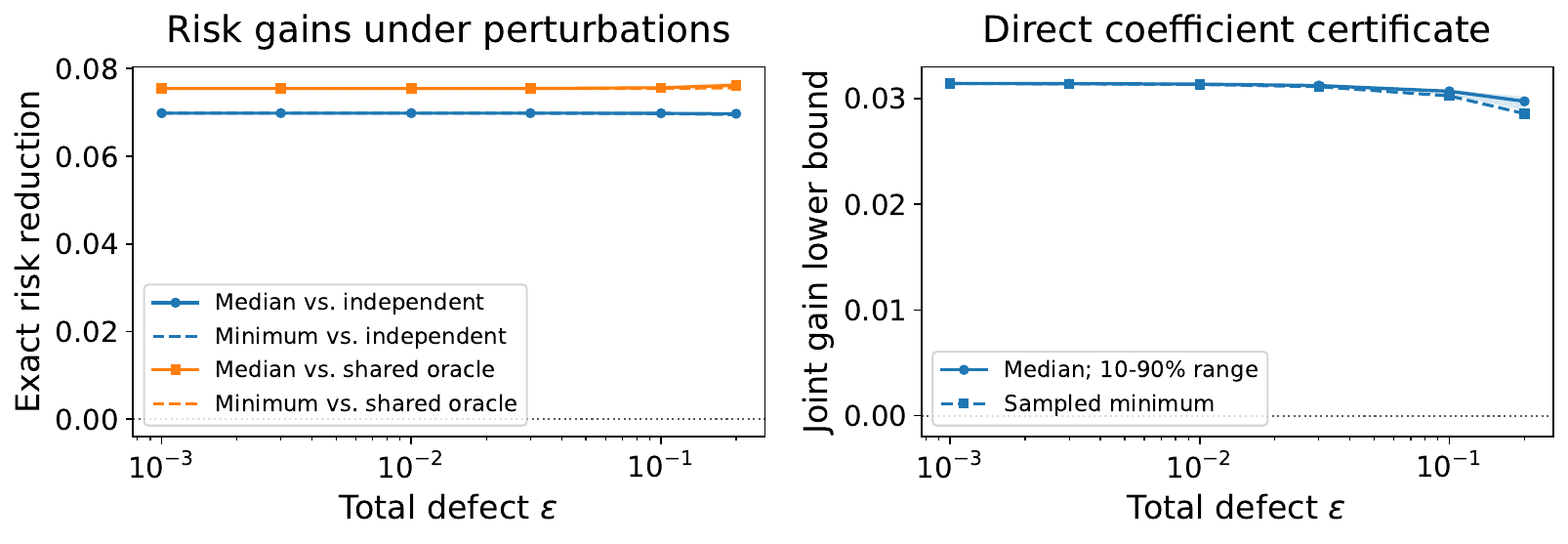}
    \caption{Stationary risk gains (left) and direct joint certificate
    $\gamma(\theta_\varepsilon)$ (right). Solid curves, shading, and dashed
    curves show medians, 10th--90th percentiles, and sampled minima over
    $20$ fixed directions. These summaries are not confidence bounds
    or worst-case guarantees.}
    \label{fig:app_risk_robustness}
\end{figure}

All \(120\) sampled configurations retain both positive exact gains
and positive direct joint certificates. Across these configurations, the
smallest gains relative to independent training and the shared
oracle are \(0.06948214\) and \(0.07543898\), respectively, and
\(\min\gamma(\theta_\varepsilon)\approx0.02856793\). For the representative
\(\varepsilon=0.1\), direction seed \(101\), the values are
\[
\begin{array}{l|r}
\text{Quantity}&\text{Value}\\ \hline
\mathcal E_\varepsilon&0.18020345\\
\mathcal E_{\rm ind}^\varepsilon&0.25000179\\
\mathcal H_{\rm one}^\varepsilon&0.25571117\\
\mathcal E_{\rm ind}^\varepsilon-\mathcal E_\varepsilon&0.06979834\\
\mathcal H_{\rm one}^\varepsilon-\mathcal E_\varepsilon&0.07550772\\
\gamma(\theta_\varepsilon)&0.03091489
\end{array}
\]
The largest absolute residual in the gain identity
\eqref{eq:app_exact_gain_identity} over the sampled configurations is
below \(4\times10^{-16}\). These risk
calculations and the positive orbit margins above verify the two
properties separately for the same evaluated family.

For the representative perturbation, the terminal Monte Carlo risk is
$0.174827$, with approximate $95\%$ interval $[0.169285,0.180370]$.
This contains its exact stationary risk $0.180203$.
The computed joint $W_2$ distance from the terminal law to stationarity
is below $2\times10^{-13}$, and the absolute finite-time risk bias is
below $10^{-15}$, so the discrepancy is dominated by sampling variability.

\subsection{Nonquadratic experiment}
\label{app:four_orbit_nonquadratic}

We retain \(C_0=0.1A\), \(\eta=H=1\), \(\sigma_i=1\), and the
reference minimizers, but replace the quadratic objectives by
\[
\mathcal L_i(w)=\log(1+e^{w-r_i})
+\log(1+e^{-(w-r_i)})+\frac14(w-r_i)^2.
\]
Their derivatives satisfy
\[
\mathcal L_i'(w)=\tanh((w-r_i)/2)+\tfrac12(w-r_i),\quad \mathcal L_i''(w)=\tfrac12\operatorname{sech}^2((w-r_i)/2)+\tfrac12.
\]
Hence \(\mathcal L_i'(r_i)=0\), \(\mathcal L_i''(r_i)=1\), and
\(\tfrac12\leq\mathcal L_i''(w)\leq1\).
These globally strongly convex losses match the quadratic reference
minimizers and their local curvatures, but the full drift is no longer
affine.

For this symmetric, constant-noise construction, define
\[
V(w)=\sum_i\mathcal L_i(w_i)+\tfrac12w^\top L_{C_0}w.
\]
Then \(dw=-\nabla V(w)dt+dB\), with
\(\nabla^2V\succeq\tfrac12I\).
The density \(Z^{-1}\exp[-2V(w)]\) is normalizable and has zero
stationary probability current. Its within-orbit permutation invariance
implies equal stationary node marginals within each orbit.
This check is specific to the present gradient system and does not
extend the affine Gaussian perturbation theorem to nonlinear dynamics.
Cross-orbit separation is assessed numerically rather than inferred
from strong convexity alone.

We use the initialization in
Appendix~\ref{app:four_orbit_simulation_settings} and \(600\)
independent realizations
to \(T=30\), with simulation seed \(30006\).
Euler--Maruyama at step \(h\) uses
\[
w_i^{n+1}=w_i^n+h\left[-\mathcal L_i'(w_i^n) +\sum_j(C_0)_{ij}(w_j^n-w_i^n)\right] +\sqrt h\,\xi_i^n,
\]
where \(\xi_i^n\stackrel{\mathrm{iid}}{\sim}\mathcal N(0,1)\).
The steps \(h=0.01\) and \(h/2=0.005\) are paired: each coarse
Brownian increment is the sum of its two fine increments.
We record all $600$ fine-step realizations at
$t=0,0.1,\ldots,30$; terminal samples also come from this simulation.
At each recorded time, node means average over the $600$ realizations.
The temporal plots use no smoothing or uncertainty bands.

Let $\widehat\nu_i(t)$ be the empirical marginal of the $600$
realizations at time $t$. Using sorted samples $x_{i,(s)}(t)$, we compute
\[
W_2^2(\widehat\nu_i(t),\widehat\nu_j(t))
=\frac1{600}\sum_{s=1}^{600}
\bigl(x_{i,(s)}(t)-x_{j,(s)}(t)\bigr)^2.
\]
Taking the same within-orbit maxima and cross-orbit minima as in
Appendix~\ref{app:four_orbit_simulation_settings} gives
$\widehat D_{\mathrm{in},W_2}(t)$ and
$\widehat D_{\mathrm{out},W_2}(t)$.

The paired terminal RMS difference per node is
\[
\left[\frac{1}{600\cdot14} \sum_{s=1}^{600}\sum_{i=1}^{14} \bigl(w_i^{(s,h)}(T)-w_i^{(s,h/2)}(T)\bigr)^2\right]^{1/2} \approx2.66058\times10^{-3}.
\]
This diagnoses discretization sensitivity; it is not a bound on the
fine-step error or a certificate of stationarity at \(T=30\).

Figure~\ref{fig:app_nonquadratic} displays the node-wise sample means,
temporal marginal distances, and terminal marginals.
The node means form four groups after the initial transient, with
persistent Monte Carlo fluctuations. The comparison concerns
distributions and their means; individual noisy trajectories need not
agree or converge to fixed locations.

\begin{figure}[b]
    \centering
    \captionsetup{font=small,skip=3pt}
    \captionsetup[subfigure]{font=small,labelfont=bf,skip=1pt}
    \begin{subfigure}[b]{0.32\linewidth}
        \centering
        \includegraphics[width=\linewidth]{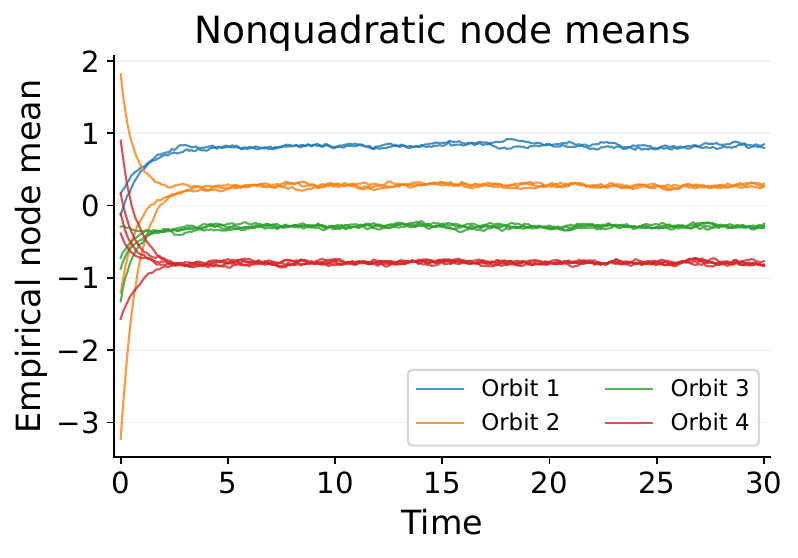}
        \caption{Node-wise sample means}
    \end{subfigure}\hfill
    \begin{subfigure}[b]{0.32\linewidth}
        \centering
        \includegraphics[width=\linewidth]{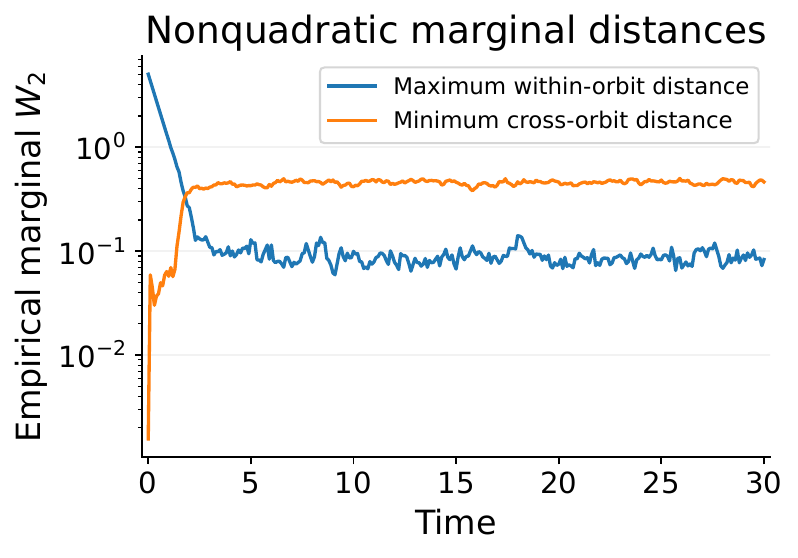}
        \caption{Empirical marginal distances}
    \end{subfigure}\hfill
    \begin{subfigure}[b]{0.32\linewidth}
        \centering
        \includegraphics[width=\linewidth]{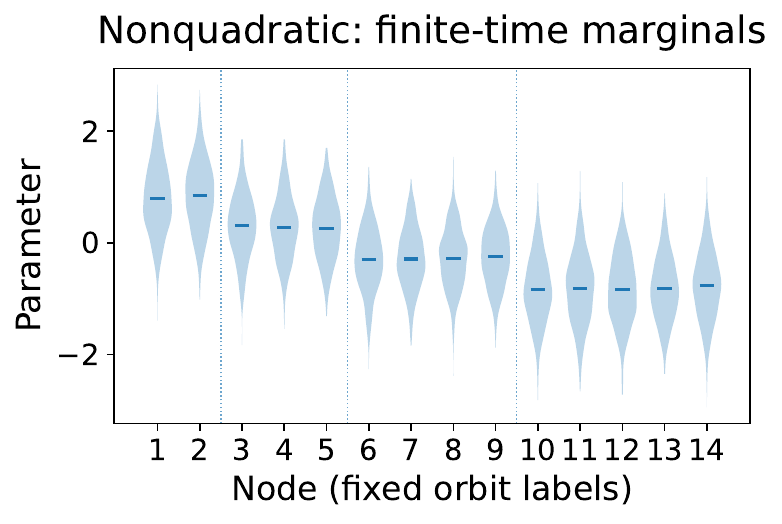}
        \caption{Terminal node marginals}
    \end{subfigure}
    \caption{Nonquadratic finite-time grouping.
    (a) Node-wise sample means; colors denote orbits.
    (b) Maximum within-orbit and minimum cross-orbit empirical $W_2$.
    (c) Terminal violin plots at $T=30$, with sample means and orbit separators.
    Empirical fluctuations include finite-sample error.}
    \label{fig:app_nonquadratic}
\end{figure}

At $T=30$, the distances and their separation margin are
\[
\widehat D_{\mathrm{in},W_2}\approx0.083102,\qquad \widehat D_{\mathrm{out},W_2}\approx0.464167,\qquad \widehat{\mathcal M}_{W_2}\approx0.381064.
\]
These empirical Wasserstein quantities differ from the Gaussian
parameter metric used in the quadratic study. They include finite-time,
sampling, and discretization effects, so nonzero within-orbit empirical
discrepancies do not contradict equality of stationary marginals.
These finite-horizon diagnostics do not establish stationarity or
monotone convergence. No Gaussian stationary approximation or nonlinear
joint-perturbation bound is tested. The experiment provides a qualitative extension of
orbit grouping beyond the affine setting.

\section{Real-world experiments}
\label{app:real_world_experiments}

\subsection{Experimental setup}

These experiments examine whether local training with graph communication produces relatively stable prediction clusters aligned with graph orbits. The experiments cross two datasets, two architectures, four conditions, and three training seeds. All \(48\) runs completed \(200\) epochs and contain evaluation records at epochs \(180,185,190,195,200\). The four dataset--architecture configurations are MNIST/MLP-LN, MNIST/ResNet-20-GN, CIFAR-10/MLP-LN, and CIFAR-10/ResNet-20-GN.

For each of MNIST and CIFAR-10, \(5{,}000\) images from the official training set, with \(500\) per class, form a validation set. The local training allocation comprises \(14\) disjoint bundles of \(2{,}000\) images, disjoint from validation. The common evaluation probe is the official test set of \(10{,}000\) images. Validation is used for communication-strength selection; test predictions are used for final evaluation. The split seed is \(2027\).

The heterogeneous allocation assigns the class counts in Table~\ref{tab:realworld_profiles} to the reference orbits of sizes \((2,3,4,5)\). Nodes in the same orbit share class counts but use distinct images. Thus the construction aligns label frequencies with graph orbits without making finite-sample objectives or gradient noise identical. In the fully symmetric condition each node instead receives \(200\) images per class. Balanced and heterogeneous allocations share the same validation indices; their training allocations are constructed separately.

\begin{table}[t]
\centering\small
\setlength{\tabcolsep}{4pt}
\caption{Training images per class at each node in the aligned allocation. Class indices follow the dataset's native order. Each row sums to \(2{,}000\).}
\label{tab:realworld_profiles}
\begin{tabular}{crrrrrrrrrr}
\toprule
Orbit & 0 & 1 & 2 & 3 & 4 & 5 & 6 & 7 & 8 & 9 \\
\midrule
\(\mathcal O_1\) & 600 & 500 & 200 & 160 & 140 & 120 & 100 & 80 & 60 & 40 \\
\(\mathcal O_2\) & 60 & 40 & 600 & 500 & 200 & 160 & 140 & 120 & 100 & 80 \\
\(\mathcal O_3\) & 100 & 80 & 60 & 40 & 600 & 500 & 200 & 160 & 140 & 120 \\
\(\mathcal O_4\) & 200 & 160 & 140 & 120 & 100 & 80 & 60 & 40 & 600 & 500 \\
\bottomrule
\end{tabular}
\end{table}

MNIST inputs have shape \(1\times28\times28\) and are normalized by mean \(0.1307\) and standard deviation \(0.3081\). CIFAR-10 inputs have shape \(3\times32\times32\) and use channel means \((0.4914,0.4822,0.4465)\) and standard deviations \((0.2470,0.2435,0.2616)\). CIFAR-10 training uses random crops with four-pixel padding and random horizontal flips. Validation and test transforms contain no random augmentation. Sampling and augmentation streams are attached to the data bundle and run seed, with epoch and image indexing as appropriate. They move with the bundle in the misaligned condition.

The MLP uses hidden widths \(256\) and \(128\), each followed by LayerNorm and ReLU, and a ten-class output layer. ResNet-20 starts with a \(16\)-channel convolution, followed by three stages of widths \(16\), \(32\), and \(64\), each containing three two-convolution residual blocks. The first block of each of the last two stages downsamples its input. Global average pooling and a ten-class linear layer complete the model. GroupNorm uses the largest divisor of the channel count not exceeding eight. Neither architecture uses BatchNorm or dropout.

Each run starts from a common base model copied to all nodes. Gaussian parameter perturbations are centered within each reference orbit and normalized with a common RMS scale. Their scale is \(0.02\) times the larger of the base tensor's RMS and \(0.001\). Orbit-average initial parameters therefore agree up to numerical precision. In the fully symmetric condition centering is over all \(14\) nodes. Orbit labels enter this initialization and the data construction; the gradient updates below use only local data and neighbor parameters.

The reference orbits are \(\mathcal O_1=\{1,2\}\), \(\mathcal O_2=\{3,4,5\}\), \(\mathcal O_3=\{6,7,8,9\}\), and \(\mathcal O_4=\{10,11,12,13,14\}\). Adjacent orbits are connected by complete bipartite graphs, with no other edges. In particular, there are no within-orbit edges. For this blown-up path, edge weights are \(C_{ij}=1/(1+\max\{\deg(i),\deg(j)\})\), with zero weights off edges and on the diagonal. Let \(L_C=\operatorname{diag}(C\mathbf 1)-C\). These weights differ from the \(0.1A\) weights of the synthetic reference. The four conditions are:

\begin{itemize}
\item
  \emph{Main}: the four-orbit graph and orbit-aligned heterogeneous bundles, with communication strength \(\lambda\) selected on validation data.
\item
  \emph{No communication}: the same aligned bundles, with \(\lambda=0\).
\item
  \emph{Misaligned data}: the same graph, bundles, and \(\lambda\) as Main, with the one-based node-to-bundle assignment \((1,10,3,6,11,2,4,7,12,5,8,9,13,14)\).
\item
  \emph{Fully symmetric}: a complete graph and balanced class counts at every node, with communication strength \(\lambda_{\rm sym}=\lambda\|L_{\rm Main}\|_2/\|L_{\rm complete}\|_2\).
\end{itemize}

The complete graph has one automorphism orbit containing all nodes. Its primary agreement diagnostic is global dispersion. Any retained four-group diagnostic uses only the original reference partition. This condition provides a joint comparison of topology and data symmetry. Spectral matching equalizes the largest Laplacian eigenvalue after scaling.

With \(\widehat L_{i,b}\) the mean cross-entropy on node \(i\)'s minibatch, each step uses

\begin{equation}
w_i^{t+1}
=w_i^t-\alpha_t\left[
\nabla\widehat L_{i,b}(w_i^t)
+\lambda\sum_j C_{ij}(w_i^t-w_j^t)
\right].
\label{eq:realworld_update}
\end{equation}

All communication terms use the same pre-update parameter snapshot. Communication is applied at every step to all trainable parameters, including normalization parameters. The protocol uses SGD without momentum or weight decay, batch size \(128\) per node, and \(200\) epochs. The learning rate uses five warmup epochs, a nominal value of \(0.05\), and subsequent cosine decay with minimum learning-rate parameter \(0.0005\). Mixed precision, label smoothing, and gradient clipping are disabled. Evaluation and checkpointing occur every five epochs. Each epoch has \(16\) synchronous steps, including a final minibatch of \(80\) images per node, for \(3{,}200\) steps per run. The training seeds are \(1,2,3\).

Each dataset--architecture combination has its own selection task. The configured pilot uses seed \(0\), \(80\) epochs, and \(\lambda\in\{0,0.03,0.1,0.3\}\), summarized over the final three valid validation checkpoints. Relative to the \(\lambda=0\) pilot, the main rule requires numerical stability, within-orbit dispersion at most \(75\%\) of baseline, between-orbit dispersion at least \(50\%\) of baseline, and an orbit-ensemble accuracy decrease of at most \(0.03\) in absolute units. The smallest positive qualifying coefficient is selected. Each orbit ensemble averages its nodes\textquotesingle{} raw logits before classification; the selection accuracy averages the four orbit-ensemble accuracies equally.

All four configurations satisfied the primary selection rule using only the balanced validation set; no fallback was used. Table~\ref{tab:realworld_lambda} reports the selected communication coefficients. Main and misaligned data use the same coefficient; no communication uses zero. Because the selection rule uses prediction geometry, the test-set results characterize this validation-selected training protocol.

\begin{table}[h]
\centering\small
\setlength{\tabcolsep}{4pt}
\caption{Communication coefficients. Fully symmetric values are rounded for display.}
\label{tab:realworld_lambda}
\begin{tabular}{llcc}
\toprule
Dataset & Model & \(\lambda_{\rm Main}\) & \(\lambda_{\rm sym}\) \\
\midrule
MNIST & MLP-LN & \(0.03\) & \(0.0423251\) \\
MNIST & ResNet-20-GN & \(0.03\) & \(0.0423251\) \\
CIFAR-10 & MLP-LN & \(0.10\) & \(0.141084\) \\
CIFAR-10 & ResNet-20-GN & \(0.10\) & \(0.141084\) \\
\bottomrule
\end{tabular}
\end{table}


\subsection{Additional prediction-space diagnostics}

Let \(Z_i(t)\in\mathbb R^{P\times c}\) contain node \(i\)\textquotesingle s logits at epoch \(t\) on the same ordered test probe, with \(N=14\), \(P=10{,}000\), and \(c=10\). Define the class-centered logits by

\[
(\Phi_i(t))_{pa}
=(Z_i(t))_{pa}
-\frac{1}{c}\sum_{b=1}^{c}(Z_i(t))_{pb},
\qquad a=1,\ldots,c.
\]

Centering removes the common logit offset for each image while retaining logit-scale differences. Suppressing the epoch when all quantities refer to the same checkpoint, the pairwise squared prediction distance is

\begin{equation}
\delta_{ij}
=\frac{\|\Phi_i-\Phi_j\|_{\mathrm F}^2}{Pc}.
\label{eq:realworld_pairwise_distance}
\end{equation}

The heatmaps in the main text use epoch \(200\) and average these distances across the three training seeds. Distances are computed before averaging; the probe and centering convention are the same as for the diagnostics below. Raw distance magnitudes are interpreted separately within each dataset--architecture configuration.

Let \(n_k=|\mathcal O_k|\), \(\bar\Phi_k=n_k^{-1}\sum_{i\in\mathcal O_k}\Phi_i\), and \(\bar\Phi=N^{-1}\sum_i\Phi_i\). Within-orbit dispersion measures deviations of nodes from their orbit centroid, while between-orbit dispersion measures deviations of orbit centroids from the global centroid:

\begin{equation}
\begin{gathered}
D_{\rm within}=\frac{\sum_k\sum_{i\in\mathcal O_k}\|\Phi_i-\bar\Phi_k\|_{\rm F}^2}{NPc},
\qquad
D_{\rm between}=\frac{\sum_k n_k\|\bar\Phi_k-\bar\Phi\|_{\rm F}^2}{NPc},\\[3pt]
D_{\rm total}=D_{\rm within}+D_{\rm between},
\qquad
R_{\rm orbit}=\frac{D_{\rm between}}{D_{\rm total}}.
\end{gathered}
\label{eq:realworld_dispersion}
\end{equation}

Here \(D_{\rm total}=(NPc)^{-1}\sum_i\|\Phi_i-\bar\Phi\|_{\rm F}^2\), and \(R_{\rm orbit}\) is evaluated for \(D_{\rm total}\geq10^{-14}\). The \(n_k\) weights account for unequal orbit sizes. In the fully symmetric condition, \(D_{\rm global}=D_{\rm total}\). The ensemble accuracy used for selection weights the four orbits equally, whereas the dispersion decomposition weights them by node count.

Let \(\mathcal P_{\rm w}\) and \(\mathcal P_{\rm c}\) contain the unordered pairs of distinct nodes within and across reference orbits, respectively. Their sizes are \(|\mathcal P_{\rm w}|=\sum_k\binom{n_k}{2}=20\) and \(|\mathcal P_{\rm c}|=\sum_{k<\ell}n_kn_\ell=71\). The mean pairwise distances and their ratio are

\begin{equation}
P_{\rm within}=\frac{1}{20}\sum_{(i,j)\in\mathcal P_{\rm w}}\delta_{ij},
\quad
P_{\rm cross}=\frac{1}{71}\sum_{(i,j)\in\mathcal P_{\rm c}}\delta_{ij},
\quad
Q=\frac{P_{\rm cross}}{P_{\rm within}}.
\label{eq:realworld_mean_pairwise}
\end{equation}

Each distinct pair has equal weight within its corresponding average. For \(P_{\rm within}>0\), \(Q\) compares mean cross-orbit and within-orbit squared prediction distances, while \(D_{\rm between}/D_{\rm within}\) compares the node-weighted centroid dispersions.

For terminal variation, let \(\mathcal T=\{180,185,190,195,200\}\). For one seed and metric \(x(t)\), set

\[
\bar x=\frac{1}{5}\sum_{t\in\mathcal T}x(t),
\qquad
s_x=\left\{\frac{1}{4}\sum_{t\in\mathcal T}
(x(t)-\bar x)^2\right\}^{1/2}.
\]

Let \(\widehat b_x\) be the ordinary least-squares slope from regressing the five values \(x(t)\) on epoch \(t\), with an intercept. The percentage diagnostics are

\begin{equation}
\operatorname{CV}(x)=100\frac{s_x}{|\bar x|},
\qquad
\operatorname{Drift}(x)=100\frac{20|\widehat b_x|}{|\bar x|}.
\label{eq:realworld_terminal_stability}
\end{equation}

These diagnostics are evaluated for \(D_{\rm within}\), \(D_{\rm between}\), and \(R_{\rm orbit}\) when all five values are finite and \(|\bar x|\geq10^{-14}\). Here \(s_x\) measures within-run variation across checkpoints, while across-seed SD summarizes differences among training runs. Drift measures the fitted linear change over \(20\) epochs. Values of both diagnostics below \(5\%\) define the descriptive criterion for finite-time terminal stability.

\begin{table}[t]
\centering\small
\setlength{\tabcolsep}{4pt}
\caption{Main prediction geometry (mean \(\pm\) sample SD over three seeds). Panel A averages the five terminal checkpoints per seed, with \(\overline R_{\rm orbit}\) averaging checkpoint ratios. Panel B uses epoch \(200\), with \(Q\) computed within each seed before aggregation.}
\label{tab:realworld_terminal}
\parbox{\linewidth}{\centering \textbf{A. Terminal dispersion.}}\par\smallskip
\begin{tabular}{llccc}
\toprule
Dataset & Model & \(\overline D_{\rm within}\) & \(\overline D_{\rm between}\) & \(\overline R_{\rm orbit}\) \\
\midrule
MNIST & MLP-LN & \(0.2460\pm0.0096\) & \(0.3092\pm0.0073\) & \(0.5569\pm0.0062\) \\
MNIST & ResNet-20-GN & \(0.2189\pm0.0229\) & \(0.3236\pm0.0180\) & \(0.5969\pm0.0319\) \\
CIFAR-10 & MLP-LN & \(0.2181\pm0.0045\) & \(0.5094\pm0.0098\) & \(0.7002\pm0.0026\) \\
CIFAR-10 & ResNet-20-GN & \(0.2689\pm0.0054\) & \(0.5641\pm0.0020\) & \(0.6772\pm0.0037\) \\
\bottomrule
\end{tabular}
\par\medskip
\parbox{\linewidth}{\centering \textbf{B. Average node-pair squared distances at epoch \(200\).}}\par\smallskip
\begin{tabular}{llccc}
\toprule
Dataset & Model & \(P_{\rm within}\) & \(P_{\rm cross}\) & \(Q\) \\
\midrule
MNIST & MLP-LN & \(0.6928\pm0.0288\) & \(1.3377\pm0.0364\) & \(1.932\pm0.036\) \\
MNIST & ResNet-20-GN & \(0.6256\pm0.0675\) & \(1.3210\pm0.0512\) & \(2.124\pm0.189\) \\
CIFAR-10 & MLP-LN & \(0.6153\pm0.0132\) & \(1.8347\pm0.0351\) & \(2.982\pm0.026\) \\
CIFAR-10 & ResNet-20-GN & \(0.7627\pm0.0204\) & \(2.0999\pm0.0300\) & \(2.754\pm0.052\) \\
\bottomrule
\end{tabular}
\end{table}

For recorded epochs \(a<b\leq200\), define centroid distances and
endpoint displacements by
\begin{equation}
\begin{gathered}
B_{k\ell}(t)
=\frac{\|\bar\Phi_k(t)-\bar\Phi_\ell(t)\|_{\rm F}^2}{Pc},
\quad
T_k(a,b)
=\frac{\|\bar\Phi_k(b)-\bar\Phi_k(a)\|_{\rm F}^2}{Pc},\\[3pt]
\rho(a,b)
=\sqrt{\frac{\max_k T_k(a,b)}
{\min_{k<\ell}B_{k\ell}(200)}}.
\end{gathered}
\label{eq:realworld_centroid_motion}
\end{equation}

The centroid geometry satisfies
\(D_{\rm between}(t)
=N^{-2}\sum_{k<\ell}n_kn_\ell B_{k\ell}(t)\).
The ratio \(\rho(a,b)\) is evaluated when
\(\min_{k<\ell}B_{k\ell}(200)>0\).
The square root puts displacement and separation on a distance scale:
\(100\rho(a,b)\) is the largest orbit-centroid net displacement
from epoch \(a\) to epoch \(b\), expressed as a percentage of the
smallest centroid separation at epoch \(200\).
Each \(T_k(a,b)\) measures net displacement between its stated endpoints.
The intervals \((a,b)=(175,200)\) and \((195,200)\) describe changes
over \(25\) and \(5\) epochs, respectively; they are distinct from
the \(180\)--\(200\) window used for
Eq.~\eqref{eq:realworld_terminal_stability}.
All twelve Main runs have recorded geometry and displacement data
for these comparisons.
For matched-seed control comparisons, let an overbar denote the mean over \(\mathcal T\), and define

\begin{equation}
\begin{gathered}
\Delta_{B,s}=\overline D_{{\rm within},s}^{\rm NoComm}-\overline D_{{\rm within},s}^{\rm Main},
\quad
\Delta_{C,s}=\overline D_{{\rm within},s}^{\rm Misaligned}-\overline D_{{\rm within},s}^{\rm Main},\\[3pt]
\Delta_{D,s}=\overline D_{{\rm total},s}^{\rm Main}-\overline D_{{\rm global},s}^{\rm FullySym}.
\end{gathered}
\label{eq:realworld_control_contrasts}
\end{equation}

The corresponding relative reductions, reported in percent, are

\begin{equation}
\eta_{B,s}=\frac{100\Delta_{B,s}}{\overline D_{{\rm within},s}^{\rm NoComm}},
\qquad
\eta_{C,s}=\frac{100\Delta_{C,s}}{\overline D_{{\rm within},s}^{\rm Misaligned}},
\qquad
\eta_{D,s}=\frac{100\Delta_{D,s}}{\overline D_{{\rm total},s}^{\rm Main}}.
\label{eq:realworld_control_reductions}
\end{equation}

Thus \(\eta_B\) and \(\eta_C\) measure the reduction in Main within-orbit dispersion relative to each control, while \(\eta_D\) measures the reduction in fully symmetric global dispersion relative to Main.

\begin{figure}[t]
  \centering
  \captionsetup{
    font=small,
    labelfont=bf,
    skip=4pt
  }
\includegraphics[width=0.98\linewidth]
    {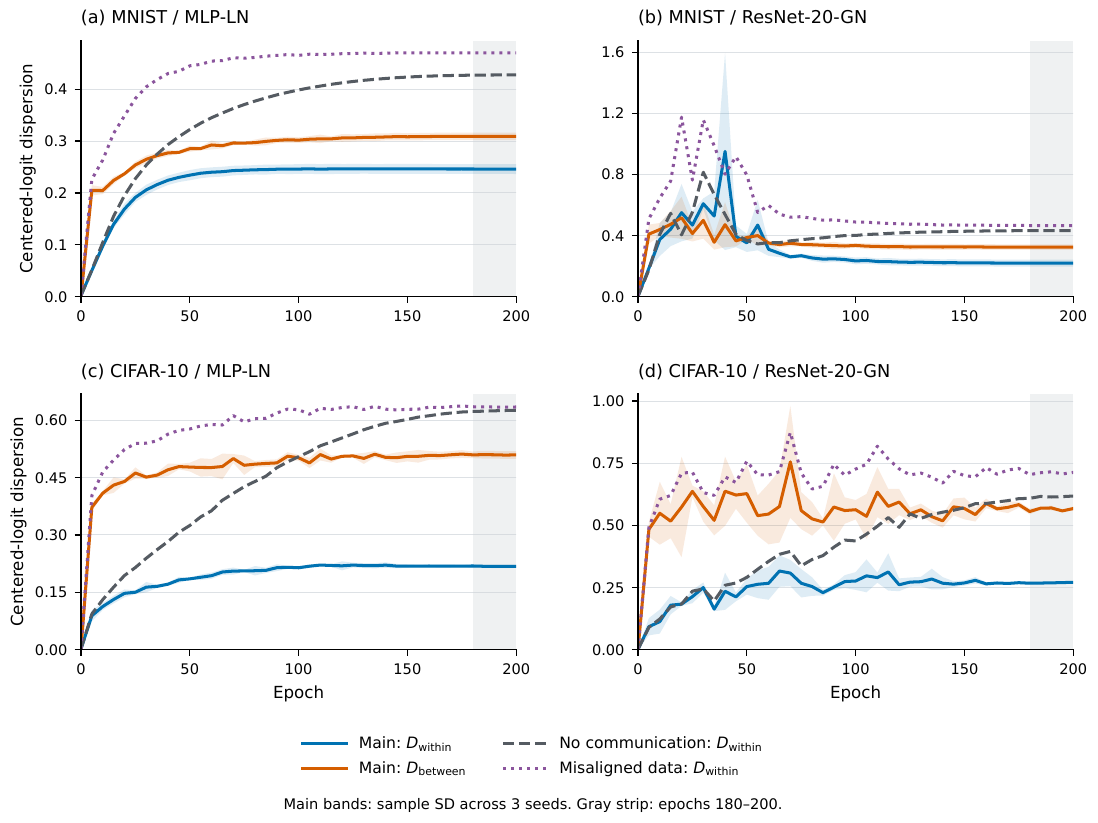}

  \caption{Prediction-space dispersion during training.
  Curves show three-seed means; bands indicate
  \(\pm\) one sample SD for Main.
  The gray strip marks epochs \(180\)--\(200\).
  Vertical scales vary across panels.}
  \label{fig:realworld_training_dynamics}
\end{figure}

\begin{table}[t]
  \centering
  \small
  \setlength{\tabcolsep}{5pt}
  \renewcommand{\arraystretch}{1.15}
  \captionsetup{
    font=small,
    labelfont=bf,
    position=top,
    skip=5pt
  }

  \caption{Matched-seed relative reductions in terminal dispersion
  (mean \(\pm\) sample SD over three seeds, in percent),
  defined in Eq.~\eqref{eq:realworld_control_reductions}.}
  \label{tab:realworld_controls}

  \begin{tabular*}{\linewidth}
    {@{\extracolsep{\fill}}llccc@{}}
    \toprule
    Dataset & Model
    & \shortstack{\(\eta_B\)\\Main vs.\ NoComm}
    & \shortstack{\(\eta_C\)\\Main vs.\ Misaligned}
    & \shortstack{\(\eta_D\)\\FullySym vs.\ Main} \\
    \midrule
    MNIST & MLP-LN
    & \(42.45\pm1.01\)
    & \(47.69\pm0.63\)
    & \(53.25\pm0.82\) \\
    MNIST & ResNet-20-GN
    & \(49.34\pm3.17\)
    & \(53.00\pm3.64\)
    & \(69.42\pm2.46\) \\
    CIFAR-10 & MLP-LN
    & \(65.07\pm0.57\)
    & \(65.64\pm0.44\)
    & \(80.24\pm0.37\) \\
    CIFAR-10 & ResNet-20-GN
    & \(56.24\pm0.33\)
    & \(62.13\pm0.84\)
    & \(77.29\pm1.43\) \\
    \bottomrule
  \end{tabular*}
\end{table}

\subsection{Additional results}
\label{app:realworld_additional_results}

The repetition unit is a full \(14\)-node training run with a distinct training seed, conditional on the fixed data split and validation-selected communication coefficient. Nodes and checkpoints are components of a run. For terminal dispersions and \(R_{\rm orbit}\), we first average the five evaluations in \(\mathcal T\) within each seed and then report the mean and sample SD across seeds \(1,2,3\). In particular, \(\overline R_{{\rm orbit},s}=\frac15\sum_{t\in\mathcal T}D_{{\rm between},s}(t)/D_{{\rm total},s}(t)\) is the mean of checkpoint ratios. Pairwise averages and \(Q\) use epoch \(200\); \(Q\), the control contrasts, and the relative reductions are computed within each seed before aggregation. CV, fitted drift, and \(\rho\) are reported as ranges across the three seeds. Across-seed SD describes training-seed variation on the fixed split.

\paragraph{Orbit-aligned prediction geometry.}
Table~\ref{tab:realworld_terminal} summarizes the Main prediction geometry.
Across configurations, terminal mean within-orbit dispersion ranges from
\(0.2181\) to \(0.2689\), while between-orbit dispersion is larger in all
twelve Main runs. The configuration means of \(\overline R_{\rm orbit}\)
range from \(0.5569\) to \(0.7002\).
At epoch \(200\), the mean cross/within squared-distance ratios are
\(1.932\), \(2.124\), \(2.982\), and \(2.754\) for MNIST/MLP-LN,
MNIST/ResNet-20-GN, CIFAR-10/MLP-LN, and CIFAR-10/ResNet-20-GN,
respectively. Each seed has a larger average cross-orbit distance than
average within-orbit distance. These results quantify the relatively
similar predictions within reference orbits and the cross-orbit
differences displayed by the heatmaps, despite the absence of
within-orbit communication edges.



\paragraph{Communication and data-topology alignment.}
Figure~\ref{fig:realworld_training_dynamics} tracks Main
within- and between-orbit dispersion together with within-orbit
dispersion under the no-communication and misaligned-data controls.
During late training, Main retains lower within-orbit dispersion
than either control while preserving cross-orbit differences.
Table~\ref{tab:realworld_controls} quantifies the terminal reductions,
which are positive in every matched-seed comparison.
Under fully symmetric topology and balanced data, global dispersion
is also lower than Main's total dispersion, indicating stronger
global predictive agreement under this joint change.





\paragraph{Finite-time terminal stability.}
Figure~\ref{fig:realworld_centroid_motion} shows an overall
reduction in normalized five-epoch centroid displacements
toward the final checkpoint, with local fluctuations.
Across all twelve Main runs, the largest orbit-centroid net
displacement over epochs \(195\)--\(200\) is at most
\(3.1536\%\) of the smallest final centroid separation.
Table~\ref{tab:realworld_terminal_variation} complements
these interval-wise measurements with terminal dispersion
variation and net displacement over epochs \(175\)--\(200\).
All \(36\) seed--metric combinations have terminal CV
and fitted drift below \(5\%\).
Together, these observations support finite-time terminal
stability under the stated learning-rate schedule.


\begin{figure}[t]
  \centering
  \captionsetup{
    font=small,
    labelfont=bf,
    skip=4pt
  }

  \includegraphics[width=0.98\linewidth]
    {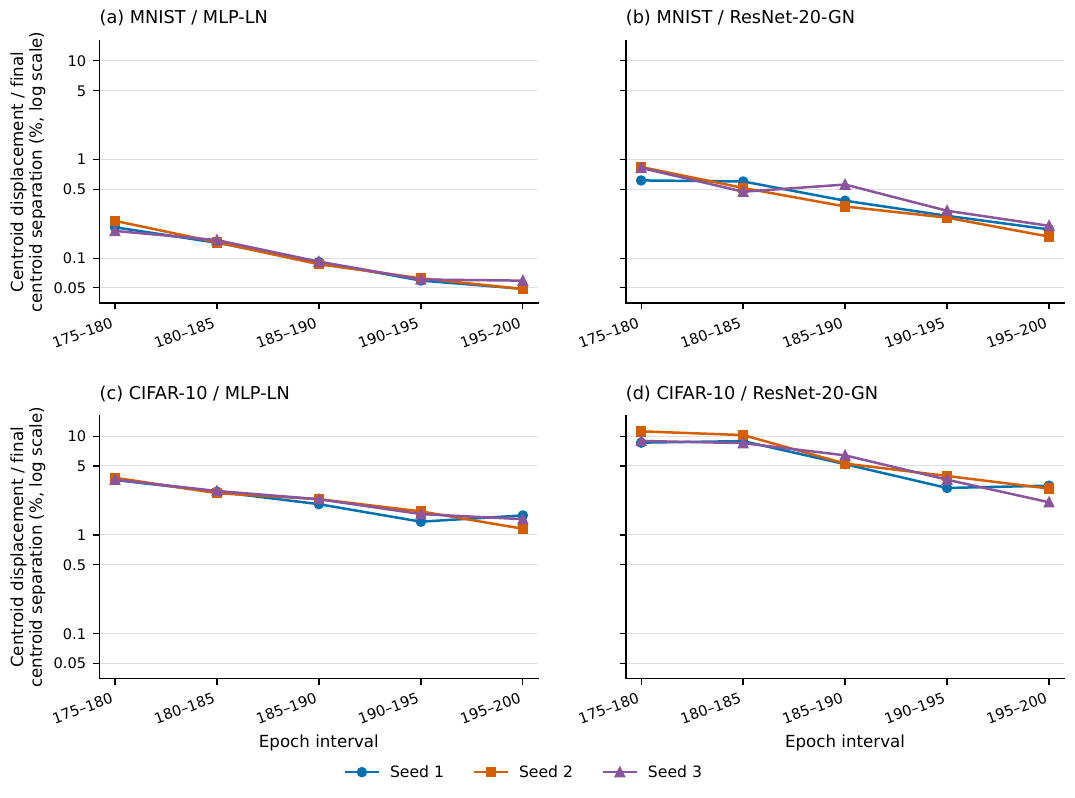}

  \caption{Main orbit-centroid motion over consecutive
  five-epoch intervals.
  Points show \(100\rho(a,a+5)\) for each seed, normalized
  by the smallest centroid separation at epoch \(200\).
  The vertical axis is logarithmic.}
  \label{fig:realworld_centroid_motion}
\end{figure}



\begin{table}[t]
  \centering
  \small
  \setlength{\tabcolsep}{5pt}
  \renewcommand{\arraystretch}{1.15}
  \captionsetup{
    font=small,
    labelfont=bf,
    position=top,
    skip=5pt
  }

  \caption{Main terminal stability, in percent.
  CV and fitted drift are maximized over three seeds and
  the metrics \(D_{\rm within}\), \(D_{\rm between}\),
  and \(R_{\rm orbit}\), using the five checkpoints
  at epochs \(180\)--\(200\).
  The last column gives the three-seed range of normalized
  net centroid displacement over epochs \(175\)--\(200\).}
  \label{tab:realworld_terminal_variation}

  \begin{tabular*}{\linewidth}
    {@{\extracolsep{\fill}}llccc@{}}
    \toprule
    Dataset & Model
    & Max.\ CV
    & Max.\ drift
    & \(100\rho(175,200)\) \\
    \midrule
    MNIST & MLP-LN
    & \(0.0312\)
    & \(0.0465\)
    & \(0.4494\)--\(0.4652\) \\
    MNIST & ResNet-20-GN
    & \(0.0572\)
    & \(0.1385\)
    & \(0.8536\)--\(1.1952\) \\
    CIFAR-10 & MLP-LN
    & \(0.3090\)
    & \(0.6350\)
    & \(4.0586\)--\(4.5155\) \\
    CIFAR-10 & ResNet-20-GN
    & \(2.6920\)
    & \(4.0136\)
    & \(6.6050\)--\(9.3083\) \\
    \bottomrule
  \end{tabular*}
\end{table}



\end{document}